%% file: main.tex
\documentclass{article}

\usepackage{microtype}
\usepackage{graphicx}
\usepackage{booktabs} 
\usepackage{amsfonts}       
\usepackage{nicefrac}       
\usepackage{xcolor}         
\usepackage{bbm}
\usepackage{natbib}
\usepackage{array}
\usepackage{tabularx}
\usepackage{makecell}
\usepackage{pifont}
\usepackage{wrapfig}
\usepackage{enumitem}
\usepackage{tcolorbox}
\newcolumntype{P}[1]{>{\centering\arraybackslash}p{#1}}
\newcolumntype{M}[1]{>{\centering\arraybackslash}m{#1}}
\input{math_commands}

\usepackage{hyperref}

\newcommand{\myparagraph}[1]{\noindent\textbf{#1.}}

\usepackage[preprint]{icml2026}

\usepackage{amsmath}
\usepackage{amssymb}
\usepackage{mathtools}
\usepackage{amsthm}

\usepackage[capitalize,noabbrev]{cleveref}

\theoremstyle{plain}

\theoremstyle{remark}

\usepackage[textsize=tiny]{todonotes}

\def\mytitle{Why Can Diffusion Models Learn to Denoise? A Convergence Analysis with Nonlinear Transformers}
\def\mytitle{Transformer-Based Diffusion Models Provably Converge to Optimal Denoiser}
\def\mytitle{Transformers Learn the Optimal DDPM Denoiser for Multi-Token GMMs}

\icmltitlerunning{\mytitle}

\begin{document}

\twocolumn[
  \icmltitle{\mytitle}



  \icmlsetsymbol{equal}{*}

  \begin{icmlauthorlist}
    \icmlauthor{Hongkang Li}{u1}
    \icmlauthor{Hancheng Min}{u2}
    \icmlauthor{Rene Vidal}{u1}
  \end{icmlauthorlist}

  \icmlaffiliation{u1}{Department of Electrical and Systems Engineering, University of Pennsylvania, Philadelphia, USA}
  \icmlaffiliation{u2}{Institute of Natural Sciences \& School of Mathematical Sciences, Shanghai Jiao Tong University, Shanghai, China}

  \icmlcorrespondingauthor{Hongkang Li}{lihk@seas.upenn.edu}

  \icmlkeywords{Machine Learning, ICML}

  \vskip 0.3in
]



\printAffiliationsAndNotice{}  

\begin{abstract}

Transformer-based diffusion models have demonstrated remarkable performance at generating high-quality samples. However, our theoretical understanding of the reasons for this success remains limited. For instance, existing models are typically trained by minimizing a denoising objective, which is equivalent to fitting the score function of the training data. However, we do not know why transformer-based models can match the score function for denoising, or why gradient-based methods converge to the optimal denoising model despite the non-convex loss landscape. To the best of our knowledge, this paper provides the first convergence analysis for training transformer-based diffusion models. More specifically, we consider the population Denoising Diffusion Probabilistic Model (DDPM) objective for denoising data that follow a \textit{multi-token Gaussian mixture} distribution. We theoretically quantify the required number of tokens per data point and training iterations for the global convergence towards the Bayes optimal risk of the denoising objective, thereby achieving a desired score matching error. A deeper investigation reveals that the self-attention module of the trained transformer implements a \emph{mean denoising} mechanism that enables the trained model to approximate the oracle Minimum Mean Squared Error (MMSE) estimator of the injected noise in the diffusion steps. Numerical experiments validate these findings.

\end{abstract}






\input{introduction}
\input{formulation}
\input{theory}

\input{experiment}

\vspace{-2mm}
\section{Conclusion}
\vspace{-2mm}
This paper provides a global convergence and score-learning analysis for a one-layer, single-head nonlinear Transformer in diffusion model training. This work also offers a theoretical understanding of how Transformer models learn the oracle MMSE estimator of the training problem through the mean denoising mechanism. Future directions include analyzing and designing different sampling strategies, optimization algorithms, and diffusion model frameworks.

\newpage

\section*{Impact Statement}

This paper aims to explore the convergence analysis and the denoising mechanism of diffusion model parameterized by Transformers. The primary focus is on the mathematical analysis of convergence and training dynamics. To the best of our knowledge, no potential societal consequences are associated with our work.

\section*{The Use of Large Language Models}
We used large-language models (ChatGPT) to help polish the writing of this paper.

\nocite{langley00}


\input{main.bbl}

\input{appendix}


\end{document}

%% file: math_commands.tex
\usepackage{amsmath,amsfonts,bm}
\usepackage{amsthm}
\usepackage{amssymb}
\usepackage{mathrsfs}

\usepackage{graphicx}
\usepackage{subfigure}

\newtheorem{lemma}{Lemma}
\newtheorem{thm}{Theorem}
\newtheorem{rmk}{Remark}
\newtheorem{prpst}{Proposition}

\newtheorem{definition}{Definition}
\newtheorem{corollary}{Corollary}

\def\bfa{{\boldsymbol a}}

\def\bff{{\boldsymbol f}}

\def\bfk{{\boldsymbol k}}

\def\bfm{{\boldsymbol m}}

\def\bfs{{\boldsymbol s}}

\def\bfv{{\boldsymbol v}}

\def\bfx{{\boldsymbol x}}

\def\bfz{{\boldsymbol z}}
\def\bfmu{{\boldsymbol \mu}}

\def\bfA{{\boldsymbol A}}

\def\bfE{{\boldsymbol E}}

\def\bfI{{\boldsymbol I}}

\def\bfM{{\boldsymbol M}}

\def\bfW{{\boldsymbol W}}
\def\bfX{{\boldsymbol X}}
\def\bfY{{\boldsymbol Y}}
\def\bfZ{{\boldsymbol Z}}

\newcommand{\be}{\begin{equation}}
\newcommand{\ee}{\end{equation}}

\def\eqref#1{equation~\ref{#1}}

\def\1{\bm{1}}

\DeclareMathAlphabet{\mathsfit}{\encodingdefault}{\sfdefault}{m}{sl}
\SetMathAlphabet{\mathsfit}{bold}{\encodingdefault}{\sfdefault}{bx}{n}

\DeclareMathOperator*{\argmin}{arg\,min}

%% file: introduction.tex
\vspace{-2mm}
\section{Introduction}\label{sec: intro}
\vspace{-2mm}

Diffusion models \citep{SWMG15, HJA20, SE19, SSKK21} have achieved state-of-the-art performance across a wide range of generative AI tasks, including the creation of images \citep{RBLE22, PX23}, videos \citep{BCTH24, XFCD24}, audio \citep{KPHZ21, ZZZZ23}, text \citep{SASG24, ASGY25}, scientific data \citep{HSVW22, LBBS24, PSAA25}, and multi-modal content \citep{RMYH23, CCLL25}. A classical diffusion model consists of two stages: a forward process and a backward process. The forward process gradually transforms data into noise by adding white Gaussian noise, while the backward process learns a score-based model to progressively remove the injected noise and generate samples from noisy inputs. Specifically, score-based models are typically formulated as neural networks trained to approximate the score function, namely the gradient of the logarithm of the probability density of the data at each time step. 

Among the various score-based diffusion models, the denoising diffusion probabilistic model (DDPM) \citep{HJA20} proposes a canonical training objective as the foundation of diffusion model training, which is to predict the added noise in the forward process. Early score-based generative models \citep{HJA20, SSKK21} adopt convolution-based models, such as U-Net, as the backbone architectures for learning the score function. More recently, motivated by their superior scalability and stronger performance in visual generation, transformer-based architectures, such as DiT \citep{PX23}, have served as effective alternatives. However, despite the remarkable empirical success of transformer-based diffusion models for score learning, 
the theoretical reasons for this success are much less explored.
These include fundamental questions such as:

\smallskip
\begin{tcolorbox}[before skip=1mm, after skip=0.0cm, boxsep=0.0cm, middle=0.0cm, top=0.1cm, bottom=0.1cm]
\textit{\textbf{(Q1)} Why can a nonlinear transformer match the score function and denoise?}
\end{tcolorbox}

\begin{tcolorbox}[before skip=1mm, after skip=1mm, boxsep=0.0cm, middle=0.0cm, top=0.1cm, bottom=0.1cm]
\textit{\textbf{(Q2)} Why gradient descent converges to the optimal nonlinear transformer under DDPM training?}
\end{tcolorbox}

\smallskip


\begin{table*}[htbp]
    \centering
    \begin{tabular}{M{2.8cm} M{2.8cm} M{2.8cm} M{3.1cm} M{3.3cm}  }
    \toprule
     Theoretical Works   & Network Model & Loss Landscape & Convergence Analysis   & Denoising Mechanism\\
    \midrule
        \citet{WZZC24}  & U-Net  & 
        Global optimum & \ding{55} &PCA\\
     \citet{HRX24} & Fully-connected & N/A   & \ding{51} & N/A\\
    \citet{WHT24} & Fully-connected & N/A & \ding{51} & N/A\\
    \citet{HHCZ25} & Convolutional & Stationary point  & \ding{55} & Balanced FL \\
    \hline
        Ours & Transformer & Global optimum &\ding{51}
         & Mean denoising \\
         \bottomrule
    \end{tabular}
    \caption{Comparison with existing works about training analysis and the optimality of denoising of diffusion models.}
    \label{tab:comparison}
    \vspace{-4mm}
\end{table*}

Existing theoretical work addresses these questions only in a limited and separate manner. One line of work \citep{WZZC24, LDQ24} studies the optimal DDPM denoiser under specific data distributions via loss landscape analysis, but it does not establish convergence guarantees for training algorithms. Another line of work studies the training dynamics of neural networks for score matching \citep{HRX24, WHT24, HHCZ25, WP25, BUBM25}, but only for simple architectures or unrealistic regimes. As far as we know, none of these studies investigates the convergence of training algorithms or the learned denoising mechanism for Transformer-based diffusion models. Please see Section~\ref{ssec: related work} for a more detailed comparison between our work and several representative papers, and Table \ref{tab:comparison} for a summary.

\begin{figure}
       \vspace{-2mm}
        \centering
        \includegraphics[width=0.35\textwidth]{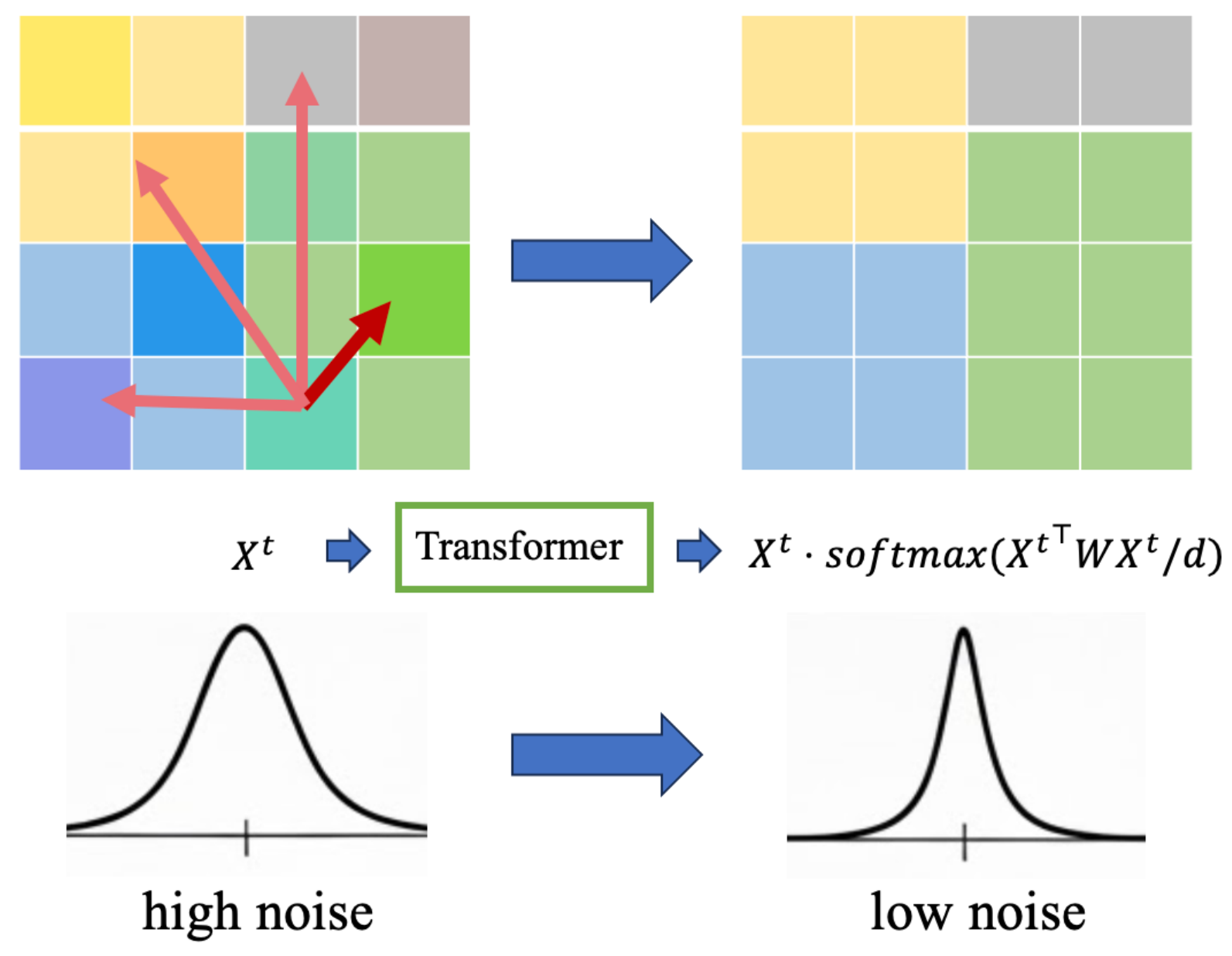}
        \vspace{-0mm}
        \caption{\footnotesize{Mean denoising mechanism by the trained Transformer. Attention reduces the noise added to the data. Dark (light) red arrows: attention weights between the query and key that share the same (different) pattern. } }\label{figure: example}

\vspace{-4mm}
\end{figure}

\vspace{-2mm}
\subsection{Main Contributions}
\vspace{-2mm}

To the best of our knowledge, this paper is the first to analyze the training dynamics of nonlinear transformers trained by gradient descent on the DDPM loss, providing theoretical convergence and score matching error guarantees. Motivated by empirical observations that data are composed of multiple patterns, we consider a Multi-Token Gaussian Mixture (MTGM) data distribution, where each data point consists of multiple tokens that are sampled from a given set of Gaussian mixture components. 
Following prior work \citep{AL23, LWLC23, LWLC24, JHZS24, LWLC24_cot, HHCZ25}, we characterize how the transformer parameters learn the mean patterns of different Gaussian components through gradient updates. 
Our main contributions include:
\begin{enumerate}[leftmargin=*, itemsep=0cm]
    \item \textbf{A quantitative analysis of how to optimize the DDPM loss with transformers towards convergence.} We theoretically analyze the training dynamics on a one-layer single-head transformer with softmax attention and quantify the number of training iterations and tokens per data required to optimize the DDPM loss. Our results characterize how the convergence is affected by the imbalance among the proportions of different Gaussian components in the MTGM distribution, the number of distinct pattern types present in each data, and the time-averaged signal-to-noise ratio of the diffusion noising process.
    
    \item \textbf{Theoretical characterization of how the trained transformer learns the oracle MMSE estimator as the optimal denoiser.} The major technical difficulty in analyzing why neural networks can converge to the optimal denoiser is that the true MMSE estimator for the DDPM loss is intractable to compute. To address this challenge, we define an \textit{oracle MMSE estimator}, which is computed with the class of the Gaussian component of each token in the MTGM data as known. We then prove that the trained transformer can converge to this estimator. Moreover, we show that the oracle denoising risk corresponding to the oracle MMSE estimator is close to the true Bayes risk of the training problem if the number of tokens in each data is large enough. This implies that the training process globally converges to the optimal denoiser, and consequently, the trained model can be used to construct a score network that enables score matching.



\item \textbf{Theoretical understanding of how the self-attention structure performs denoising through a mean denoising mechanism.} The key challenge in characterizing how a transformer approximates the oracle MMSE estimator lies in explaining how the model parameters learn the mean patterns of an MTGM data distribution. We are the first to propose a \textit{mean denoising} mechanism of self-attention in diffusion model training, i.e., attention aggregates queries and keys that share the same pattern, thereby producing a minimum-variance unbiased estimator (MVUE) of the mean pattern and effectively removing the noise injected by the forward diffusion process. This mechanism also enables the trained model to denoise data with the same Gaussian components but shifted mixture proportions, as long as each data point contains a sufficient number of tokens.
\end{enumerate}

\vspace{-2mm}
\subsection{Related Work}\label{ssec: related work}
\vspace{-2mm}

\myparagraph{Theoretical analysis of diffusion models}
Recent work \citep{WZZC24,LDQ24} analyzes the landcape of the DDPM denoiser and shows that the optimal diffusion model essentially performs principal component analysis (PCA) for Gaussian data or low-rank Gaussian mixture data. While these results theoretically characterize the structure of optimal solutions, they do not establish whether such solutions are attainable through gradient-based training of neural networks. \citet{WP25, BUBM25, HRX24, WHT24} analyze the training dynamics of denoising score matching by considering linear models or by adopting theoretical assumptions that reduce nonlinear models to linear ones, such as random feature models~\cite{RR07} or the neural tangent kernel regime~\cite{JGH18}. \citet{BJTZ25} study optimizing the score matching loss over infinitely wide shallow networks. Only \citet{HHCZ25} prove that, under the DDPM loss, a diffusion model parameterized by a two-layer convolutional neural network learns data features and noise to the same order, a phenomenon referred to as the balanced feature learning (FL) mechanism. However, their analysis 
does not provide a convergence guarantee for training dynamics and does not analyze transformers. Other works study convergence guarantees of DDPM samplers in terms of total variation or KL divergence \citep{LLZB23, CCLL23, ADR24, HWC24, LY25, LHC25} or the generalization of diffusion models \citep{BUBM25, LZLC25, SFWM25, PRNZ25}. These works do not involve the model training analysis and therefore differ from the focus of our paper.

\myparagraph{Optimization and generalization of transformers} Many works study the optimization and generalization of transformers for supervised learning tasks. \citet{JSL22, LWLC23, LWMS24, HWCL24, JHZS24} study the convergence of Transformer with a generalization guarantee for binary classification or linear regression tasks via feature learning analysis. 
\citet{TLTO23,TLZO23} show the gradient updates of weights or prompts converge to a max-margin SVM solution. None of these works involves convergence analysis of denoising tasks.   

%% file: formulation.tex
\vspace{-2mm}
\section{Problem Formulation}
\vspace{-2mm}

\myparagraph{Data distribution}
Each data point $\bfX=[\bfx_1, \cdots, \bfx_P]\in\mathbb{R}^{d\times P}$ contains $P$ \emph{tokens} $\bfx_1, \cdots, \bfx_P$ in $\mathbb{R}^d$, each one sampled i.i.d.\ from 
a \textit{Multi-Token Gaussian Mixture (MTGM)} distribution $\mathcal{D}(\tilde{\boldsymbol{\pi}}, K,\{\bfmu_i\}_{i=1}^M, \rho^2)$, where $\{\bfmu_i\in\mathbb{R}^d\}_{i=1}^M$ is a set of $M$ orthogonal \emph{patterns}, i.e., $\bfmu_i^\top\bfmu_j=\sqrt{d}\delta_{ij}$; $\tilde{\boldsymbol{\pi}}\in\Delta^{M-1}$ is a vector in the probability simplex such that $\boldsymbol{1}^\top\tilde{\boldsymbol{\pi}}=1$ and $\min_{m\in [M]}\tilde{\pi}_m>0$; 
$K\leq M$ is the number of distinct patterns in $\bfX$; and $\rho^2=\Theta(1)$ is the variance of Gaussian components.
Specifically,
\begin{definition}\label{def: data dist}
    $\bfX=(\bfx_1,\cdots,\bfx_P)\sim\mathcal{D}(\tilde{\boldsymbol{\pi}}, K,\{\bfmu_i\}_{i=1}^M, \rho)$ is sampled according to the following procedure:
    \begin{enumerate}[leftmargin=*, itemsep=0cm]
        \item Sample $\bfZ\!\sim\! \mathrm{Unif}(\{\bfz\!\in\!\{0,1\}^{M}: 0\!<\!\|\bfz\|_0\!=\!K\!<\!M\})$;
        \item Let $\boldsymbol{\pi}(\bfZ)\!=\!(\pi_1,\cdots,\pi_M)$ with $\pi_i(\bfZ)=[\bfZ]_i \tilde{\pi}_i/\bfZ^\top\tilde{\boldsymbol{\pi}} $;
        \item For each $p\!\in\![P]$, sample $Y_p|\bfZ\!\sim\!\mathrm{Categorical}_M(\boldsymbol{\pi}(\bfZ))$, and then sample $\bfx_p|Y_p\!\sim\! \mathcal{N}(\bfmu_{Y_p}, \rho^2\bfI)$.
    \end{enumerate}  
\end{definition}

\begin{rmk}
Definition \ref{def: data dist} is an extension of the Gaussian mixture distribution. When $K=1$ and $P=1$, each data point contains one token sampled from a Gaussian Mixture Model with mutually orthogonal cluster centers, which is frequently used in theoretical studies of training neural networks for classification tasks \citep{MV25, SZYS25}. When $K>1$, each data point consists of multiple patterns, 
a common, albeit simplified, assumption in computer vision to model an image composed of $K$ out of $M$ possible objects.
\end{rmk}

\myparagraph{Learning model} Let $\bff(\Psi;\bfX,t)\in\mathbb{R}^{d\times P}$ be the output of the learning model, where $\Psi$ is the set of model parameters, $\bfX$ is the input to the model, and $t$ is the diffusion time step. We assume $\bff$ is a one-layer single-head transformer with parameters $\Psi=\{\bfW,\{v_t\}_{t=1}^T\}\in \mathbb{R}^{d\times d}\times \mathbb{R}^T$, i.e.,
\vspace{-1mm}
\begin{equation}
\begin{aligned}
    \bff(\Psi;\bfX, t)\!=&(\bff(\Psi;\bfX, t)_1,\cdots, \bff(\Psi;\bfX, t)_P),\\
    \bff(\Psi;\bfX, t)_p\!=&v_t(\bfX-\!\bfX\mathrm{softmax}(\frac{{\bfX}^\top\bfW\bfx_p}{d})),\label{eqn: model}
\end{aligned}
\end{equation}
where $\mathrm{softmax}([a_1,\!\cdots\!,a_P]^\top\!)\!=\![e^{a_1},\!\cdots\!,e^{a_P}\!]^\top\!/(\sum_je^{a_j}\!)$. This softmax function applies to a matrix column-wise, i.e., $\mathrm{softmax}([\bfa_1,\cdots,\bfa_P])=[\mathrm{softmax}(\bfa_i)]_{i\in [P]}$.

\myparagraph{Training objective and algorithm} We use the training objective of the Denoising Diffusion Probabilistic Model (DDPM)~\citep{HJA20}. For a given time step $t\in[T]$ and an input data point $\bfX^0$, we sample 
\vspace{-1mm}
\begin{equation}
    \bfX^t=\sqrt{\bar{\alpha}_t}\bfX^0+\sqrt{1-\bar{\alpha}_t}\bfE,\label{eqn: Xt}
\end{equation}
where $\bfE=(\boldsymbol{\epsilon}_1,\cdots,\boldsymbol{\epsilon}_P)\in\mathbb{R}^{d\times P}$ is the additive white Gaussian noise with $\boldsymbol{\epsilon}_p\overset{i.i.d.}{\sim} \mathcal{N}(0,\bfI_d)$, and $\{\bar{\alpha}_t\}_{t=1}^T$ is the pre-determined noise scheduling coefficients. 
Given a sample data point $\bfX^0\sim\mathcal{D}(\tilde{\boldsymbol{\pi}}, K, \{\bfmu_i\}_{i=1}^M, \rho)$ and time step $t\sim \mathrm{Unif}([T])$, $\bfX^t$ is obtained from (\ref{eqn: Xt}). We then minimize the following per-dimension DDPM loss in expectation as introduced in \citet{HJA20, BUBM25}:
\vspace{-1mm}
\begin{equation}
\begin{aligned}
 L(\Psi)= \sum_{t=1}^T\mathbb{E}_{\bfX^0, \bfE}[\|\bff(\Psi;\bfX^t, t)-\bfE\|_F^2/(2dPT)].\label{eqn: objective}
\end{aligned}
\end{equation}
The above learning objective (\ref{eqn: objective}) is minimized via gradient descent with a learning rate $\eta>0$. That is, at the training step $s=0$, we set $\bfW^{(0)}=0$ and randomly initialize each $v_t^{(0)}$, $t\in [T]$ from $\mathcal{N}(0, 1/d)$. Then, for each training iteration $s$, the parameters are updated as follows:
\vspace{-1mm}
\begin{equation}\label{eqn:gd}
\begin{aligned}
    \bfW^{(s+1)}&=\bfW^{(s)}-\eta\nabla_\bfW L(\Psi),\\
    v_t^{(s+1)}&=v_t^{(s)}-\eta\nabla_{v_t} L(\Psi), \forall t\in [T].
    \end{aligned}
\end{equation}

\myparagraph{Score matching} Let $p_t(\bfX^t)$ be the probability density function of $\bfX^t$ at time step $t\in [T]$. The \emph{score function} is defined as $\bfs(\bfX^t,t)=\nabla_{\bfX^t}\log p_t(\bfX^t)$. The goal of score matching is to train a neural network $\bfs_\theta(\bfX^t,t)$ parameterized by $\theta$ that minimizes the score matching error:
\vspace{-1mm}
\begin{equation}
    \!\!\mathcal{E}(\theta)\!=\!\!\sum_{t=1}^T\mathbb{E}_{\bfX^0,\bfE}[\|\bfs_\theta(\bfX^t,t)-\bfs(\bfX^t,t)\|_F^2/(2dPT)].\label{eqn: score matching}
\end{equation}
Note that we use the per-dimension score matching error in (\ref{eqn: score matching}) following \citep{BUBM25}.

%% file: theory.tex
\vspace{-2mm}
\section{Main Theoretical Results}
\vspace{-2mm}
Let us first introduce the Bayes risk for the DDPM objective.
\begin{definition}\label{def: bayes}
The \textbf{Bayes denoising risk} for the MTGM data model in Definition \ref{def: data dist} and $t\sim \mathrm{Unif}([T])$ is defined as:
    \vspace{-1mm}
    \begin{equation}\label{eqn: bayes_risk}
        R_{\textrm{Bayes}}:=\mathbb{E}_{\bfX^0,\bfE,t}[\|\bfE-\mathbb{E}[\bfE|\bfX^t]\|_F^2/(2dP)]\,.
    \end{equation}
\end{definition}

$R_{\textrm{Bayes}}$ is the optimal risk achieved if one can minimize the DDPM loss over all possible denoising models. Our main results will show that under certain conditions, the gradient descent in (\ref{eqn: iteration}) can learn a transformer model $\bff(\Psi;\bfX,t)$ that attains a risk of $R_{\textrm{Bayes}}+O(\epsilon)$ for any arbitrarily small $\epsilon$.

To state our results, we need some additional notation. For $\bfX^0\sim \mathcal{D}(\tilde{\boldsymbol{\pi}}, K, \{\bfmu_i\}_{i=1}^M, \rho)$ with $\bfY$ the latent variable in Definition \ref{def: data dist}, let the \textit{minimal average pattern ratio} for pattern $u\in [M]$ and the \textit{pattern imbalance ratio} be defined~as:
\vspace{-4mm}
\begin{align}
\!\!
    \nu_{\min}^{\tilde{\boldsymbol{\pi}}}(K)&=\min_{u\in[M]}\nu_u^{\tilde{\boldsymbol{\pi}}}(K):=\mathbb{E}_{\bfX^0}\Big[\sum_{p=1}^P \mathbbm{1}[Y_p=u]/P\Big]     \!\!
\\
    \delta(\tilde{\boldsymbol{\pi}})&=\min_{u\in [M]}\tilde{\pi}_u/\max_{u\in [M]}\tilde{\pi}_u.
\end{align}
\vspace{-4mm}

The former is the minimum average probability of selecting a pattern, and the latter measures the degree of imbalance between the prior probabilities of different patterns. Since $\bar{\alpha}_t/(1-\bar{\alpha}_t)$ is computed as the signal-noise ratio at time step $t$ by \citep{L22}, we denote $\textrm{SNR}=\mathbb{E}_{t} [\bar{\alpha}_t/(1-\bar{\alpha}_t)]$ as the \textit{time-averaged SNR} over the noise schedule. 

With the above definitions and notations, we now state the following theoretical result about the convergence of diffusion model training.

\begin{thm}[Convergence]\label{thm: training}
For any $\epsilon\in (0, \delta(\tilde{\boldsymbol{\pi}})^{\Theta(1)})$, if 
(i) the dimension $d\geq \Omega(\epsilon^{-1}\log (\epsilon^{-1}\nu_{\min}^{\tilde{\boldsymbol{\pi}}}(K)^{-1}))$, 
(ii) the number of diffusion steps $T\geq\Omega(\log d)$, 
(iii) the data distribution satisfies $P\geq \Omega(\nu_{\min}^{\tilde{\boldsymbol{\pi}}}(K)^{-1} (\rho^2+1)\epsilon^{-1}\log d)$, 
(iv) the algorithm in (\ref{eqn:gd}) is run with a step size $\eta\leq O((\max\{\rho,1\}+\epsilon)^{-1})$ and (v) with a number of iterations 
\vspace{-1mm}
\begin{equation}
\begin{aligned}
    S=&\;\Omega\big( (\epsilon^{-1}+\nu^{\tilde{\boldsymbol{\pi}}}_{\min}(K)^{-3})\eta^{-1}\nu^{\tilde{\boldsymbol{\pi}}}_{\min}(K)^{-1}\\
    &\qquad\cdot \textrm{SNR}^{-3}+\log(\rho^2+1)\epsilon^{-1}\big),
    \end{aligned}\label{eqn: iteration}
\end{equation}
then with high probability over the Gaussian random initialization, the learned model with parameters $\Psi^{(S)}$ satisfies
\vspace{-1mm}
\begin{equation}
\begin{aligned}
    L(\Psi^{(S)})
    \leq R_{\textrm{Bayes}}+O(\epsilon).\label{eqn: thm1 loss}
    \end{aligned}
\end{equation}

\end{thm}

Theorem \ref{thm: training} shows that, if each data contains a \emph{sufficiently large number of tokens} and \emph{the number of training iterations is large enough}, then the one-layer single-head transformer diffusion model trained by gradient descent (\ref{eqn:gd}) will achieve a DDPM loss that deviates from the Bayes denoising risk $R_{\textrm{Bayes}}$ by only $O(\epsilon)$. 

We shall elaborate those conditions. The requirement of the dimension in condition (i) is to ensure that the difference between attention weights of query-key pairs with the same pattern is small, as will be described in Section \ref{ssec: mechanism}. 

\myparagraph{Number of tokens} The required number of tokens per data point in condition (iii) scales linearly in $\nu^{\tilde{\boldsymbol{\pi}}}_{\min}(K)^{-1}$ and $\rho^2$. Therefore, a less uniform distribution over the patterns and a higher noise level in a single data point increases the complexity of the denoising task, requiring more tokens in the data to learn an effective denoising model. The reason why more tokens per data benefits denoising will be explained in more detail in Section \ref{sec: analysis}.

\myparagraph{Number of iterations} The required number of GD iterations in condition (v) scales polynomially in $\nu^{\tilde{\boldsymbol{\pi}}}_{\min}(K)^{-1}$ and $\textrm{SNR}^{-1}$. This means that a more uniform distribution over the patterns and a larger time-averaged SNR in the forward process help the model parameters learn all the patterns so that the self-attention identifies tokens sampled from the same Gaussian cluster, which is a core mechanism for denoising the MTGM data (discussed in Section \ref{ssec: mechanism}).   
Moreover, we can derive the following simplification for (\ref{eqn: iteration}) regarding different choices of $K\in [M]$.
\begin{corollary}\label{cor: K}
    (a) When $K=1$, the number of iterations reaches its minimum, which leads to $S=\Omega(\epsilon^{-1}\eta^{-1} M\cdot \textrm{SNR}^{-3})$. (b) When $K=M$, the required number of iterations reaches its maximum, which results in $S=\Omega((\epsilon^{-1}+ \min_{u\in [M]}\{\tilde{\pi}_u\}^{-3})\eta^{-1} \min_{u\in [M]}\{\tilde{\pi}_u\}^{-1}\textrm{SNR}^{-3})$. 
\end{corollary}

Notice that $\min_{u\in [M]}\{\tilde{\pi}_u\}<1/M$, which leads to $M^{-1}\min_{u\in [M]}\{\tilde{\pi}_u\}^{-1}>1$, then Corollary \ref{cor: K} indicates that the required number of iterations for $K=M$ is at least $\Omega(M^{-1}\min_{u\in [M]}\{\tilde{\pi}_u\}^{-1})$ times larger than that for $K=1$. Therefore, the simplest pattern structure in the data, i.e., the case of $K=1$, leads to the fastest convergence. This is also aligned with the previous intuition that more diverse patterns in each data increase the complexity of the denoising task, making the training more challenging. 



\myparagraph{Constructing score model from learned denoiser} Next, we show how the trained transformer model can achieve 
a desired score matching error. 

\begin{thm}[Score Matching]\label{thm: score matching}
    Given the trained model in Theorem \ref{thm: training} with parameters $\Psi^{(S)}$ that satisfies (\ref{eqn: thm1 loss}) for some $\epsilon\in (0,\delta(\tilde{\boldsymbol{\pi}})^{\Theta(1)})$, we can construct 
    \vspace{-1mm}
    \begin{equation}
        s_\theta(\bfX^t, t)=s_{\Psi^{(S)}}(\bfX^t, t)=-\frac{\bff(\Psi^{(S)};\bfX^t,t)}{\sqrt{1-\bar{\alpha}_t}},\label{eqn: score network}
    \end{equation}
    \vspace{-1mm}
    with $\theta=\Psi^{(S)}$, such that
    \begin{equation}
        \mathcal{E}(\theta)=\mathcal{E}(\Psi^{(S)})\leq \epsilon\cdot (\textrm{SNR}+1).
    \end{equation}
\end{thm}


Theorem \ref{thm: score matching} shows that a model trained under the conditions in Theorem \ref{thm: training} can be directly used to construct a score network that can achieve a score matching error of $O(\epsilon)$. Note that the construction in (\ref{eqn: score network}) is typically used to fit the conditional score function $\nabla_{\bfX^t} \log p_t(\bfX^t|\bfX^0)$. Theorem \ref{thm: score matching} demonstrates that, under our problem setting, (\ref{eqn: score network}) can also match $\nabla_{\bfX^t} \log p_t(\bfX^t)$ with an error that is close to $0$.

\vspace{-2mm}
\section{In-Depth Analysis of Convergence and Denoising With the Trained Transformer}\label{sec: analysis}
\vspace{-2mm}

This section investigates why the trained Transformer can reduce the DDPM loss in (\ref{eqn: objective}) to the Bayes denoising risk and enable score learning, as stated in Theorem \ref{thm: training} and Theorem \ref{thm: score matching}. In Section \ref{ssec: mmse}, we show that
the Transformer $\bff(\Psi)$ learns the ``oracle MMSE estimator" through GD training, thereby achieving the oracle denoising risk up to an excessive $O(\epsilon)$ risk.
In Section \ref{ssec: mechanism}, we show that the learned self-attention structure exhibits a mean denoising mechanism to enable denoising and score matching on the MTGM data even with different Gaussian mixture proportions from training. 


\vspace{-2mm}
\subsection{What Does the Trained Transformer Converge to?}\label{ssec: mmse}
\vspace{-2mm}

The optimal model that achieves the Bayes denoising risk in (\ref{def: bayes}) is computed as $\mathbb{E}[\bfE|\bfX^t]$, which is the Minimum Mean Squared Error (MMSE) estimator of the added Gaussian noise given a noisy input at time step $t$. However, computing $\mathbb{E}[\bfE|\bfX^t]$ for the MTGM data is challenging due to the highly complicated probability density function of $p_t(\bfX^t)$. In our problem setting, we define another estimator of the added noise with the data mean known, which is easier to obtain and does not need the knowledge of $p_t(\bfX^t)$. 
The specific definition is as follows. 

\begin{definition}[Oracle MMSE estimator and denoising risk]\label{def: oracle}
    With Definition \ref{def: data dist}, let $\bfM_{\bfY}=(\bfmu_{Y_1}, \cdots, \bfmu_{Y_P})\in\mathbb{R}^{d\times P}$ be the matrix of mean patterns given $\bfY$. Then, for a noisy data $\bfX^t$ obtained from some $\bfX^0$ with latent variable $\bfY$, we define the \textbf{oracle MMSE estimator} of $\bfE$ given $\bfM_{\bfY}$ as 
    \vspace{-1mm}
    \begin{equation}
        \mathbb{E}[\bfE|\bfX^t, \bfM_{\bfY}]=\frac{\sqrt{1-\bar{\alpha}_t}}{1-\bar{\alpha}_t+\rho^2\bar{\alpha}_t}(\bfX^t-\sqrt{\alpha}_t\bfM_{\bfY}).\label{eqn: mmse}
    \end{equation}
    We define the \textbf{oracle denoising risk}, denoted by $R_{\textrm{oracle}}$, as the DDPM loss with the oracle MMSE estimator as the denoising model for $t\sim \mathrm{Unif}([T])$, which is computed as
    \vspace{-1mm}
    \begin{equation}
    \begin{aligned}\label{eqn: oracle_risk}
        \!\!
        R_{\textrm{oracle}}\!:=\!\mathbb{E}_{\bfX^0,\bfE,t}[\|\bfE-\mathbb{E}[\bfE|\bfX^t,\bfM_{\bfY}]\|_F^2/(2dP)].
        \end{aligned}
    \end{equation}
\end{definition}
Note that $R_{\textrm{oracle}}$ is a lower bound of $R_{\textrm{Bayes}}$, because providing the prior knowledge of $\bfM_{\bfY}$ gives more information than conditioning on $\bfX^t$ alone.  Under squared loss, more information cannot worsen the optimal estimator.

\textbf{Convergence to the oracle MMSE estimator. } We then show that a one-layer, single-head Transformer trained under conditions in Theorem \ref{thm: training} approximates the oracle MMSE estimator. The following proposition reveals the implicit mechanism learned by the trained model.
\begin{prpst}\label{prpst: mmse}
    Given training conditions (i)-(v) in Theorem \ref{thm: training}, with a high probability over random initialization, the algorithm in (\ref{eqn:gd}) returns a model with parameters $\Psi^{(S)}=\{\bfW^{(S)},\{v_t^{(S)}\}_{t=1}^T\}$ with the following properties:
    \vspace{-2mm}
    \begin{enumerate}[leftmargin=*, itemsep=0cm]
        \item The self-attention module satisfies that with a high probability over the sampling of the clean data $\bfX^0$ and the noise $\bfE$, for any computed noisy data $\bfX^t, t\in[T]$ (together with latent variable $\bfY$), we have
        \vspace{-1mm}
        \begin{equation}
    \begin{aligned}
        & 
        \frac{\|\sqrt{\bar{\alpha}}_t\bfM_{\bfY}-\bfX^t\mathrm{softmax}(\bfX^t{}^\top\bfW^{(S)}\bfX^t/d)\|_F^2}{dP}\\
        &\leq  O((\rho^2+1)\log d/(P\nu_{\min}^{\tilde{\boldsymbol{\pi}}}(K))),\label{eqn: trained W}
        \end{aligned}
    \end{equation}
    \vspace{-4mm}
    \item The output weights satisfy that $\forall t\in[T]$,
    \vspace{-1mm}
    \begin{equation}
        |v_t^{(S)}-\sqrt{1-\bar{\alpha}_t}/(1-\bar{\alpha}_t+\rho^2\bar{\alpha}_t)
        |\leq O(\epsilon).\label{eqn: trained vt}
    \end{equation}
    \end{enumerate}
\end{prpst}
Recall that our Transformer model is defined as $$\bff(\Psi; \bfX^t, t)=v_t\big(\bfX^t - \bfX^t\mathrm{softmax}(\bfX^t{}^\top\bfW\bfX^t/d)\big)\,,$$ then (\ref{eqn: trained W}) and (\ref{eqn: trained vt}) in Proposition \ref{prpst: mmse} show that the trained model $\bff(\Psi^{(S)})$ can approximate (\ref{eqn: mmse}), i.e., the oracle MMSE estimator under known mean pattern $\bfM_{\bfY}$ of each data $\bfX^0$ with a diminishing error. Specifically, (\ref{eqn: trained W}) indicates that the approximation error of $\bfM_{\bfY}$ decreases with an increasing total number of tokens $P$. A larger $P$ can reduce the estimation variance introduced by the noise added through the diffusion process. This means that if condition (iii) in Theorem \ref{thm: training} holds, the trained self-attention structure $\bfX^t\text{softmax}(\bfX^t{}^\top\bfW^{(S)}\bfX^t/d)$ can approximate $\sqrt{\bar{\alpha}_t}\bfM_{\bfY}$ with a squared error of $O(\epsilon)$ per dimension (The key mechanism behind such an approximation is discussed in Section \ref{ssec: mechanism}). Then, (\ref{eqn: trained vt}) implies that the trained $v_t^{(S)}$ for any $t\in[T]$ can approximate the linear coefficient term in(\ref{eqn: mmse}) with an $O(\epsilon)$ error. Finally, one can conclude with (\ref{eqn: trained W}), (\ref{eqn: trained vt}) that $\bff(\Psi^{(S)})$ can approximate the oracle MMSE estimator (\ref{eqn: mmse}) in the sense that they achieve similar denoising risks, i.e.,
\begin{corollary}\label{cor: R_oracle}
    The trained model with parameter $\Psi^{(S)}$ in Proposition \ref{prpst: mmse} satisfies that 
    $L(\Psi^{(S)})\leq R_{oracle}+O(\epsilon)$.
\end{corollary}

\textbf{Trained Transformer achieves near-optimal denoising. } 
Recall that in the DDPM loss (\ref{eqn: objective}) used for training, the mean pattern of each data $\bfM_{\bfY}$ is unknown, so in principle the best attainable loss after training should be $R_{\textrm{Bayes}}$ as in (\ref{eqn: thm1 loss}). However, (\ref{eqn: trained W}) has a straightforward but important implication: since the latent mean patterns $\bfM_{\bfY}$ can be reliably estimated from the observed noisy data $\bfX^t$ at every diffusion time steps when the total number of tokens per data is large, the Bayes risk in (\ref{eqn: bayes_risk}) with only the knowledge of $\bfX^t$ should not be much worse than the oracle one in (\ref{eqn: oracle_risk}). Indeed, it can be shown independently that
\begin{prpst}\label{proposition: risk_diff}
    Given condition (iii) in Theorem \ref{thm: training} hold for some $\epsilon\in(0, \delta^{\Theta(1)})$, we can obtain 
    \vspace{-1mm}
    \begin{equation}
        R_{\textrm{Bayes}}-R_{\textrm{oracle}}\leq O(\epsilon).\label{eqn: R diff}
    \end{equation}
\end{prpst}
Combining Corollary \ref{cor: R_oracle} and Proposition \ref{proposition: risk_diff}, one obtain our main Theorem \ref{thm: training}, showing that the trained transformer achieves near optimal denoising risk.


\begin{rmk}
    As illustrated in this Section \ref{ssec: mmse}, we have taken a novel approach to characterizes the near-optimal denoising model for the MTGM data.  
    Prior work, \citet{WZZC24} for example, considers much simpler data distributions such as low-rank Gaussian mixture model, for which the true MMSE estimator $\mathbb{E}[\bfE|\bfX^t]$ can be analyzed directly. 
    For MTGM data, however, the true MMSE estimator has a complicated expression due to the additional pattern subset selection step in the data sampling procedure (as in Definition \ref{def: data dist}). To address this challenge, our work adopts the oracle MMSE estimator, which has a simple interpretable expression and approximates the true MMSE estimator closely in the large-$P$ regime, as a bridge to characterize the near-optimal denoising model, thereby showing that Transformer can indeed learn the optimal denoiser. 
\end{rmk}

\vspace{-2mm}
\subsection{What Mechanism Does the Trained Transformer Parameters Learn from Diffusion Model Training?}\label{ssec: mechanism}
\vspace{-2mm}

In this section, we delve into the question of what denoising mechanism the Transformer parameters learn during training, which enables the model to approximate the oracle MMSE estimator and exhibit a desired performance of score matching. Recall the expressions for the transformer model in (\ref{eqn: model}) and the oracle MMSE estimator in (\ref{eqn: mmse}), one core question is why the self-attention structure in (\ref{eqn: model}) can approximate the mean patterns $\sqrt{\alpha}_t\bfM_\bfY$ at diffusion time step $t$, which we shall explain carefully next. 

\textbf{Query-key inner products reveal tokens with the same pattern.} First, the following proposition reveals how the trained self-attention behaves on different input data.

\begin{prpst}\label{prpst: W}
    Consider the trained model in Theorem \ref{thm: training} with parameters $\Psi^{(S)}=\{\bfW^{(S)},\{v_t^{(S)}\}_{t=1}^T\}$ that satisfies (\ref{eqn: thm1 loss}) for some $\epsilon\in (0,\delta(\tilde{\boldsymbol{\pi}})^{\Theta(1)})$. With a high probability over the sampling of the clean data $\bfX^0$ and the noise $\bfE$, any computed noisy data $\bfX^t, t\in[T]$ (together with latent variable $\bfY$) satisfies that for any triplet $i,j,k\in[P]$ whose corresponding latent variables satisfy that $Y_i=Y_j\neq Y_k$, we have
    \vspace{-1mm}
    \begin{equation}
        \bfx_j^t{}^\top\bfW^{(S)}\bfx_i^t/d\geq \log \big(\Omega(\epsilon^{-1}K\delta(\tilde{\boldsymbol{\pi}}))\big)/2,\label{eqn: qk same}
    \end{equation}
    \vspace{-1mm}
    \begin{equation}
        |\bfx_k^t{}^\top\bfW^{(S)}\bfx_i^t/d|\leq O(\log d/\sqrt{d})\cdot \bfx_j^t{}^\top\bfW^{(S)}\bfx_i^t/d.\label{eqn: qk similar}
    \end{equation}
\end{prpst}

    Proposition \ref{prpst: W} shows that, on the one hand, if the query and key vectors are two tokens sampled with the same pattern, which is indicated by their latent variables $Y_i=Y_j$, then their inner product admits a large lower bound of order $\log\Omega(\epsilon^{-1})$ after training. On the other hand, if the query and key vectors are tokens with two different patterns, then the absolute value of their inner product is relatively small, which is on the order of $O(d^{-1/2})$ times that of a query–key pair with the same pattern. This implies that, even though each token of $\bfX^t$ contains a large amount of noise injected by the forward diffusion process, the trained self-attention layer can still capture and pair tokens with the same pattern.

\textbf{Softmax attention concentration enables mean denoising.} Based on Proposition \ref{prpst: W}, we then compute the weights output by the softmax attention of the trained model to introduce the mean denoising mechanism in the following corollary.

\begin{corollary}[Mean denoising mechanism]\label{cor: attention}
    For any data $\bfX^0\sim\mathcal{D}(\tilde{\boldsymbol{\pi}},\{\bfmu_i\}_{i=1}^M,\rho)$ with a latent variable $\bfY$, denote $\mathcal{S}_u^{\bfY}=\{p\in [P]: Y_p=u\}$ for any $u\in [M]$. Then, 
    given the same trained model as in Proposition \ref{prpst: W}, with a high probability over the sampling of the clean data $\bfX^0$ and the noise $\bfE$, for any computed noisy data $\bfX^t, t\in [T]$ (together with latent variable $\bfY$), $u\in[M]$, and $p\in \mathcal{S}_u^{\bfY}$, $p', p''\in\mathcal{S}_u^{\bfY}$, $p'\neq p''$, we have 
    \vspace{-1mm}
    \begin{equation}
        \sum_{p'\in \mathcal{S}_u}\mathrm{softmax}(\bfX^t{}^\top\bfW^{(S)}\bfx_p^t/d)_{p'}\geq 1-\sqrt{\epsilon},\label{eqn: attention_concentration}
    \end{equation}
    \vspace{-1mm}
    \begin{equation}
        \begin{aligned}          &\mathrm{softmax}(\bfX^t{^\top\bfW^{(S)}}\bfx_p^t/d)_{p'}\\
        = &(1\pm \Theta(\epsilon))\mathrm{softmax}(\bfX^t{}^\top\bfW^{(S)}\bfx_p^t/d)_{p''}.\label{eqn: attention_uniform}
        \end{aligned}
    \end{equation}
\end{corollary}

    Corollary \ref{cor: attention} shows that:
    When one inspects each column of the softmax attention output of the trained model given a new noisy data $\bfX^t$ whose $p$-th token $\bfx_p^t$ is associated with the latent variable $Y_p=u$, the attention weights are concentrated among all tokens with the same pattern, as characterized by (\ref{eqn: attention_concentration}); Moreover, the attention probabilities are distributed almost uniformly among those tokens, as shown by (\ref{eqn: attention_uniform}).
    As such, for each input token $\bfx_p^t$, 
    the self-attention structure approximately outputs the mean of tokens that share the same pattern, i.e., a minimum-variance unbiased estimator (MVUE) of the mean of the Gaussian component from which $\bfx_p^t$ is sampled. 
    We refer to this attention behavior as the \textit{Mean denoising} mechanism.
    This mechanism suggests that self-attention can estimate $\sqrt{\alpha}_t\bfM_\bfY$ in (\ref{eqn: mmse}) reliably and with minimal bias as long as for every pattern appeared in $\bfM_\bfY$, there is a sufficient number of tokens sampled with that pattern. Consequently, the total number of tokens per data $P$ is required to be large, as stated in (iii) in our Theorem \ref{thm: training}, to achieve near-optimal denoising for the MTGM data.

\textbf{Implication of the mechanism on denoising data with a shifted $\tilde{\boldsymbol{\pi}}$. } The conclusion of Proposition \ref{prpst: W} shows that the trained self-attention module learns all patterns, regardless of their proportions in the data distribution. This motivates our discussion of the generative performance on test data with a shifted pattern proportion parameter $\tilde{\boldsymbol{\pi}}$. 

Specifically, consider $\bfX^0\sim\mathcal{D}(\tilde{\boldsymbol{\pi}}', K, \{\bfmu_i\}_{i=1}^M, \rho)$, where $\tilde{\boldsymbol{\pi}}'$ may not equal to $\tilde{\boldsymbol{\pi}}$, i.e., the fraction of Gaussian components of training data. The DDPM loss in expectation and the score matching error are then computed following (\ref{eqn: objective}) and (\ref{eqn: score matching}), respectively, but based on input distribution parameterized with $\tilde{\boldsymbol{\pi}}'$. We obtain the following corollary. 

\begin{corollary}\label{cor: shift}
    Given the trained model in Theorem \ref{thm: training} with parameters $\Psi^{(S)}$ that satisfies (\ref{eqn: thm1 loss}) for some $\epsilon\in (0,\delta(\tilde{\boldsymbol{\pi}})^{\Theta(1)})$, then for any $\bfX^0\sim\mathcal{D}(\tilde{\boldsymbol{\pi}}', K, \{\bfmu_i\}_{i=1}^M, \rho)$ with the number of tokens $P\geq \Omega(\nu^{\tilde{\boldsymbol{\pi}}'}_{\min}(K)^{-1} (\rho^2+1)\epsilon^{-1}\log d)$ in each data, we have $L(\Psi^{(S)})\!\leq\! R_{oracle}+O(\epsilon),\ \mathcal{E}(\Psi^{(S)})\leq O(\epsilon)$.
\end{corollary}

Corollary \ref{cor: shift} shows that a model trained under the conditions of Theorem \ref{thm: training} can also achieve an $O(\epsilon)$ DDPM loss and score matching error on data with distribution-shifted pattern proportions if the number of tokens per data is large enough. This is because (\ref{eqn: qk same}) shows that the trained self-attention mechanism yields a large lower bound on the inner product between queries and keys that share the same pattern. This bound holds uniformly for all patterns, and therefore the model does not fail to learn a pattern simply because it appears with low frequency. As a result, even when $\tilde{\boldsymbol{\pi}}$ shifts to $\tilde{\boldsymbol{\pi}}'$, as long as each data point contains sufficiently many tokens that scales with $\nu_{\min}^{\tilde{\boldsymbol{\pi}}'}(K)^{-1}$ such that self-attention can denoise by averaging tokens with the same pattern, the model can successfully denoise under the shifted data distribution and achieve effective score learning.

\vspace{-2mm}
\subsection{Proof Idea, Technical Novelty, and Limitations}\label{subsec: proof idea}
\vspace{-2mm}
\textbf{Proof idea of Theorem \ref{thm: training}}. In Lemma \ref{lemma: W training}, we prove that the gradient updates of $\bfW^{(s)}$ along directions corresponding to query–key pairs with the same pattern admit a lower bound, while the norm of the gradient updates along directions corresponding to query–key pairs with different patterns is very small. By accumulating gradients updates over steps, we obtain the mean denoising mechanism described in Proposition \ref{prpst: W} and Corollary \ref{cor: attention}, at which point the optimization of $\bfW^{(s)}$ converges (Lemma \ref{lemma: stage 1}). Note that patterns with smaller fractions are learned more slowly than those with larger fractions. To ensure that the mean denoising mechanism holds for all patterns, the required number of training iterations we derive depends on $\nu_{\min}^{\tilde{\boldsymbol{\pi}}}$, the minimum probability of selecting a pattern in data. The training of $v_t^{(s)}$ is then reduced to a linear problem, and Lemma \ref{lemma: stage 2} provides a proof of convergence. Since the learned parameters are close to the oracle MMSE estimator, we can show that the DDPM loss after training is close to $R_{\text{oracle}}$. Combined with Proposition \ref{proposition: risk_diff}, this yields the global convergence result stated in (\ref{eqn: thm1 loss}).

\textbf{Proof idea of Theorem \ref{thm: score matching}. }The score matching error (\ref{eqn: score matching}) can be decomposed into the error of fitting the conditional score function $\nabla_{\bfX^t} \log p_t(\bfX^t|\bfX^0)$ and the discrepancy between the score function and the conditional score function. The former can be upper bounded by $O(\epsilon)$ since the trained Transformer can approximate the conditional score function by Proposition \ref{prpst: mmse}. The latter can be shown, based on Proposition \ref{proposition: risk_diff}, to be $O(\epsilon)$. Therefore, we can construct a score network in (\ref{eqn: score network}) via the trained model $\Psi$ such that the score matching error is as small as $O(\epsilon)$.

\textbf{Technical novelty.} Our proof technique is inspired by the feature learning technique in studying Transformers. For the first time, we extend their analysis of label-prediction tasks, such as binary classification \citep{LWLC23, LWLC24, JHZS24} and linear regression \citep{ZFB23, HCL23}, to denoising tasks. Our work also extends the mechanism of attention concentration \citep{HCL23, LWLC23, LWLC24} to diffusion models in denoising (Corollary \ref{cor: attention}). Our technique preserves the nonlinearity of diffusion models rather than linearizing the model by an impractical extremely-wide network assumption in \citep{HRX24,WHT24}. This enables a convergence analysis beyond the NTK regime. 

\textbf{Limitations and possible future extensions} Although we consider more complex models and data distributions than those in prior works as discussed in Section \ref{sec: intro}, our analysis is still under a restricted setting: the network model is a one-layer, single-head Transformer, the data follows the MTGM distribution with orthogonal patterns, and the GD algorithm is run on population DDPM loss. We emphasize, however, that the focus of this paper is to provide an initial theoretical understanding of the convergence and denoising mechanisms of Transformer-based diffusion models, which could serve as building blocks for rigorous analysis in more realistic settings. 
One potential future extension is to study the convergence behavior of the multi-head attention Transformer when the data model possesses multiple types of internal relationships among tokens. Another is to extend the convergence results to empirical DDPM losses and analyze the generalization of DDPM by studying the gap between the empirical and population loss.

%% file: experiment.tex
\begin{figure*}[htbp]
\centering
\centerline{
\begin{tabular}{ccccc}
\hspace*{2mm} (A)
\vspace*{-0.7mm}
& \hspace*{2mm} (B)
& \hspace*{2mm} (C)
& \hspace*{2mm} (D)
& \hspace*{2mm} (E)
\\

\includegraphics[width=.18\textwidth]{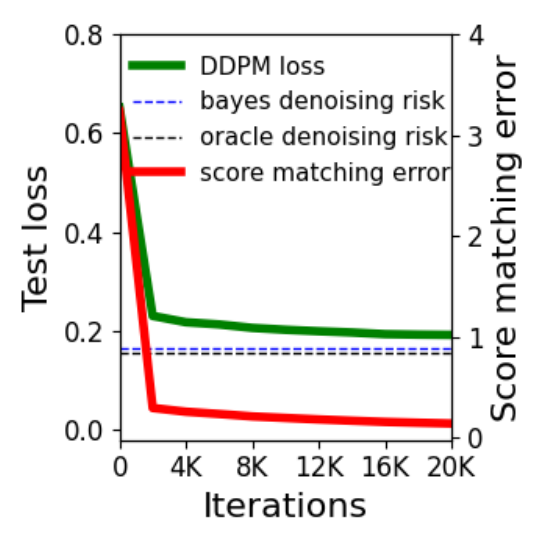}  
&
\hspace*{-1mm}
\includegraphics[width=.18\textwidth]{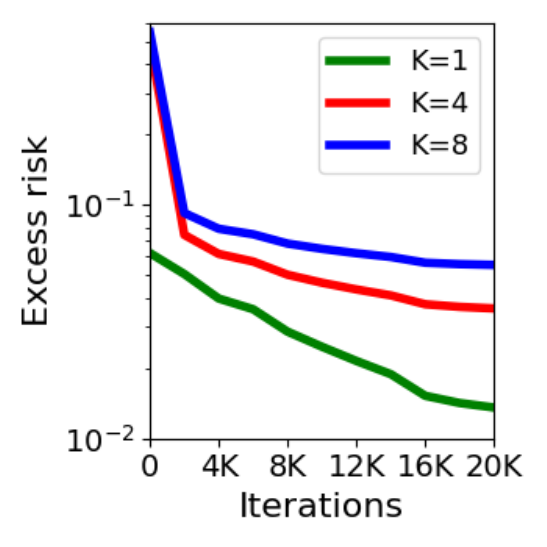}
&
\hspace*{-1mm}
\includegraphics[width=.18\textwidth]{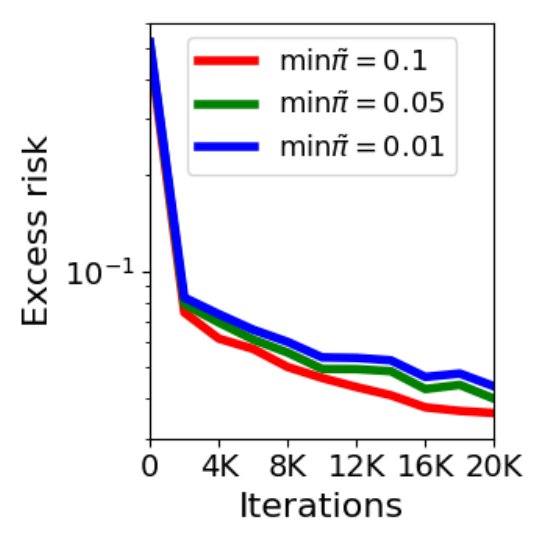}
&
\hspace*{-1mm}
\includegraphics[width=.18\textwidth]{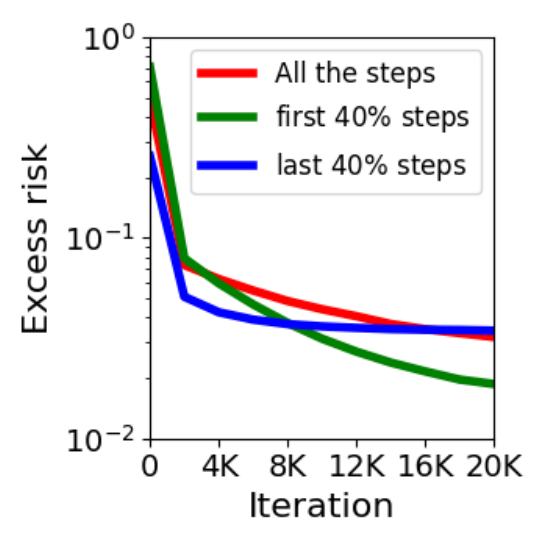}
& 
\hspace*{-1mm}
\includegraphics[width=.18\textwidth]{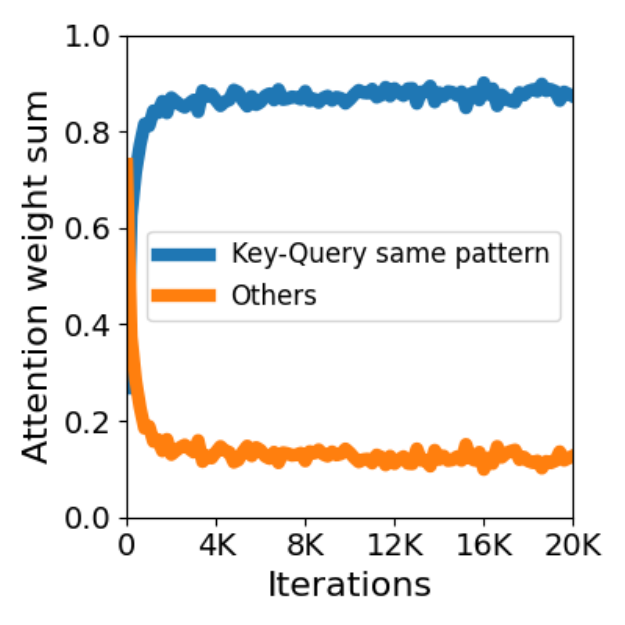}
\end{tabular}}
\vspace*{-3mm}
\caption{\footnotesize{The convergence performance and the attention behavior of the trained model. (A) The green and red curves are the test loss and score matching error during diffusion model training, respectively. Blue dashed line: Bayes denoising risk. Black dashed line: oracle denoising risk. (B) Excess risk with varying $K$, the number of Gaussian components per data.  (C)  Excess risk with varying $\min_{u\in[M]}\tilde{\pi}_u$, i.e., the minimal fraction among all the Gaussian components. A larger $\min_{u\in[M]}\tilde{\pi}_u$ indicates a more uniform distribution of all the patterns. (D) Excess risk with different sampling strategies. Red curve: uniform sampling $t\sim \mathrm{Unif}([T])$. Green curve: sampling from the first $40\%$ time steps, i.e., $t\sim \mathrm{Unif}([1, 0.4\cdot T])$. Blue curve: sampling from the last $40\%$ time steps, i.e., $t\sim \mathrm{Unif}(0.6\cdot T, T)$. 
(E) The attention weight summation on keys with the same pattern as the query and on other keys. }}\label{fig: convergence}
\end{figure*}

\vspace{-2mm}
\section{Numerical Experiments}
\vspace{-2mm}

In this section, we conduct synthetic experiments in Section \ref{ssec: synthetic} and real-data experiments in Section \ref{ssec: real} to justify our findings, respectively. Due to space limitations, some additional experiments are moved to Appendix.

\vspace{-2mm}
\subsection{Synthetic Experiments}\label{ssec: synthetic}
\vspace{-2mm}


\textbf{Setup. }Synthetic data are generated as described in Definition \ref{def: data dist}. Let $d=64$, $M=8$, $P=256$, $\rho=0.3$. If not specified, $K=4$. We consider generating uniform or non-uniform scenarios by varying $\min_{u\in[M]}\tilde{\pi}_u$. A smaller $\min_{u\in[M]}\tilde{\pi_u}$ indicates a more non-uniform distribution among patterns. 
The learning model is a one-layer single-head Transformer as formulated in (\ref{eqn: model}). The total number of time steps is $T=50$. We adopt a linear schedule, i.e., $\bar{\alpha}_t=\prod_{i=1}^t\alpha_i$, where $\alpha_t=\alpha_1-(\alpha_1-\alpha_T)\cdot (t-1)/(T-1)$. We set $\alpha_1=0.98$, $\alpha_T=0.95$. Since the Bayes denoising risk may vary under different hyperparameter settings, we compute the excess risk, denoted by $L(\Psi^{(s)})-R_{\textrm{oracle}}$ with $\bfX^0\sim\mathcal{D}(\tilde{\boldsymbol{\pi}}', K, \{\bfmu_i\}_{i=1}^M, \rho)$, to characterize the distance between the model and global convergent point. Note that during evaluation, we directly use 
$\tilde{\boldsymbol{\pi}}'$, which is a randomly generated uniform pattern distribution independent of $\tilde{\boldsymbol{\pi}}$. This is to measure the performance of the trained model under distribution shifts in the pattern proportion, thereby more accurately characterizing the quality of how the model learns the patterns.


\textbf{Convergence and score learning. }In Figure \ref{fig: convergence}, we showcase the results of convergence and score matching performance of the trained model. \ref{fig: convergence}~A reveals that the Bayes denoising risk is close to the oracle denoising risk, which verifies (\ref{eqn: R diff}) of Proposition \ref{proposition: risk_diff}. In addition, the DDPM loss gradually decreases during training to a value close to these two risks, while the score matching error gradually decreases to near zero. These observations are consistent with (\ref{eqn: thm1 loss}) of Theorem \ref{thm: training} and (\ref{eqn: score matching}) in Theorem \ref{thm: score matching}, respectively. \ref{fig: convergence}~B substantiates the discussion in Corollary \ref{cor: K} regarding the effect of $K$ on the number of training iterations required for convergence.
\ref{fig: convergence}~C shows that a more uniform distribution over pattern types can reduce the number of iterations needed for training. \ref{fig: convergence}~D explains that sampling only at smaller time steps can accelerate convergence, because under a linear schedule, $\bar{\alpha}_t$ is a decreasing function of the time step $t$. Therefore, sampling only from the first $40\%$ of time steps is equivalent to increasing the time-averaged SNR, which speeds up convergence according to Theorem \ref{thm: training}.

\textbf{Mean denoising mechanism. }
We next verify the findings in Section \ref{ssec: mechanism} regarding the mean denoising mechanism. 
We demonstrate that the sum of attention weights on keys with the same pattern as the query increases to close to $1$ during the training in Figure \ref{fig: convergence} (E), which justifies (\ref{eqn: attention_concentration}) in Corollary \ref{cor: attention} for mean denoising. 

\vspace{-2mm}
\subsection{Real-Data Experiments}\label{ssec: real}
\vspace{-2mm}

\textbf{Setup. }We conduct experiments on the real dataset MNIST \citep{LBBH02}. We select digits ``0'', ``1'', ``2'', and ``3'' for training and generation, where digit ``2'' is treated as a minority class with only $30\%$ of its training set used, while the other digits use the full training set. The training model is a $6$-layer, $4$-head DiT \citep{PX23}. We also train a CNN on MNIST to label the generated digits. During the training, we compute the FID score for each of the four generated digits to measure generation quality at different training stages.

\begin{wrapfigure}{r}{0.20\textwidth}
       \hspace{-5mm}
       \vspace{-2mm}
        \centering
  \includegraphics[width=1\linewidth, height= 1\linewidth]{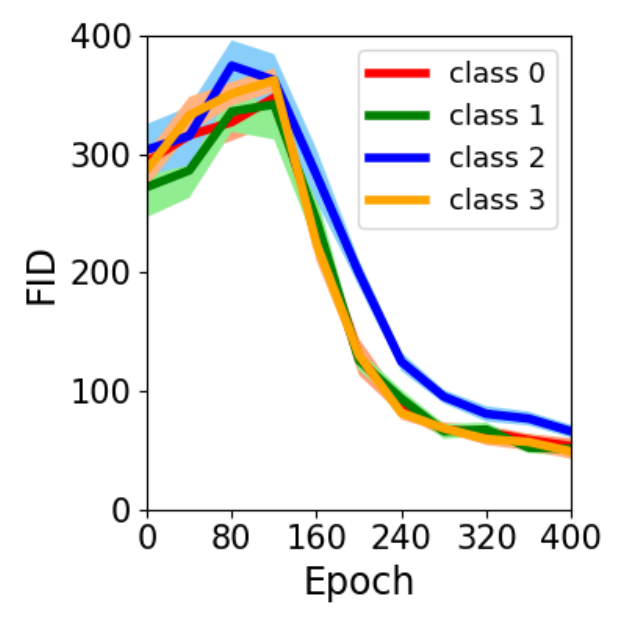} 
  \vspace{-2mm}
  \captionof{figure}{\footnotesize{FID score of the four generated digits of MNIST. The FID of the minority, ``2'' decreases more slowly than the others. }}\label{fig: fid}

\vspace{-4mm}
\end{wrapfigure}

\textbf{Training dynamics for generation. }As a minority pattern, digit ``2'' exhibits a slower decrease in FID score than the other digits (Figure \ref{fig: fid}). 
This indicates that different patterns are learned at different speeds, i.e., the high-frequency pattern is learned faster than the low-frequency pattern, which is consistent with the training dynamics from the proof of Theorem \ref{thm: training} in Section \ref{subsec: proof idea}.

%% file: appendix.tex
\newpage
\appendix
\onecolumn

The appendix is organized as follows. In Section \ref{sec: additional exp}, we show some extra experiments. In Section \ref{sec: preliminary}, we introduce important notations and lemmas used in the paper. In Section \ref{sec: proof theory}, we prove the main theorems of the paper. In Section \ref{sec: proof lemmas}, we provide the proof of the supporting key lemmas. 

\section{Additional Experiments}\label{sec: additional exp}

We first plot the query-key inner product of synthetic data. Figure \ref{fig: qk product} shows that, first, the query–key inner products corresponding to the same pattern are large and grow along the training, while those corresponding to different patterns remain small. Second, different patterns are learned at different speeds: the high-frequency pattern associated with $\max_{u\in[M]}\tilde{\pi}_u$ is learned faster and with smaller magnitude fluctuations than the low-frequency pattern associated with $\min_{u\in[M]}\tilde{\pi}_u$. We use a red dashed line and a red solid line to show the different required number of iterations of the inner products corresponding to different patterns. This result is aligned with Proposition \ref{prpst: W}.

\begin{figure}[htbp]
\centering
\includegraphics[width=0.3\textwidth]{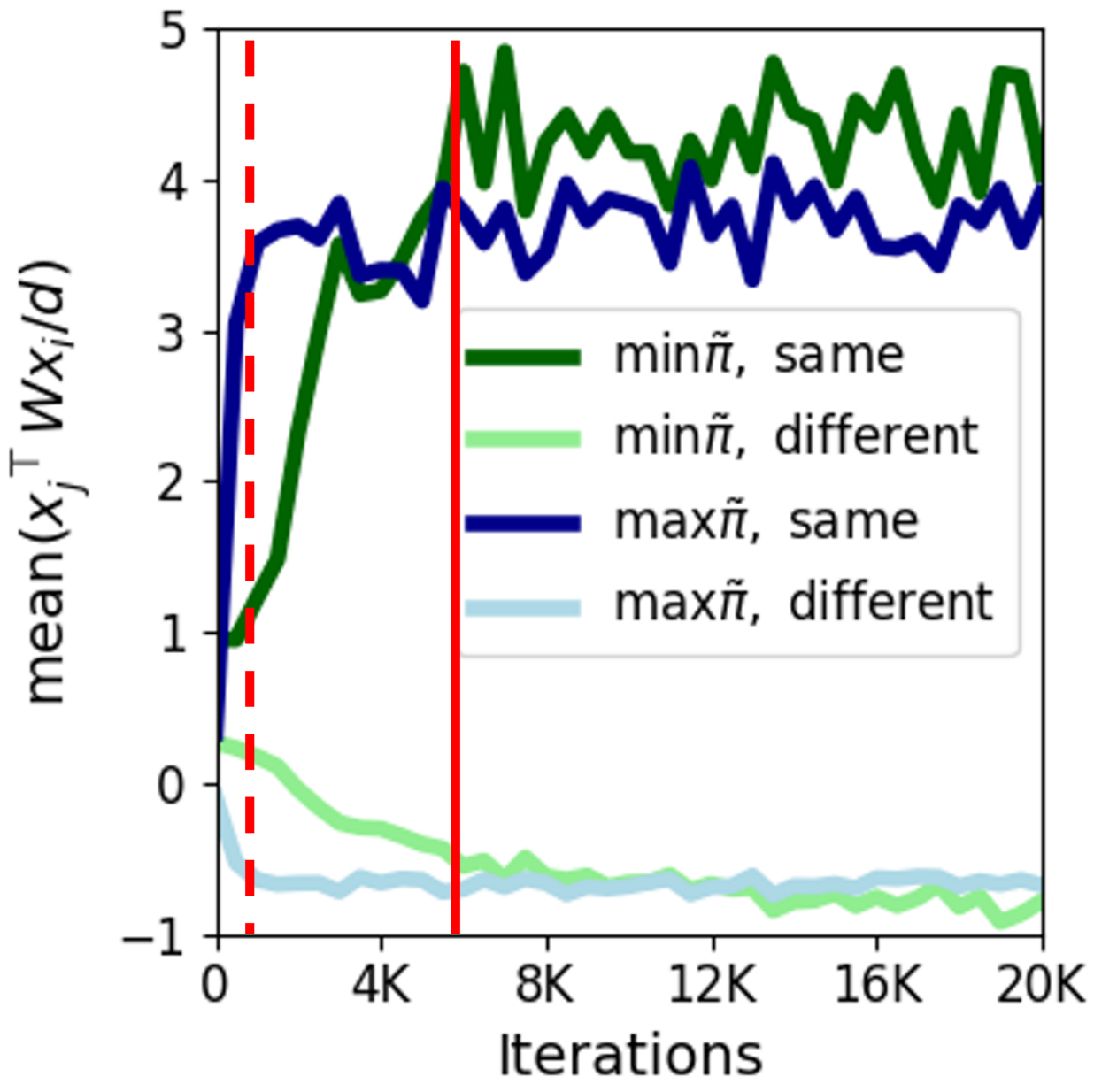}
\captionof{figure}{\footnotesize{Query-key inner products with the same or different patterns, where query patterns are the minimal or maximal of $\tilde{\boldsymbol{\pi}}$. $\min_{u\in[M]}\tilde{\pi}_u=0.01$.} }\label{fig: qk product}
\end{figure}

In Figure \ref{fig: visualization}, we show the visualization of the four generated MNIST digits using DiT \citep{PX23}. The result shows that the final generation quality of the minority digit ``2'', is relatively worse. 

\begin{figure}
  \centering
  \includegraphics[width=0.3\linewidth,height= 0.3\linewidth]{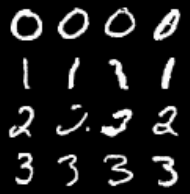} 
  \vspace{-0mm}
  \captionof{figure}{\footnotesize{Visualization of the generated digits.}}\label{fig: visualization}
\end{figure}

\section{Preliminaries}\label{sec: preliminary}

We first present Table \ref{tbl:notations} for a summary of notations used in the proof.

\begin{table}[h!]

  \begin{center}
        \caption{Summary of Notations}
        \label{tbl:notations}
\begin{tabularx}{\textwidth}{lX} 

\toprule
Notations & Annotation \\
\hline 
$d$, $P$ & $d$ is the dimension of each token of data. $P$ is the total number of tokens in each data.\\
\hline 
$M$, $K$ & $M$ is the total number of mean patterns. $K$ is the number of distinct patterns in each data. \\
$\tilde{\boldsymbol{\pi}}$, $\{\bfmu_i\}_{i=1}^M$, $\rho$ & $\tilde{\boldsymbol{\pi}}$ is the fraction vector of all the Gaussian components of the MTGM distribution. $\{\bfmu_i\}_{i=1}^M$ is the set of all the mean patterns. $\rho^2$ is the variance of the Gaussian components. \\
\hline
$\bfX$, $\bfY$, $\bfZ$ & $\bfX$ denotes the data. $\bfY$ ad $\bfZ$ are the latent variable to define the distribution of $\bfX$.\\
\hline 
$\Psi$, $\bfW$, $T$, $\{v_t\}_{t=1}^T$ & $\Psi$ is the set of parameters in the learning model. In our work, $\Psi$ contains $\bfW$ and $\{v_t\}_{t=1}^T$, where $\bfW$ is the self-attention parameter, and $\{v_t\}_{t=1}^T$ is the coefficient parameter. $T$ is the total number of diffusion time steps. \\
\hline $\bfX^0$, $\{\bar{\alpha}_t\}_{t=1}^T$, $\bfE$, $\bfX^t$ & $\bfX^t$ is the noisy input at time step $t$. $\bfX^0$ is the clean input before adding the noise. $\{\bar{\alpha}\}_{t=1}^T$ is the noise schedule coefficients. $\bfE$ is the additive Gaussian noise. \\
\hline 
$s(\bfX^t,t)$, $s_{\theta}(\bfX^t, t)$ & $s(\bfX^t,t)$ is the score function. $s_{\theta}(\bfX^t, t)$ is the neural network parameterized by $\theta$ to learn the score function.\\
\hline 
$\nu_{\min}^{\tilde{\boldsymbol{\pi}}}$, $\delta(\boldsymbol{\pi})$, $\textrm{SNR}$ & $\nu_{\min}^{\tilde{\boldsymbol{\pi}}}$ is the minimum average probability of selecting a pattern in the data. $\delta(\boldsymbol{\pi})$ is the degree of imbalance between the prior probabilities of the least and the most probable patterns. $\textrm{SNR}$ is the time-averaged signal-noise-ratio over the noise schedule. \\
\hline 
$R_{\text{Bayes}}$, $\bfM_{\bfY}$, $R_{\text{oracle}}$ & $R_{\text{Bayes}}$ is the optimal risk of minimizing the DDPM loss, where the denoising model is chosen as the MMSE estimator of the noise. $\bfM_{\bfY}$ is the matrix of mean patterns given $\bfY$, the latent variable of $\bfX$. $R_{\text{oracle}}$ is the risk if the denoising model is chosen as the oracle MMSE estimator with $\bfM_{\bfY}$ as known. \\
\hline 
 \small $\mathcal{O}()$, $\Omega()$, $\Theta()$ & We follow the convention that $f(x)=O(g(x))$ (or $\Omega(g(x))$, $\Theta(g(x)))$) means that $f(x)$ increases at most, at least, or in the order of $g(x)$, respectively. Specifically, if $f(x)=O(g(x))$, then there exists $C>0$ and  $a>0$, such that $f(x)\leq C\cdot g(x)$ when $x>a$. If $f(x)=\Omega(g(x))$, then there exists $c>0$ and $a>0$, such that $f(x)\geq c\cdot g(x)$ when $x>a$. If $f(x)=\Theta(g(x))$, then there exists $C>c>0$ and $a>0$, such that $c\cdot g(x)\leq f(x)\leq C\cdot g(x)$ when $x>a$.\\
 \hline 
 \small $\gtrsim$, $\lesssim$ & $f(x)\gtrsim g(x)$ (or $f(x)\lesssim g(x)$ ) means that $f(x)\geq \Omega(g(x))$ (or $f(x)\lesssim \mathcal{O}(g(x))$).\\
 \hline
 \small $\mathrm{poly}()$ & If $f(x)=\mathrm{poly}(x)$, then there exists $k>0$ and a set of constants $\{c_i\}_{i=0}^k$, such that $f(x)=\sum_{i=0}^k c_i x^i$, which means $f(x)$ is a polynomial function of $x$ with a finite maximal power.\\
 \hline
\bottomrule
\end{tabularx}
\end{center}

\end{table}

\begin{lemma} \label{lemma: chernoff}
    (Multiplicative Chernoff bounds, Theorem D.4 of \citep{MRT18}) Let $X_1$, $\cdots$, $\bfX_m$ be independent random variables drawn according to some distribution $\mathcal{D}$ with mean $p$ and support included in $[0,1]$. Then, for any $\gamma\in[0,\frac{1}{p}-1]$,  the following inequality holds for $\hat{p}=\frac{1}{m}\sum_{i=1}^m X_i$:
    \begin{equation}
        \Pr(\hat{p}\geq (1+\gamma)p)\leq e^{-\frac{m p\gamma^2}{3}},
    \end{equation}
    \begin{equation}
        \Pr(\hat{p}\leq (1-\gamma)p)\leq e^{-\frac{m p\gamma^2}{2}}.
    \end{equation}
\end{lemma}


\begin{definition}\label{def: sub-Gaussian}\citep{V10}
We say $X$ is a sub-Gaussian random variable with sub-Gaussian norm $K>0$, if $(\mathbb{E}|X|^p)^{\frac{1}{p}}\leq K\sqrt{p}$ for all $p\geq 1$. In addition, the sub-Gaussian norm of X, denoted $\|X\|_{\psi_2}$, is defined as $\|X\|_{\psi_2}=\sup_{p\geq 1}p^{-\frac{1}{2}}(\mathbb{E}|X|^p)^{\frac{1}{p}}$.
\end{definition}

\begin{lemma}  (\cite{V10} 
Proposition 5.1,  Hoeffding's inequality)  Let $X_1, X_2, \cdots, X_N$ be independent centered sub-gaussian random variables, and let $K=\max_i\|\bfX_i\|_{\psi_2}$. Then for every $\bfa=(a_1,\cdots,a_N)\in\mathbb{R}^N$ and every $t\geq0$, we have
\begin{equation}
    \Pr\Big(\Big|\sum_{i=1}^N a_i X_i\Big|\geq t\Big)\leq e\cdot \exp\left(-\frac{ct^2}{K^2\|\bfa\|^2}\right),\label{hoeffding}
\end{equation}
where $c>0$ is an absolute constant.
\end{lemma}

\begin{definition}
    For $i,p\in [P]$, $t\in [T]$, $u\in [M]$, and $\bfX$ that follows Definition \ref{def: data dist}, we denote $\zeta^u_{i,p,t}(s)=\text{softmax}_p(\bfx_i^t{}^\top \bfW^{(s)}\bfx_p^t/d)\mathbbm{1}[Y_p=u]$.
\end{definition}

\begin{lemma}\label{lemma: W training}
    Given conditions (i)-(v) in Theorem \ref{thm: training}, we have that for any $\tilde{\bfX}$ that follows the distribution in Definition \ref{def: data dist}, where $\tilde{\bfx}_j$ and $\tilde{\bfx}_{j'}$ have $\bfmu_u$ as the mean, and $\tilde{\bfx}_k$ has $\bfmu_{u'}$ as the mean ($u\neq u'$), then we have for any $s>0$,
    \begin{equation}
        \begin{aligned}
        &(-\tilde{\bfx}_j^t{}^\top) \frac{1}{T}\sum_{t=1}^T\mathbb{E}_{\bfE,\bfX^0}\Big[\frac{\partial L(\Psi^{(s_0)})}{\partial \bfW}\Big]\tilde{\bfx}_{j'}^t\\
        \gtrsim & \frac{1}{T}\sum_{t=1}^T\mathbb{E}_{\bfE,\bfX^0}[(v_t^{(s_0)})^2\bar{\alpha}_t^3 d (1-\sum_{l=1}^P \zeta_{l,p,t}^u(s_0))^2\sum_{i=1}^P\zeta_{i,p,t}^u(s_0)]\cdot \nu_u^{\tilde{\boldsymbol{\pi}}}(K),\label{same mean W}
        \end{aligned}
    \end{equation}
    \begin{equation}
        \begin{aligned}
             (-\tilde{\bfx}_{k}^\top)\frac{1}{T}\sum_{t=1}^T\mathbb{E}_{\bfE,\bfX^0}\Big[\frac{\partial L(\Psi^{(s_0)})}{\partial \bfW}\Big]\tilde{\bfx}_{j'}^t
            \lesssim \frac{\log d}{\sqrt{d}}\cdot (-\tilde{\bfx}_{j}^\top)\frac{1}{T}\sum_{t=1}^T\mathbb{E}_{\bfE,\bfX^0}\Big[\frac{\partial L(\Psi^{(s_0)})}{\partial \bfW}\Big]\tilde{\bfx}_{j'}^t.
        \end{aligned}\label{qk small eqnlemma}
    \end{equation}
\end{lemma}

\begin{lemma}\label{lemma: stage 1}
    For any $\epsilon\in (0, \delta^{\Theta(1)})$, when the number of iterations satisfies
\vspace{-0.mm}
\begin{equation}
    I_1\geq \Omega((\epsilon^{-1}+\nu_{\min}^{\tilde{\boldsymbol{\pi}}}(K)^{-3})\eta^{-1}\nu_{\min}^{\tilde{\boldsymbol{\pi}}}(K)^{-1}\textrm{SNR}^{-3}),
\end{equation}
 $P\geq \Omega(\nu_{\min}^{\tilde{\boldsymbol{\pi}}}(K)^{-1}(\rho^2+1)\epsilon^{-1}\log d)$, $d\geq \Omega(\epsilon^{-1}\log (\epsilon^{-1}\nu_{\min}^{\tilde{\boldsymbol{\pi}}}(K)^{-1}))$, and $T\geq\Omega(\log d)$, with the step size $\eta_1\leq O(1)$, then w.h.p., the learned model returns $\bfW^{(I_1)}$ such that
\begin{equation}
    \Big\|\frac{1}{T}\sum_{t=1}^T\nabla_\bfW \mathbb{E}_{\bfX^0, \bfE}\|\bff(\Psi^{(I_1)};\sqrt{\bar{\alpha}_t}\bfX^{0}+\sqrt{1-\bar{\alpha}_t}\bfE, t)-\bfE\|_F^2/(dP)\Big\|\leq \epsilon^2\cdot \eta \cdot K^{-1}\delta(\tilde{\boldsymbol{\pi}})^{-1}\cdot \log(I_1(\rho^2+1)\epsilon^{-1}).
\end{equation}
\end{lemma}

\begin{lemma}\label{lemma: stage 2}
    For any $\epsilon\in (0, \delta^{\Theta(1)})$, $s>\Omega(I_1)$, 
with the step size $\eta\leq O((\max\{\rho,1\}+\epsilon)^{-1})$, then $v_t$ converges linearly to $v_t^*$ with
\begin{equation}
    |v_t^{(s)}-v_t^*|\leq (1-2\eta (\rho\sqrt{\bar{\alpha}_t}+\sqrt{1-\bar{\alpha}_t}+\epsilon))^s|v_t^{(I_1)}-v_t^*|,
\end{equation}
where
\begin{equation}
    |v_t^*-\sqrt{1-\bar{\alpha}_t}/(\bar{\alpha}_t\rho^2+1-\bar{\alpha}_t)|\leq \epsilon/(\bar{\alpha}_t\rho^2+1-\bar{\alpha}_t).
\end{equation}
When (i) the number of iterations satisfies
\vspace{-0.mm}
\begin{equation}
    I_2\geq \Omega(\log \epsilon^{-1}I_1(1+\rho^2)),
\end{equation}
and (ii) $P\geq \Omega(\nu_{\min}^{\tilde{\boldsymbol{\pi}}}(K)^{-1}(\rho^2+1)\epsilon^{-1}\log d)$, then w.h.p., we have for $s\geq \Omega(I_1+I_2)$
\begin{equation}
    |v_t^{(s)}-v_t^*|\leq \epsilon.
\end{equation}
\end{lemma}

\section{Proof of Main Theorems}\label{sec: proof theory}
\subsection{Proof of Theorem \ref{thm: training}}
\begin{proof}
Following the analytical framework used in \citep{LWLC23, LWLW23, LLSW24,LWZL24, LWLC24_cot0, ZLYC25, LZZC25, SZLW25, LWLC24_cot, ZLSR25, LLCC25, SLZW26}, we provide a convergence analysis of the DDPM training. This part mainly introduces the proof steps that combine the lemmas of gradient updates of different model parameters and training stages to derive the final convergence conclusion. The overall proof idea is summarized in Section \ref{subsec: proof idea}. 
Note that the oracle denoising risk with the mean matrix of $\bfX^t$ known is
\begin{equation}
    R_{oracle}=\frac{\rho^2\bar{\alpha}_t}{\rho^2\bar{\alpha}_t+1-\bar{\alpha}_t}.
\end{equation}
Therefore, by the Mean Value Theorem, for $\tilde{v}$ between $v_t^{(I_2)}$ and $\frac{\sqrt{1-\bar{\alpha}_t}}{1-\bar{\alpha}_t+\rho^2\bar{\alpha}_t}$,
\begin{equation}
    \begin{aligned}
        &\mathbb{E}_{\bfE,\bfX^0} [\|v_t^{(I_2)}(\bfx_p^t-\sum_{i=1}^P \bfx_i^t\text{softmax}_p(\frac{{\bfx_i^t}^\top\bfW^{(I_1)}\bfx^t_p}{d}))-\boldsymbol{\epsilon}_p\|^2/d]-R_{oracle}\\
        \leq & \epsilon\cdot 2(\tilde{v}\cdot \mathbb{E}_{\bfE,\bfX^0}[\bar{\alpha}_t\rho^2+1-\bar{\alpha}_t+\|\sqrt{\bar{\alpha}}_t\bfmu_u-\sum_{i=1}^P \bfx_i^t\text{softmax}_p(\frac{{\bfx_i^t}^\top\bfW^{(I_1)}\bfx^t_p}{d})\|^2/d\\
        &+\frac{2\beta_1\rho\sqrt{\bar{\alpha}_t}}{d}+\frac{2\beta_2\sqrt{1-\bar{\alpha}_t}}{d}]-(\sqrt{1-\bar{\alpha}_t}+\frac{\beta_2}{d}))\\
        \lesssim & O(\epsilon),
    \end{aligned}
\end{equation}
and we can obtain that for any $t$,
\begin{equation}
    \mathbb{E}_{\bfX^0,\bfE}\|\bff(\Psi^{(S)};\sqrt{\bar{\alpha}_t}\bfX^{0}+\sqrt{1-\bar{\alpha}_t}\bfE, t)-\bfE\|_F^2/(dP)\|\leq R_{oracle}+O(\epsilon).
\end{equation}
Note that
\begin{equation}
    I_1+I_2\lesssim I_1+\log (1+\rho^2)\epsilon^{-1}.
\end{equation}
Therefore, by combining Lemma \ref{lemma: stage 1} and \ref{lemma: stage 2}, we can obtain the following result. For any $\epsilon\in (0,\delta^{\Theta(1)})$, when the number of iterations satisfies
\begin{equation}
    s\geq \Omega((\epsilon^{-1}+\nu_{\min}^{\tilde{\boldsymbol{\pi}}}(K)^{-3})\eta^{-1}\nu_{\min}^{\tilde{\boldsymbol{\pi}}}(K)^{-1}\textrm{SNR}^{-3}+\log (1+\rho^2)\epsilon^{-1}),
\end{equation}
and the number of tokens 
\begin{equation}
    P\geq \Omega(\nu_{\min}^{\tilde{\boldsymbol{\pi}}}(K)^{-1}(\rho^2+1)\epsilon^{-1}\log d),
\end{equation}
with the step size $\eta\leq (\max\{\rho,1\}+\epsilon)^{-1}$, then with a high probability, the learned model $\Psi^{(S)}$ satisfies
\begin{equation}
\begin{aligned}
    &\frac{1}{T}\sum_{t=1}^T\mathbb{E}_{\bfX^0\sim\mathcal{D}(\tilde{\boldsymbol{\pi}},\{\boldsymbol{\mu_i}\}_{i=1}^M,\rho),\bfE}\big[\|\bff(\Psi^{(S)};\sqrt{\bar{\alpha}_t}\bfX^{0}+\sqrt{1-\bar{\alpha}_t}\bfE, t)-\bfE\|_F^2/dP\big]\\
    \leq &R_{\textrm{oracle}}+O(\epsilon).
    \end{aligned}
\end{equation}
Combining Corollary \ref{cor: R_oracle}, we have
\begin{equation}
\begin{aligned}
    &\frac{1}{T}\sum_{t=1}^T\mathbb{E}_{\bfX^0\sim\mathcal{D}(\tilde{\boldsymbol{\pi}},\{\boldsymbol{\mu_i}\}_{i=1}^M,\rho),\bfE}\big[\|\bff(\Psi^{(S)};\sqrt{\bar{\alpha}_t}\bfX^{0}+\sqrt{1-\bar{\alpha}_t}\bfE, t)-\bfE\|_F^2/dP\big]\\
    \leq &R_{\textrm{Bayes}}+O(\epsilon).
    \end{aligned}
\end{equation}

\end{proof}

\subsection{Proof of Corollary \ref{cor: K}}
\begin{proof}

\noindent Given a fixed $\tilde{\boldsymbol{\pi}}$ and $M$, when $K=1$, we have $\mathbb{E}[\pi_u]=1$ for any $u\in[M]$. Then, 
\begin{equation}
    \mathbb{E}[\pi_u]=\mathbb{E}_{\bfZ\in\{0,1\}^M, \|\bfZ\|_0=K}\big[\frac{\tilde{\pi}_u}{\bfZ^\top\tilde{\boldsymbol{\pi}}}\big],
\end{equation}

\begin{equation}
    \min_{u\in[M]}\{\mathbb{E}[\pi_u]\}=\frac{1}{K}.
\end{equation}
Since that 
\begin{equation}
    \sum_{u\in [M]}\mathbb{E}[\pi_u]=1,
\end{equation}
$K=1$ is the case where $\min_{u\in[M]}\{\mathbb{E}[\pi_u]\}\cdot \frac{K}{M}$ reaches its maximal. In this case, we have
\begin{equation}
    I_1=(\epsilon^{-1}+1) \eta^{-1}M \cdot \textrm{SNR}^{-3}
\end{equation}

\noindent When $K=M$, we have $\mathbb{E}[\pi_u]=\tilde{\pi}_u$ for any $u\in[M]$. Then,
\begin{equation}
    \min_{u\in[M]}\{\mathbb{E}[\pi_u]\}\cdot \frac{K}{M}=\min_{u\in[M]}\{\tilde{\pi}_u\}.
\end{equation}
By Jensen's inequality, 
\begin{equation}
\begin{aligned}
    \mathbb{E}[\pi_u]=&\mathbb{E}_{\bfZ\in\{0,1\}^M, \|\bfZ\|_0=K}\big[\frac{\tilde{\pi}_u}{\bfZ^\top\tilde{\boldsymbol{\pi}}}\big]\\
    =& \mathbb{E}_{\bfZ\in\{0,1\}^M, \|\bfZ\|_0=K}\big[\frac{\tilde{\pi}_u}{(\bfZ^\top\tilde{\boldsymbol{\pi}}-\tilde{\pi}_u)+\tilde{\pi}_u}\big]\\
    \geq & \frac{\tilde{\pi}_u}{\mathbb{E}_{\bfZ\in\{0,1\}^M, \|\bfZ\|_0=K}[\bfZ^\top\tilde{\boldsymbol{\pi}}-\tilde{\pi}_u]+\tilde{\pi}_u}\\
    =& \frac{\tilde{\pi}_u}{\frac{(K-1)(1-\tilde{\pi}_u)}{M-1}+\tilde{\pi}_u}
    \end{aligned}
\end{equation}
Then, for any $u\in [M]$, we have
\begin{equation}
    \mathbb{E}_{K=M}[\pi_{u}]=\tilde{\pi}_{u^*}=\frac{\tilde{\pi}_{u}}{\frac{(M-1)(1-\tilde{\pi}_{u})}{M-1}+\tilde{\pi}_u}\leq \frac{\tilde{\pi}_u}{\frac{(K-1)(1-\tilde{\pi}_u)}{M-1}+\tilde{\pi}_u}\leq \mathbb{E}_{K< M}[\pi_u],
\end{equation}
where the first inequality comes from the fact that $g(K)=\frac{\tilde{\pi}_u}{\frac{(K-1)(1-\tilde{\pi}_u)}{M-1}+\tilde{\pi}_u}$ is a decreasing function of $K$. Therefore, $K=M$ is the case where $\min_{u\in[M]}\{\mathbb{E}[\pi_u]\}\cdot \frac{K}{M}$ reaches its minimal. In this case, we have
\begin{equation}
    I_1=(\epsilon^{-1}+ \min_{u\in [M]}\{\tilde{\pi}_u\}^{-3})\eta^{-1} \min_{u\in [M]}\{\tilde{\pi}_u\}^{-1}\textrm{SNR}^{-3}
\end{equation}
\end{proof}

\subsection{Proof of Theorem \ref{thm: score matching}}
\begin{proof}
By Fisher identity, we have
\begin{equation}
\begin{aligned}
    \nabla_{\bfX^t}\log q(\bfX^t)=&\mathbb{E}[\nabla_{\bfX^t}\log q(\bfX^t|\bfX^0)|\bfX^t]\\
    =& \mathbb{E}\Big[-\frac{1}{1-\bar{\alpha}_t}(\bfX^t-\sqrt{\bar{\alpha}_t}\bfX^0)\Big|\bfX^t\Big]\\
    =& -\frac{1}{\sqrt{1-\bar{\alpha}_t}}\mathbb{E}[\bfE |\bfX^t]
    \end{aligned}
\end{equation}
where the second step is by $\bfX^T=\sqrt{\bar{\alpha}_t}\bfX^0+\sqrt{1-\bar{\alpha}_t}\bfE$. 
Let \begin{equation}
    s_{\Psi^{(S)}}(\bfX^t,\bfE,t)=-\frac{1}{\sqrt{1-\bar{\alpha}_t}}\bff(\bfW^{(S)},\bfv^{(S)};\sqrt{\bar{\alpha}_t}\bfX^{0}+\sqrt{1-\bar{\alpha}_t}\bfE, t)
\end{equation}
for the required number of iterations in (\ref{eqn: iteration}). 
Then, we can obtain that with $s\geq \Omega(I_1+I_2)$,
\begin{equation}
\begin{aligned}
& \mathbb{E}_{\bfX^0,\bfE}\big[\|\nabla_{\bfX^t}\log q(\bfX^t)-s_{\Psi^{(S)}}(\bfX^t,\bfE,t)\|^2\big]\\
=& \mathbb{E}_{\bfX^0,\bfE}\Big[\Big\|\frac{1}{\sqrt{1-\bar{\alpha}_t}}\mathbb{E}[\bfE |\bfX^t]-\frac{1}{\sqrt{1-\bar{\alpha}_t}}\mathbb{E}[\bfE |\bfX^t, \bfY]\\
&+\frac{1}{\sqrt{1-\bar{\alpha}_t}}\mathbb{E}[\bfE |\bfX^t,\bfY]+s_{\Psi^{(S)}}(\bfX^t,\bfE,t)\Big\|^2\Big]\\
\leq & 2\mathbb{E}_{\bfX^0,\bfE}\Big[\Big\|\frac{1}{\sqrt{1-\bar{\alpha}_t}}\mathbb{E}[\bfE |\bfX^t]-\frac{1}{\sqrt{1-\bar{\alpha}_t}}\mathbb{E}[\bfE |\bfX^t, \bfY]\Big\|^2\Big]\\
&+ 2\mathbb{E}_{\bfX^0,\bfE}\Big[\Big\|\frac{1}{\sqrt{1-\bar{\alpha}_t}}\mathbb{E}[\bfE |\bfX^t,\bfY]+s_{\Psi^{(S)}}(\bfX^t,\bfE,t)\Big\|^2\Big]\\
:=& 2C_1+2C_2.
    \end{aligned}
\end{equation}
Note that 
\begin{equation}
    \mathbb{E}[\bfE|\bfX^t, \bfY]=\frac{\sqrt{1-\bar{\alpha}_t}}{1-\bar{\alpha}_t+\rho^2\bar{\alpha}_t}(\bfX^t-\sqrt{\bar{\alpha}_t}\bfM_\bfY).
\end{equation}
Therefore,
\begin{equation}
    \begin{aligned}
        C_2=&\mathbb{E}_{\bfX^0,\bfE}\Big[\frac{1}{1-\bar{\alpha}_t}\Big\|\bff(\bfW,\bfv;\sqrt{\bar{\alpha}_t}\bfX^{0}+\sqrt{1-\bar{\alpha}_t}\bfE, t)-\mathbb{E}[\bfE|\bfX^t, \bfY]\Big\|^2\Big]\\
        =& \mathbb{E}_{\bfX^0,\bfE}\Big[\frac{\sum_{p=1}^P}{1-\bar{\alpha}_t}\Big\|(v_t^{(s)}-\frac{\sqrt{1-\bar{\alpha}_t}}{1-\bar{\alpha}_t+\rho^2\bar{\alpha}_t})(\bfx_p^t-\sum_{i=1}^P \bfx_i^t\text{softmax}_p(\frac{{\bfx_i^t}^\top\bfW^{(s)}\bfx^t_p}{d}))\\
        &+\frac{\sqrt{1-\bar{\alpha}_t}\cdot (\sum_{i=1}^P \bfx_i^t\text{softmax}_p(\frac{{\bfx_i^t}^\top\bfW^{(s)}\bfx^t_p}{d})-\sqrt{\bar{\alpha}_t}(\bfM_{\bfY})_p)}{1-\bar{\alpha}_t+\rho^2\bar{\alpha}_t}\Big\|^2\Big]\\
        \lesssim &dP\cdot \frac{1}{1-\bar{\alpha}_t}(\epsilon^2+\frac{\epsilon/d}{(1-\bar{\alpha}_t)^2})+ dP\epsilon\cdot \frac{1}{(1-\bar{\alpha}_t+\rho^2\bar{\alpha}_t)^2}\label{C2}
    \end{aligned}
\end{equation}
where the last step is by (\ref{stage 2 second term}), and 
\begin{equation}
    \begin{aligned}
        &|v_t^{(s)}-\frac{\sqrt{1-\bar{\alpha}_t}}{1-\bar{\alpha}_t+\rho^2\bar{\alpha}_t}|\\
        \leq & |v_t^{(s)}-v_t^*|+|v_t^*-\frac{\sqrt{1-\bar{\alpha}_t}}{1-\bar{\alpha}_t+\rho^2\bar{\alpha}_t}|\\
        \leq & \epsilon+\frac{\sqrt{\epsilon/d}}{\bar{\alpha}_t\rho^2+1-\bar{\alpha}_t}+O(\frac{\sqrt{1-\bar{\alpha}_t}\sqrt{\epsilon/d}}{(\bar{\alpha}_t\rho^2+1-\bar{\alpha}_t)^\frac{3}{2}})\\
        \leq & \epsilon+O(\frac{\sqrt{\epsilon/d}}{1-\bar{\alpha}_t}).\label{vts gt}
    \end{aligned}
\end{equation}
Consider the Hilbert space $\mathcal{H}:=L^2(\Omega, \mathcal{F}, \mathbb{P})$ equipped with the inner product $\left\langle X,Y\right\rangle=\mathbb{E}[XY]$, where $(\Omega, \mathcal{F}, \mathbb{P})$ is the underlying probability space. We define 
\begin{equation}
    H_\mathcal{A}:=\{X\in\mathcal{H}: X\text{ is }\mathcal{A}\text{-measurable}\}
\end{equation}
as the closed subspace for any sub-$\sigma$-algebra $\mathcal{A}\subseteq\mathcal{F}$. Let $\mathcal{G}_1=\sigma(\bfX^t)$ and $\mathcal{G}_2=\sigma(\bfX^t, \bfY)$. We have $\mathcal{G}_1\subseteq\mathcal{G}_2$. We know that $\mathbb{E}[\bfE|\bfX^t]$ and $\mathbb{E}[\bfE|\bfX^t, \bfY]$ are orthogonal projections from $\bfE$ onto $\mathcal{G}_1$ and $\mathcal{G}_2$, respectively. Then, by Pythagorean identity, we have
\begin{equation}
    \mathbb{E}[\|\bfE-\mathbb{E}[\bfE|\bfX^t]\|^2]=\mathbb{E}[\|\bfE-\mathbb{E}[\bfE|\bfX^t,\bfY]\|^2]+\mathbb{E}[\|\mathbb{E}[\bfE|\bfX^t]-\mathbb{E}[\bfE|\bfX^t, \bfY]\|^2].
\end{equation}
Hence, 
\begin{equation}
\begin{aligned}
    C_1\leq &\frac{1}{1-\bar{\alpha}_t}(\mathbb{E}[\|\bfE-\mathbb{E}[\bfE|\bfX^t]\|^2]-\mathbb{E}[\|\bfE-\mathbb{E}[\bfE|\bfX^t,\bfY]\|^2])\\
    \leq & \frac{1}{1-\bar{\alpha}_t}\cdot dP\epsilon.\label{C1}
    \end{aligned}
\end{equation}
Combining (\ref{C1}) and (\ref{C2}), we have
\begin{equation}
\mathbb{E}_{\bfX^0,\bfE}\big[\|\nabla_{\bfX^t}\log q(\bfX^t)-s_{\Psi^{(S)}}(\bfX^t,\bfE,t)\|^2\big]\leq dP\epsilon\cdot (\frac{1}{1-\bar{\alpha}_t}).
\end{equation}
By Hoeffding's inequality (\ref{hoeffding}), we have that with a probability of $1-d^{-C}$ for a large $C>1$, 
\begin{equation}
\begin{aligned}
    &\frac{1}{T}\sum_{t=1}^T\mathbb{E}_{\bfX^0,\bfE}\big[\|\nabla_{\bfX^t}\log q(\bfX^t)-s_{\Psi^{(S)}}(\bfX^t,\bfE,t)\|^2\big]\\
    \leq & \mathbb{E}_{\bfX^0,\bfE,t}\big[\|\nabla_{\bfX^t}\log q(\bfX^t)-s_{\Psi^{(S)}}(\bfX^t,\bfE,t)\|^2\big]+dP\epsilon\cdot \frac{1}{(1-\bar{\alpha}_1)}\cdot\sqrt{\frac{\log d}{T}}\\
    \lesssim & dP\epsilon(\textrm{SNR}+1),
    \end{aligned}
\end{equation}
where the first step is by Hoeffding's inequality (\ref{hoeffding}), and the last step holds if $T\geq \Omega(\log d)$. Hence,
\begin{equation}
\begin{aligned}
    &\frac{1}{dPT}\sum_{t=1}^T\mathbb{E}_{\bfX^0,\bfE}\big[\|\nabla_{\bfX^t}\log q(\bfX^t)-s_{\Psi^{(S)}}(\bfX^t,\bfE,t)\|^2\big]
    \lesssim \epsilon(\textrm{SNR}+1),
    \end{aligned}
\end{equation}

\end{proof}

\subsection{Proof of Proposition \ref{prpst: mmse}}
\begin{proof}
    From (\ref{stage 2 second term}), we have
    \begin{equation}
        \begin{aligned}
            &\frac{\|\sqrt{\bar{\alpha}}_t\bfM_{\bfY,\bfZ}-\bfX^t\text{softmax}(\bfX^t{}^\top\bfW^{(S)}\bfX^t/d)\|_F^2}{dP}\\
        \leq & O(\frac{(\rho^2+1)\log d}{\sum_{p=1}^P\mathbbm{1}[Y_p=\arg\min_{u\in[M]} \nu_u^{\tilde{\boldsymbol{\pi}}}(K)]})\\
        \lesssim & \frac{(\rho^2+1)\log d}{P\nu_{\min}^{\tilde{\boldsymbol{\pi}}}(K)}.
        \end{aligned}
    \end{equation}
    By (\ref{vts gt}), we have
    \begin{equation}
    \begin{aligned}
        |v_t^{(s)}-\frac{\sqrt{1-\bar{\alpha}_t}}{1-\bar{\alpha}_t+\rho^2\bar{\alpha}_t}|
        \leq  \epsilon+O(\frac{\sqrt{\epsilon/d}}{1-\bar{\alpha}_t})\leq O(\epsilon+\sqrt{\epsilon/d})\leq O(\epsilon),
    \end{aligned}
\end{equation}
where the last step comes from $d\gtrsim \epsilon^{-1}$. 
\end{proof}

\subsection{Proof of Corollary \ref{cor: R_oracle}}
\begin{proof}

Note that 
\begin{equation}
    R_{\mathrm{oracle}}=\frac{1}{dP}\mathbb{E}[\|\bfE-\mathbb{E}[\bfE|\bfX^t, \bfY]\|^2]=\frac{1}{dP}\mathbb{E}[\mathrm{Var}(\mathbb{E}|\bfX^t, \bfY)],
\end{equation}
while
\begin{equation}
    R_{\mathrm{Bayes}}=\frac{1}{dP}\mathbb{E}[\|\bfE-\mathbb{E}[\bfE|\bfX^t]\|^2]=\frac{1}{dP}\mathbb{E}[\mathrm{Var}(\bfE|\bfX^t)].
\end{equation}
By the law of total variance, we have
\begin{equation}
    \mathrm{Var}(\bfE|\bfX^t)=\mathbb{E}[\mathrm{Var}(\bfE|\bfX^t, Y, Z)|\bfX^t]+\mathrm{Var}(\mathbb{E}[\bfE|\bfX^t, \bfY]|\bfX^t).
\end{equation}
Therefore,
\begin{equation}
    R_{\mathrm{Bayes}}-R_{\mathrm{oracle}}=\frac{1}{dP}\mathbb{E}[\mathrm{Var}(\mathbb{E}[\bfE|\bfX^t, \bfY]|\bfX^t)]>0.
\end{equation}
Recall that
\begin{equation}
    \mathbb{E}[\bfE|\bfX^t, \bfY]=\frac{\sqrt{1-\bar{\alpha}_t}}{1-\bar{\alpha}_t+\rho^2\bar{\alpha}_t}(\bfX^t-\sqrt{\bar{\alpha}_t}\bfM_\bfY),
\end{equation}
where $\bfM_\bfY$ denotes the mean matrix of $\bfX^t$ given $\bfY$. Then,
\begin{equation}
\begin{aligned}
    \mathbb{E}[\bfX^0|\bfX^t, \bfY]=&\mathbb{E}\Big[\frac{\bfX^t-\sqrt{1-\bar{\alpha}_t}\bfE}{\sqrt{\bar{\alpha}_t}}|\bfX^t, \bfY\Big]\\
    =& \frac{\rho^2\sqrt{\bar{\alpha}_t}}{1-\bar{\alpha}_t+\rho^2\bar{\alpha}_t}\bfX^t+\frac{1-\bar{\alpha}_t}{1-\bar{\alpha}_t+\rho^2\bar{\alpha}_t}\bfM_\bfY,
    \end{aligned}
\end{equation}
\begin{equation}
    \mathrm{Var}(\mathbb{E}[\bfX^0|\bfX^t, \bfY]|\bfX^t)=(\frac{1-\bar{\alpha}_t}{1-\bar{\alpha}_t+\rho^2\bar{\alpha}_t})^2\mathrm{Var}(\bfM_\bfY|\bfX^t).
\end{equation}
Note that
\begin{equation}
\begin{aligned}
    \mathbb{E}[\mathrm{Var}(\bfM_\bfY|\bfX^t)]=&\mathbb{E}[\|\bfM_\bfY-\mathbb{E}[\bfM_\bfY|\bfX^t]\|^2]\\
    =& \mathbb{E}[\|\bfM_\bfY-\argmin_{\bfA(\bfX^t)}\mathbb{E}[\|\bfM_\bfY-\bfA(\bfX^t)\|^2]\|^2]\\
    \leq & \mathbb{E}[\|\bfM_\bfY-\bfA(\bfX^t)\|^2].
    \end{aligned}
\end{equation}
Let
\begin{equation}
    \bfA(\bfX^t)=\Big(\frac{1}{|\{i: Y_i=Y_p\}|}\sum_{i: Y_i=Y_p}\bfx_i^t\Big)_{p=1}^P.
\end{equation}
Since that
\begin{equation}
    \frac{1}{|\{i: Y_i=Y_p\}|}\sum_{i: Y_i=Y_p}\bfx_i^t\sim\mathcal{N}(\bfmu_{Y_p}, \frac{(\bar{\alpha}_t\rho^2+(1-\bar{\alpha}_t))\cdot \bfI}{|\{i: Y_i=Y_p\}|}),
\end{equation}
we have
\begin{equation}
    \mathbb{E}[\|\bfM_\bfY-\bfA(\bfX^t)\|^2]\leq \sum_{p=1}^P\frac{\rho^2}{|\{i: Y_i=Y_p\}|}=\rho^2 K.
\end{equation}
Hence, we can obtain
\begin{equation}
    R_{\mathrm{Bayes}}-R_{\mathrm{oracle}}\leq (\frac{1-\bar{\alpha}_t}{1-\bar{\alpha}_t+\rho^2\bar{\alpha}_t})^2\cdot \frac{\rho^2}{|\{i: Y_i=Y_p\}|}\leq \epsilon,
\end{equation}
as long as for any $p\in [P]$, $t\in [T]$,
\begin{equation}
    |\{i: Y_i=Y_p\}|\geq \epsilon^{-1}\rho^2,
\end{equation}
which holds if $P\geq \nu_{\min}^{\tilde{\boldsymbol{\pi}}}(K)^{-1}(\rho^2+1)\epsilon^{-1}\log d$.

\end{proof}

\subsection{Proof of Proposition \ref{prpst: W}}
\begin{proof}
    By (\ref{Z lower bound}), we have that for $Y_i=Y_j$,
    \begin{equation}
    \begin{aligned}
        &\bfx_j^t{}^\top\bfW^{(S)}\bfx_i^t/d\\
        \gtrsim & \frac{1}{2}\log (\epsilon^{-1}(\frac{(1-\max_{u\in[M]}\nu_u^{\tilde{\boldsymbol{\pi}}}(K))}{\max_{u\in[M]}\nu_u^{\tilde{\boldsymbol{\pi}}}(K)})^2)\\
        \gtrsim & \frac{1}{2}\log \epsilon^{-1}K\delta(\tilde{\boldsymbol{\pi}}),
        \end{aligned}
    \end{equation}
    where the last step holds since
    \begin{equation}
        \max_{u\in[M]}\nu_u^{\tilde{\boldsymbol{\pi}}}(K)\leq \frac{\max_{u\in[M]}\tilde{\boldsymbol{\pi}}}{\max_{u\in[M]}\tilde{\boldsymbol{\pi}}+(K-1)\min_{u\in[M]}\tilde{\boldsymbol{\pi}}}\leq \frac{1}{K\delta(\tilde{\boldsymbol{\pi}})}.
    \end{equation}
    This leads to (\ref{eqn: qk same}). (\ref{eqn: qk similar}) comes from an accumulation of (\ref{qk small eqnlemma}).
\end{proof}

\subsection{Proof of Corollary \ref{cor: attention}}
\begin{proof}
    (\ref{eqn: attention_concentration}) comes from (\ref{1-zeta final}). (\ref{eqn: attention_uniform}) is derived by (\ref{zeta_uniform}) plus the training condition (i) in Theorem \ref{thm: training}.
\end{proof}

\subsection{Proof of Corollar \ref{cor: shift}}
\begin{proof}
    We still need the condition (\ref{S_lb_1}) to hold for generation on $\mathcal{D}(\tilde{\boldsymbol{\pi}}', K, \{\bfmu_i\}_{i=1}^M, \rho)$. Then, the required number of tokens per data becomes
    \begin{equation}
        P\gtrsim \nu_{\min}^{\tilde{\boldsymbol{\pi}}'}(K)^{-1}(\rho^2+1)\epsilon^{-1}\log d.
    \end{equation}
    Combining Corollary (\ref{cor: R_oracle}), we can obtain the desired result.
\end{proof}

\section{Proof of Key Lemmas}\label{sec: proof lemmas}

\subsection{Proof of Lemma \ref{lemma: W training}}
\begin{proof}

\noindent Let $\bfE=(\boldsymbol{\epsilon}_1,\boldsymbol{\epsilon}_2,\cdots,\boldsymbol{\epsilon}_P)$. Then, we can compute
\begin{equation}
    \begin{aligned}
         & \frac{\partial L(\Psi^{(s_0)})}{\partial \bfW}\\
        =& \sum_{p=1}^P \frac{\partial \|\bff(\bfW;\sqrt{\bar{\alpha}_t}\bfX^0+\sqrt{1-\bar{\alpha}_t}\bfE,t)_p-\boldsymbol{\epsilon}_p\|^2}{\partial \bfW}\cdot \frac{1}{dP}\\
        =& \sum_{p=1}^P (\bff(\bfW;\sqrt{\bar{\alpha}_t}\bfX^0+\sqrt{1-\bar{\alpha}_t}\bfE,t)_p-\boldsymbol{\epsilon}_p)^\top\frac{\partial \bff(\bfW;\sqrt{\bar{\alpha}_t}\bfX^0+\sqrt{1-\bar{\alpha}_t}\bfE,t)}{\partial \bfW}\cdot \frac{1}{dP}\\
        =& \sum_{p=1}^P(\bff(\bfW;\sqrt{\bar{\alpha}_t}\bfX^0+\sqrt{1-\bar{\alpha}_t}\bfE,t)_p-\boldsymbol{\epsilon}_p)^\top (-\frac{v_t}{d^2P})\sum_{i=1}^P \bfx_i^{t}\text{softmax}_p(\frac{{\bfx_i^t} {}^\top\bfW\bfx_p^t}{d})\\
        &\cdot (\bfx_i^t-\sum_{r=1}^P\text{softmax}_p(\frac{{\bfx_r^t}{}^\top\bfW\bfx_p^t)\bfx_r^t}{d}){\bfx_p^t}{}^\top.\label{grad_ini}
    \end{aligned}
\end{equation}
We then complete the proof using induction. When iterations $s=0$, we have that with a high probability, 
\begin{equation}
    \|\bfx_p^t\|^2\gtrsim (1+\bar{\alpha}_t\rho^2) d\log d.
\end{equation}
If $\bfx_i^t$ and $\bfx_p^t$ share the same feature $\bfmu_u$ as the mean, then
\begin{equation}
    \begin{aligned}
        &\bfmu_u^\top (\bfx_i^t-\sum_{r=1}^P\text{softmax}_p(\frac{{\bfx_r^t}{}^\top\bfW^{(0)}\bfx_p^t}{d})\bfx_r^t){\bfx_p^t}{}^\top\bfmu_u\\
        =&(\sqrt{\bar{\alpha}_t}d+\bfmu_u^\top(\bfx_i^t-\sqrt{\bar{\alpha}_t}\bfmu_u)-\sum_{r=1}^P\text{softmax}_p(\frac{{\bfx_r^t}{}^\top\bfW^{(0)}\bfx_p^t}{d})\bfmu_u^\top\bfx_r^t)(\sqrt{\bar{\alpha}_t}d+\bfmu_u^\top(\bfx_p^t-\sqrt{\bar{\alpha}_t}\bfmu_u)).
    \end{aligned}
\end{equation}
Consider $\tilde{\bfX}^0$ independently sampled from Definition x so that $\bfx_j^0$ and $\bfx_{j'}^0$ shares the same mean as $\bfx_p^0$. $\tilde{\bfX}^0\neq \bfX^0$. Then,
\begin{equation}
    \begin{aligned}
        &\tilde{\bfx}_j^t{}^\top (\bfx_i^t-\sum_{r=1}^P\text{softmax}_p(\frac{{\bfx_r^t}{}^\top\bfW^{(0)}\bfx_p^t}{d})\bfx_r^t){\bfx_p^t}{}^\top\tilde{\bfx}_{j'}^t\\
        \gtrsim &(\bar{\alpha}_t d-\sum_{r=1}^P\text{softmax}_p(\frac{{\bfx_r^t}{}^\top\bfW^{(0)}\bfx_p^t}{d})\tilde{\bfx}_j^t{}^\top\bfx_r^t)\cdot \bar{\alpha}_t d,\label{it0_p1}
    \end{aligned}
\end{equation}
Meanwhile, 
\begin{equation}
    \begin{aligned}
        &(\bff(\bfW;\sqrt{\bar{\alpha}_t}\bfX^0+\sqrt{1-\bar{\alpha}_t}\bfE,t)_p-\boldsymbol{\epsilon}_p)^\top\bfx_i^t\\
        =& (v_t^{(0)}\sqrt{\bar{\alpha}_t}\rho\boldsymbol{\epsilon}_p'+v_t^{(0)}(\sqrt{\bar{\alpha}_t}\bfmu_u-\sum_{i=1}^P\bfx_i^t\text{softmax}_p(\frac{{\bfx_i^t}{}^\top\bfW^{(0)}\bfx_p^t}{d}))-(v_t^{(0)}\sqrt{1-\bar{\alpha}_t}-1)\boldsymbol{\epsilon}_p)(\sqrt{\bar{\alpha}_t}(\bfmu_u+\rho\boldsymbol{\epsilon}'_p)\\
        &+\sqrt{1-\bar{\alpha}_t}\boldsymbol{\epsilon}_p).\label{it0_p2}
    \end{aligned}
\end{equation}
If $\bfx_i^t$ and $\bfx_p^t$ does not share the same feature $\bfmu_u$ as the mean, then for $\tilde{\bfx}_j^t$ and $\tilde{\bfx}_{j'}^t$ that is from another $\tilde{\bfX}^t$ with the same mean as $\bfx_p^t$, 
\begin{equation}
    \begin{aligned}
        &\tilde{\bfx}_j^\top (\bfx_i^t-\sum_{r=1}^P\text{softmax}_p(\frac{{\bfx_r^t}{}^\top\bfW^{(0)}\bfx_p^t}{d})\bfx_r^t){\bfx_p^t}{}^\top\tilde{\bfx}_{j'}\\
        =&(\tilde{\bfx}_j^\top(\bfx_i^t-\sqrt{\bar{\alpha}_t}\bfmu_{u'})-\sum_{r=1}^P\text{softmax}_p(\frac{{\bfx_r^t}{}^\top\bfW^{(0)}\bfx_p^t}{d})\tilde{\bfx}_j^\top\bfx_r^t)\cdot (\bar{\alpha}_t d+\tilde{\bfx}_{j'}^\top(\bfx_p^t-\sqrt{\bar{\alpha}_t}\bfmu_u)),\label{it0_p12}
    \end{aligned}
\end{equation}
\begin{equation}
    \begin{aligned}
        &(\bff(\bfW;\sqrt{\bar{\alpha}_t}\bfX^0+\sqrt{1-\bar{\alpha}_t}\bfE,t)_p-\boldsymbol{\epsilon}_p)^\top\bfx_i^t\\
        =& (v_t^{(0)}\sqrt{\bar{\alpha}_t}\rho\boldsymbol{\epsilon}_p'+v_t^{(0)}(\sqrt{\bar{\alpha}_t}\bfmu_u-\sum_{i=1}^P\bfx_i^t\text{softmax}_p(\frac{{\bfx_i^t}{}^\top\bfW^{(0)}\bfx_p^t}{d}))-(v_t^{(0)}\sqrt{1-\bar{\alpha}_t}-1)\boldsymbol{\epsilon}_p)^\top(\sqrt{\bar{\alpha}_t}(\bfmu_{u'}+\rho\boldsymbol{\epsilon}'_i)\\
        &+\sqrt{1-\bar{\alpha}_t}\boldsymbol{\epsilon}_i),\label{it0_p22}
    \end{aligned}
\end{equation}
if the mean of $\bfx_i^t$ is $\bfmu_{u'}$. 
Therefore, combining (\ref{it0_p1}) and (\ref{it0_p2}), 
we have that for $\tilde{\bfx}_j^t$ and $\tilde{\bfx}_{j'}^t$ that share the same mean as $\bfx_p^t$,
\begin{equation}
    \begin{aligned}
        &(-\tilde{\bfx}_j^t{}^\top)\mathbb{E}_{\bfE,\bfX^0}\big[(\bff(\bfW^{(0)};\sqrt{\bar{\alpha}_t}\bfX^0+\sqrt{1-\bar{\alpha}_t}\bfE,t)_p-\boldsymbol{\epsilon}_p)^\top (-\frac{v_t^{(0)}}{d^2P}) \sum_{i=1}^P\bfx_i^{t}\text{softmax}_p(\frac{{\bfx_i^t} {}^\top\bfW^{(0)}\bfx_p^t}{d})\\
        &\cdot (\bfx_i^t-\sum_{r=1}^P\text{softmax}(\frac{{\bfx_r^t}{}^\top\bfW^{(0)}\bfx_p^t}{d})\bfx_r^t){\bfx_p^t}{}^\top\big]\tilde{\bfx}_{j'}^t\\
        \gtrsim & \mathbb{E}_{\bfE,\bfX^0}\big[\frac{1}{P}\mathbbm{1}[Y_p=u]\sum_{i=1}^P\zeta_{i,p,t}(0)(v_t^{(0)})^2\cdot (d\bar{\alpha}_t^3(1-\sum_{l=1}^P\zeta^u_{l,p,t}(0))^2+\bar{\alpha}_t^3d\rho^2-\bar{\alpha}_t^2(1-\bar{\alpha}_t-\frac{\sqrt{1-\bar{\alpha}_t}}{v_t^{(0)}})d\mathbbm{1}[i=p]\\
        &-d\bar{\alpha}_t^3\zeta_{i,p,t}^{u}(0)\rho^2)\big],
    \end{aligned}
\end{equation}
\begin{equation}
    \begin{aligned}
        &\frac{1}{T}\sum_{t=1}^T(-\tilde{\bfx}_j^t{}^\top)\mathbb{E}_{\bfE,\bfX^0}\big[\sum_{p=1}^P(\bff(\bfW^{(1)};\sqrt{\bar{\alpha}_t}\bfX^0+\sqrt{1-\bar{\alpha}_t}\bfE,t)_p-\boldsymbol{\epsilon}_p)^\top (-\frac{v_t^{(1)}}{d^2P}) \sum_{i=1}^P\bfx_i^{t}\\
        &\cdot\text{softmax}_p(\frac{{\bfx_i^t} {}^\top\bfW^{(1)}\bfx_p^t}{d})\cdot (\bfx_i^t-\sum_{r=1}^P\text{softmax}_p(\frac{{\bfx_r^t}{}^\top\bfW^{(1)}\bfx_p^t}{d})\bfx_r^t){\bfx_p^t}{}^\top\big]\tilde{\bfx}_{j'}^t\\
        \gtrsim & \frac{1}{T}\sum_{t=1}^T\sum_{p=1}^P\frac{\mathbbm{1}[Y_p=u]}{P^2}\mathbb{E}_{\bfE,\bfX^0}[v_t^{(0)}\bar{\alpha}_t^2 \sqrt{1-\bar{\alpha}_t} d (1-\sum_{l=1}^P \zeta_{l,p,t}^u(0))^2]\\
        \gtrsim & v_t^{(0)}\cdot \frac{d}{P},
    \end{aligned}
\end{equation}
which means that the update from $\bfW^{(0)}$ to $\bfW^{(1)}$ almost makes no difference to the attention map. 
Note that 
\begin{equation}
    \begin{aligned}
        &\frac{\partial \mathbb{E}_{\bfE,\bfX^0} [\|v_t^{(0)} (\bfx_p^t-\sum_{i=1}^P \bfx_i^t\text{softmax}_p(\frac{{\bfx_i^t}^\top\bfW^{(0)}\bfx^t_p}{d}))-\boldsymbol{\epsilon}_p\|^2/d]}{\partial v_t}\\
         =& v_t^{(0)}\rho^2\bar{\alpha}_t+\sqrt{1-\bar{\alpha}_t}(v_t^{(0)}\sqrt{1-\bar{\alpha}_t}-1)+\mathbb{E}_{\bfE,\bfX^0}[v_t^{(0)}\|\sqrt{\bar{\alpha}}_t\bfmu_u-\sum_{i=1}^P \bfx_i^t\text{softmax}_p(\frac{{\bfx_i^t}^\top\bfW^{(0)}\bfx^t_p}{d})\|^2/d]\\
        &+ 2\mathbb{E}_{\bfE,\bfX^0}[\frac{\beta_1 v_t^{(0)}\rho\sqrt{\bar{\alpha}_t}}{d}]+\mathbb{E}_{\bfE,\bfX^0}[\frac{2\beta_2\sqrt{1-\bar{\alpha}_t}v_t^{(0)}-\beta_2}{d}]\\
        =& (v_t^{(0)}\cdot \mathbb{E}_{\bfE,\bfX^0}[\bar{\alpha}_t\rho^2+1-\bar{\alpha}_t+\|\sqrt{\bar{\alpha}}_t\bfmu_u-\sum_{i=1}^P \bfx_i^t\text{softmax}_p(\frac{{\bfx_i^t}^\top\bfW^{(0)}\bfx^t_p}{d})\|^2/d\\
        &+\frac{2\beta_1\rho\sqrt{\bar{\alpha}_t}}{d}+\frac{2\beta_2\sqrt{1-\bar{\alpha}_t}}{d}]-(\sqrt{1-\bar{\alpha}_t}+\frac{\beta_2}{d})),
    \end{aligned}
\end{equation}
where we denote $\beta_1=\boldsymbol{\epsilon}_p'{}^\top (\sqrt{\bar{\alpha}}_t\bfmu_u-\sum_{i=1}^P \bfx_i^t\text{softmax}_p(\frac{{\bfx_i^t}^\top\bfW^{(0)}\bfx^t_p}{d}))$ and $\beta_2=\boldsymbol{\epsilon}_p^\top (\sqrt{\bar{\alpha}}_t\bfmu_u-\sum_{i=1}^P \bfx_i^t\text{softmax}_p(\frac{{\bfx_i^t}^\top\bfW^{(0)}\bfx^t_p}{d}))$. We have $\beta_1,\beta_2\lesssim \sqrt{d}$. Hence, since $|v_t^{(0)}|\lesssim \sqrt{\log d/d}$, we have
\begin{equation}
\begin{aligned}
    v_t^{(1)}=&v_t^{(0)}-\eta \partial\mathbb{E}_{\bfE,\bfX^0} [\|v_t^{(0)}(\bfx_p^t-\sum_{i=1}^P \bfx_i^t\text{softmax}_p(\frac{{\bfx_i^t}^\top\bfW^{(0)}\bfx^t_p}{d}))-\boldsymbol{\epsilon}_p\|^2/d]/\partial v_t\\
    <& \eta\sqrt{1-\bar{\alpha}_t},
    \end{aligned}
\end{equation}
and
\begin{equation}
    1-\bar{\alpha}_t-\frac{\sqrt{1-\bar{\alpha}_t}}{v_t^{(0)}}<1-\eta^{-1}-\bar{\alpha}_t<0.
\end{equation}
Then, we can obtain
\begin{equation}
    \begin{aligned}
        &(-\tilde{\bfx}_j^t{}^\top)\mathbb{E}_{\bfE,\bfX^0}\big[(\bff(\bfW^{(1)};\sqrt{\bar{\alpha}_t}\bfX^0+\sqrt{1-\bar{\alpha}_t}\bfE,t)_p-\boldsymbol{\epsilon}_p)^\top (-\frac{v_t^{(1)}}{d^2P}) \sum_{i=1}^P\bfx_i^{t}\text{softmax}_p(\frac{{\bfx_i^t} {}^\top\bfW^{(1)}\bfx_p^t}{d})\\
        &\cdot (\bfx_i^t-\sum_{r=1}^P\text{softmax}(\frac{{\bfx_r^t}{}^\top\bfW^{(1)}\bfx_p^t}{d})\bfx_r^t){\bfx_p^t}{}^\top\big]\tilde{\bfx}_{j'}^t\\
        \gtrsim & \mathbb{E}_{\bfE,\bfX^0}\big[\frac{1}{P}\mathbbm{1}[Y_p=u]\sum_{i=1}^P\zeta_{i,p,t}(1)(v_t^{(1)})^2\cdot (d\bar{\alpha}_t^3(1-\sum_{l=1}^P\zeta^u_{l,p,t}(1))^2+\bar{\alpha}_t^3d\rho^2-\bar{\alpha}_t^2(1-\bar{\alpha}_t-\frac{\sqrt{1-\bar{\alpha}_t}}{v_t^{(1)}})d\mathbbm{1}[i=p]\\
        &-d\bar{\alpha}_t^3\zeta_{i,p,t}^{u}(1)\rho^2)\big]\\
        \gtrsim & \mathbb{E}_{\bfE,\bfX^0}\big[\frac{1}{P}\mathbbm{1}[Y_p=u]\sum_{i=1}^P\zeta_{i,p,t}(1)(v_t^{(1)})^2\cdot (d\bar{\alpha}_t^3(1-\sum_{l=1}^P\zeta^u_{l,p,t}(1))^2\big],\\\label{it0_i!=p}
    \end{aligned}
\end{equation}
where the last step holds if $v_t^{(0)}>0$. 
Therefore, by summing up over (\ref{it0_i!=p}), we can obtain that for any $\tilde{\bfx}_j^t$ and $\tilde{\bfx}_{j'}$ with $\bfmu_u$ as the mean, 
\begin{equation}
    \begin{aligned}
        &\frac{1}{T}\sum_{t=1}^T(-\tilde{\bfx}_j^t{}^\top)\mathbb{E}_{\bfE,\bfX^0}\big[\sum_{p=1}^P(\bff(\bfW^{(1)};\sqrt{\bar{\alpha}_t}\bfX^0+\sqrt{1-\bar{\alpha}_t}\bfE,t)_p-\boldsymbol{\epsilon}_p)^\top (-\frac{v_t^{(1)}}{d^2P}) \sum_{i=1}^P\bfx_i^{t}\\
        &\cdot\text{softmax}_p(\frac{{\bfx_i^t} {}^\top\bfW^{(1)}\bfx_p^t}{d})\cdot (\bfx_i^t-\sum_{r=1}^P\text{softmax}_p(\frac{{\bfx_r^t}{}^\top\bfW^{(1)}\bfx_p^t}{d})\bfx_r^t){\bfx_p^t}{}^\top\big]\tilde{\bfx}_{j'}^t\\
        \gtrsim & \frac{1}{T}\sum_{t=1}^T\sum_{p=1}^P\frac{\mathbbm{1}[Y_p=u]}{P}\mathbb{E}_{\bfE,\bfX^0}[(v_t^{(1)})^2\bar{\alpha}_t^3 d (1-\sum_{l=1}^P \zeta_{l,p,t}^u(0))^2\sum_{i=1}^P\zeta_{i,p,t}^u(0)]\\
        \gtrsim &\frac{1}{T}\sum_{t=1}^T\mathbb{E}_{\bfE,\bfX^0}[(v_t^{(1)})^2\bar{\alpha}_t^3 d (1-\sum_{l=1}^P \zeta_{l,p,t}^u(0))^2\sum_{i=1}^P\zeta_{i,p,t}^u(0)]\cdot \nu_u^{\tilde{\boldsymbol{\pi}}}(K),
    \end{aligned}
\end{equation}
where the last step holds with a high probability if
\begin{equation}
    P\cdot \nu_u^{\tilde{\boldsymbol{\pi}}}(K)\gtrsim \log dK,
\end{equation}
by Lemma \ref{lemma: chernoff}. 
\noindent For $\bfmu_u\neq \bfmu_{u'}$, we consider two cases of $\bfx_i^t$ as follows, where the mean of $\tilde{\bfmu}_k$ is $\bfmu_{u'}$.
\begin{enumerate}
    \item The corresponding mean of $\bfx_i^t$ is $\bfmu_{u'}$: Then, $\bfmu_{u'}^\top (\bfx_i^t-\sum_{r=1}^P\text{softmax}_p(\frac{{\bfx_r^t}{}^\top\bfW^{(1)}\bfx_p^t}{d})\bfx_r^t){\bfx_p^t}{}^\top\bfmu_u$ is in the order of $d^2$. $(\bff(\bfW^{(1)};\sqrt{\bar{\alpha}_t}\bfX^0+\sqrt{1-\bar{\alpha}_t}\bfE,t)_p-\boldsymbol{\epsilon}_p)^\top\bfx_i^t$ is in the order of $-d$ if $d$ is large enough, and the order of positive term is no more than $\sqrt{d}\log d$.
    \item The corresponding mean of $\bfx_i^t$ is not $\bfmu_{u'}$: Then, the order of $\bfmu_{u'}^\top (\bfx_i^t-\sum_{r=1}^P\text{softmax}_p(\frac{{\bfx_r^t}{}^\top\bfW^{(1)}\bfx_p^t}{d})\bfx_r^t){\bfx_p^t}{}^\top\bfmu_u$ is at most $d^\frac{3}{2}\log d$, which is already smaller than the order of $d^2$. $(\bff(\bfW^{(1)};\sqrt{\bar{\alpha}_t}\bfX^0+\sqrt{1-\bar{\alpha}_t}\bfE,t)_p-\boldsymbol{\epsilon}_p)^\top\bfx_i^t$ is in the order of at most $\sqrt{d}\log d$. 
\end{enumerate}

\noindent The above discussion indicates that for $\tilde{\bfx}_j$ and $\tilde{\bfx}_{j'}$ with the mean of $\bfmu_u$ and $\tilde{\bfx}_{k}$ with the mean of $\bfmu_{u'}$, 
\begin{equation}
    \begin{aligned}
        &(-\tilde{\bfx}_{k}^\top)\mathbb{E}_{\bfE,\bfX^0}\big[(\bff(\bfW^{(1)};\sqrt{\bar{\alpha}_t}\bfX^0+\sqrt{1-\bar{\alpha}_t}\bfE,t)_p-\boldsymbol{\epsilon}_p)^\top (-\frac{v_t^{(1)}}{d^2P}) \sum_{i=1}^P\bfx_i^{t}\\
        &\cdot \text{softmax}_p(\frac{{\bfx_i^t} {}^\top\bfW^{(1)}\bfx_p^t}{d})\cdot (\bfx_i^t-\sum_{r=1}^P\text{softmax}_p(\frac{{\bfx_r^t}{}^\top\bfW^{(1)}\bfx_p^t}{d})\bfx_r^t){\bfx_p^t}{}^\top\big]\tilde{\bfx}_j\\
        \lesssim &\frac{\log d}{\sqrt{d}}\cdot (-\tilde{\bfx}_{j'}^\top)\mathbb{E}_{\bfE,\bfX^0}\big[(\bff(\bfW^{(1)};\sqrt{\bar{\alpha}_t}\bfX^0+\sqrt{1-\bar{\alpha}_t}\bfE,t)_p-\boldsymbol{\epsilon}_p)^\top (-\frac{v_t^{(1)}}{d^2P}) \sum_{i=1}^P\bfx_i^{t}\\
        &\cdot \text{softmax}_p(\frac{{\bfx_i^t} {}^\top\bfW^{(1)}\bfx_p^t}{d})\cdot (\bfx_i^t-\sum_{r=1}^P\text{softmax}_p(\frac{{\bfx_r^t}{}^\top\bfW^{(1)}\bfx_p^t}{d})\bfx_r^t){\bfx_p^t}{}^\top\big]\tilde{\bfx}_{j}\\
    \end{aligned}
\end{equation}
\begin{equation}
    \begin{aligned}
        &\frac{1}{T}\sum_{t=1}^T(-\tilde{\bfx}_{k}^\top)\mathbb{E}_{\bfE,\bfX^0}\big[(\bff(\bfW^{(1)};\sqrt{\bar{\alpha}_t}\bfX^0+\sqrt{1-\bar{\alpha}_t}\bfE,t)_p-\boldsymbol{\epsilon}_p)^\top (-\frac{v_t^{(1)}}{d^2P}) \sum_{i=1}^P\bfx_i^{t}\\
        &\cdot \text{softmax}_p(\frac{{\bfx_i^t} {}^\top\bfW^{(1)}\bfx_p^t}{d})\cdot (\bfx_i^t-\sum_{r=1}^P\text{softmax}_p(\frac{{\bfx_r^t}{}^\top\bfW^{(1)}\bfx_p^t}{d})\bfx_r^t){\bfx_p^t}{}^\top\big]\tilde{\bfx}_j\\
        \leq &\frac{\log d}{\sqrt{d}}\cdot  (-\tilde{\bfx}_{j'}^\top)\frac{1}{T}\sum_{t=1}^T\mathbb{E}_{\bfE,\bfX^0}\big[(\bff(\bfW^{(1)};\sqrt{\bar{\alpha}_t}\bfX^0+\sqrt{1-\bar{\alpha}_t}\bfE,t)_p-\boldsymbol{\epsilon}_p)^\top (-\frac{v_t^{(1)}}{d^2P}) \sum_{i=1}^P\bfx_i^{t}\\
        &\cdot \text{softmax}_p(\frac{{\bfx_i^t} {}^\top\bfW^{(1)}\bfx_p^t}{d})\cdot (\bfx_i^t-\sum_{r=1}^P\text{softmax}_p(\frac{{\bfx_r^t}{}^\top\bfW^{(1)}\bfx_p^t}{d})\bfx_r^t){\bfx_p^t}{}^\top\big]\tilde{\bfx}_j.
    \end{aligned}
\end{equation}

\noindent Suppose that when iterations $s=s_0>1$, the conclusion holds. Then, when $s=s_0+1$, we have that if $\bfx_i^t$ and $\bfx_p^t$ share the same feature $\bfmu_u$ as the mean, then
\begin{equation}
    \begin{aligned}
        &\bfmu_u^\top (\bfx_i^t-\sum_{r=1}^P\text{softmax}_p(\frac{{\bfx_r^t}{}^\top\bfW^{(s_0)}\bfx_p^t}{d})\bfx_r^t){\bfx_p^t}{}^\top\bfmu_u\\
        =&(\sqrt{\bar{\alpha}_t}d+\bfmu_u^\top(\bfx_i^t-\sqrt{\bar{\alpha}_t}\bfmu_u)-\sum_{r=1}^P\text{softmax}_p(\frac{{\bfx_r^t}{}^\top\bfW^{(s_0)}\bfx_p^t}{d})\bfmu_u^\top\bfx_r^t)(\sqrt{\bar{\alpha}_t}d+\bfmu_u^\top(\bfx_p^t-\sqrt{\bar{\alpha}_t}\bfmu_u)).\label{it0_p1_new0}
    \end{aligned}
\end{equation}
Consider $\tilde{\bfX}^0$ independently sampled from Definition x so that $\bfx_j^0$ and $\bfx_{j'}^0$ share the same mean as $\bfx_p^0$. $\tilde{\bfX}^0\neq \bfX^0$. Then,
\begin{equation}
    \begin{aligned}
        &\tilde{\bfx}_j^t{}^\top (\bfx_i^t-\sum_{r=1}^P\text{softmax}_p(\frac{{\bfx_r^t}{}^\top\bfW^{(s_0)}\bfx_p^t}{d})\bfx_r^t){\bfx_p^t}{}^\top\tilde{\bfx}_{j'}^t\\
        \gtrsim &(\bar{\alpha}_t d-\sum_{r=1}^P\text{softmax}_p(\frac{{\bfx_r^t}{}^\top\bfW^{(s_0)}\bfx_p^t}{d})\tilde{\bfx_j}^t{}^\top\bfx_r^t)\cdot \bar{\alpha}_t d,\label{it0_p1_new}
    \end{aligned}
\end{equation}
Meanwhile, 
\begin{equation}
    \begin{aligned}
        &(\bff(\bfW^{(s_0)};\sqrt{\bar{\alpha}_t}\bfX^0+\sqrt{1-\bar{\alpha}_t}\bfE,t)_p-\boldsymbol{\epsilon}_p)^\top\bfx_i^t\\
        =& (v_t^{(s_0)}\sqrt{\bar{\alpha}_t}\rho\boldsymbol{\epsilon}_p'+v_t^{(s_0)}(\sqrt{\bar{\alpha}_t}\bfmu_u-\sum_{i=1}^P\bfx_i^t\text{softmax}_p(\bfx_i^t{}^\top\bfW^{(s_0)}\bfx_p^t))-(v_t^{(s_0)}\sqrt{1-\bar{\alpha}_t}-1)\boldsymbol{\epsilon}_p)(\sqrt{\bar{\alpha}_t}(\bfmu_u+\rho\boldsymbol{\epsilon}'_p)\\
        &+\sqrt{1-\bar{\alpha}_t}\boldsymbol{\epsilon}_p).\label{it0_p2_new}
    \end{aligned}
\end{equation}
If $\bfx_i^t$ and $\bfx_p^t$ do not share the same feature $\bfmu_u$ as the mean, then for $\bfx_j^t$ and $\bfx_{j'}^t$ that share the same mean as $\bfx_p^t$,
\begin{equation}
    \begin{aligned}
        &\tilde{\bfx}_j^\top (\bfx_i^t-\sum_{r=1}^P\text{softmax}_p(\frac{{\bfx_r^t}{}^\top\bfW^{(s_0)}\bfx_p^t}{d})\bfx_r^t){\bfx_p^t}{}^\top\tilde{\bfx}_{j'}\\
        =&(\tilde{\bfx}_j^\top(\bfx_i^t-\sqrt{\bar{\alpha}_t}\bfmu_{u'})-\sum_{r=1}^P\text{softmax}_p(\frac{{\bfx_r^t}{}^\top\bfW^{(s_0)}\bfx_p^t}{d})\tilde{\bfx}_j^\top\bfx_r^t)\cdot (\sqrt{\bar{\alpha}_t}d+\tilde{\bfx}_{j'}^\top(\bfx_p^t-\sqrt{\bar{\alpha}_t}\bfmu_u)),
    \end{aligned}
\end{equation}
\begin{equation}
    \begin{aligned}
        &(\bff(\bfW^{(s_0)};\sqrt{\bar{\alpha}_t}\bfX^0+\sqrt{1-\bar{\alpha}_t}\bfE,t)_p-\boldsymbol{\epsilon}_p)^\top\bfx_i^t\\
        =& (v_t^{(s_0)}\sqrt{\bar{\alpha}_t}\rho\boldsymbol{\epsilon}_p'+v_t^{(s_0)}(\sqrt{\bar{\alpha}_t}\bfmu_u-\sum_{i=1}^P\bfx_i^t\text{softmax}_p(\bfx_i^t{}^\top\bfW^{(s_0)}\bfx_p^t))-(v_t^{(s_0)}\sqrt{1-\bar{\alpha}_t}-1)\boldsymbol{\epsilon}_p)(\sqrt{\bar{\alpha}_t}(\bfmu_{u'}+\rho\boldsymbol{\epsilon}'_i)\\
        &+\sqrt{1-\bar{\alpha}_t}\boldsymbol{\epsilon}_i),
    \end{aligned}
\end{equation}
if the mean of $\bfx_i^t$ is $\bfmu_{u'}$. Therefore, combining (\ref{it0_p1_new}) and (\ref{it0_p2_new}), 
we have that for $\tilde{\bfx}_j^t$ and $\tilde{\bfx}_{j'}^t$ from another $\tilde{\bfX}^t$ that share the same mean as $\bfx_p^t$,
\begin{equation}
    \begin{aligned}
        &(-\tilde{\bfx}_j^t{}^\top)\mathbb{E}_{\bfE,\bfX^0}\big[(\bff(\bfW^{(s_0)};\sqrt{\bar{\alpha}_t}\bfX^0+\sqrt{1-\bar{\alpha}_t}\bfE,t)_p-\boldsymbol{\epsilon}_p)^\top (-\frac{v_t^{(s_0)}}{d^2P}) \sum_{i=1}^P\bfx_i^{t}\text{softmax}_p(\frac{{\bfx_i^t} {}^\top\bfW^{(s_0)}\bfx_p^t}{d})\\
        &\cdot (\bfx_i^t-\sum_{r=1}^P\text{softmax}(\frac{{\bfx_r^t}{}^\top\bfW^{(s_0)}\bfx_p^t}{d})\bfx_r^t){\bfx_p^t}{}^\top\big]\tilde{\bfx}_{j'}^t\\
        \gtrsim & \mathbb{E}_{\bfE,\bfX^0}\big[\frac{1}{P}\mathbbm{1}[Y_p=u]\sum_{i=1}^P\zeta_{i,p,t}(s_0)(v_t^{(s_0)})^2\cdot (d\bar{\alpha}_t^3(1-\sum_{l=1}^P\zeta^u_{l,p,t}(s_0))+\bar{\alpha}_t^3d\rho^2-\bar{\alpha}_t^2(1-\bar{\alpha}_t-\frac{\sqrt{1-\bar{\alpha}_t}}{v_t^{(s_0)}})d\mathbbm{1}[i=p]\\
        &-d\bar{\alpha}_t^3\sum_{p=1}^P\zeta^u_{i,p,t}(s_0)\rho^2)\big]\\
        \gtrsim & \mathbb{E}_{\bfE,\bfX^0}\big[\frac{1}{P}\mathbbm{1}[Y_p=u]\sum_{i=1}^P\zeta_{i,p,t}(s_0)(v_t^{(s_0)})^2\cdot (d\bar{\alpha}_t^3(1-\sum_{l=1}^P\zeta^u_{l,p,t}(s_0))^2\big],\label{it0_i!=p_new}
    \end{aligned}
\end{equation}
where the last step holds because if $v_t^{(1)}<\frac{\sqrt{1-\bar{\alpha}_t}}{\bar{\alpha}_t\rho^2+1-\bar{\alpha}_t}+o(1)$, then $v_t^{(s_0)}<\frac{\sqrt{1-\bar{\alpha}_t}}{\bar{\alpha}_t\rho^2+1-\bar{\alpha}_t}+o(1)$ and
\begin{equation}
    1-\bar{\alpha}_t-\frac{\sqrt{1-\bar{\alpha}_t}}{v_t^{(s_0)}}<-\rho^2<0;
\end{equation}
and if $v_t^{(1)}>\frac{\sqrt{1-\bar{\alpha}_t}}{\bar{\alpha}_t\rho^2+1-\bar{\alpha}_t}+o(1)$, then $v_t^{(s_0)}>\frac{\sqrt{1-\bar{\alpha}_t}}{\bar{\alpha}_t\rho^2+1-\bar{\alpha}_t}+o(1)$ and
\begin{equation}
    1-\bar{\alpha}_t-\frac{\sqrt{1-\bar{\alpha}_t}}{v_t^{(0)}}<0.
\end{equation}

\noindent For $\bfmu_u\neq \bfmu_{u'}$ and $\tilde{\bfx}_j$, $\tilde{\bfx}_{j'}$ with the mean of $\bfmu_u$ and $\tilde{\bfx}_{k}$ with the mean of $\bfmu_{u'}$, we have that 
\begin{equation}
    \begin{aligned}
        &(-\tilde{\bfx}_{k}^\top)\mathbb{E}_{\bfE,\bfX^0}\big[(\bff(\bfW^{(s_0)};\sqrt{\bar{\alpha}_t}\bfX^0+\sqrt{1-\bar{\alpha}_t}\bfE,t)_p-\boldsymbol{\epsilon}_p)^\top (-\frac{v_t^{(s_0)}}{d^2P}) \sum_{i=1}^P\bfx_i^{t}\\
        &\cdot \text{softmax}_p(\frac{{\bfx_i^t} {}^\top\bfW^{(s_0)}\bfx_p^t}{d})\cdot (\bfx_i^t-\sum_{r=1}^P\text{softmax}_p(\frac{{\bfx_r^t}{}^\top\bfW^{(s_0)}\bfx_p^t}{d})\bfx_r^t){\bfx_p^t}{}^\top\big]\tilde{\bfx}_{j'}\\
        \leq &\frac{\log d}{\sqrt{d}}\cdot (-\tilde{\bfx}_{j}^\top)\mathbb{E}_{\bfE,\bfX^0}\big[(\bff(\bfW^{(s_0)};\sqrt{\bar{\alpha}_t}\bfX^0+\sqrt{1-\bar{\alpha}_t}\bfE,t)_p-\boldsymbol{\epsilon}_p)^\top (-\frac{v_t^{(s_0)}}{d^2P}) \sum_{i=1}^P\bfx_i^{t}\\
        &\cdot \text{softmax}_p(\frac{{\bfx_i^t} {}^\top\bfW^{(s_0)}\bfx_p^t}{d})\cdot (\bfx_i^t-\sum_{r=1}^P\text{softmax}_p(\frac{{\bfx_r^t}{}^\top\bfW^{(s_0)}\bfx_p^t}{d})\bfx_r^t){\bfx_p^t}{}^\top\big]\tilde{\bfx}_{j'}\\
    \end{aligned}
\end{equation}
\begin{equation}
    \begin{aligned}
        &(-\tilde{\bfx}_{k}^\top)\frac{1}{T}\sum_{t=1}^T\mathbb{E}_{\bfE,\bfX^0}\big[(\bff(\bfW^{(s_0)};\sqrt{\bar{\alpha}_t}\bfX^0+\sqrt{1-\bar{\alpha}_t}\bfE,t)_p-\boldsymbol{\epsilon}_p)^\top (-\frac{v_t^{(s_0)}}{d^2P}) \sum_{i=1}^P\bfx_i^{t}\\
        &\cdot \text{softmax}_p(\frac{{\bfx_i^t} {}^\top\bfW^{(s_0)}\bfx_p^t}{d})\cdot (\bfx_i^t-\sum_{r=1}^P\text{softmax}_p(\frac{{\bfx_r^t}{}^\top\bfW^{(s_0)}\bfx_p^t}{d})\bfx_r^t){\bfx_p^t}{}^\top\big]\tilde{\bfx}_{j'}\\
        \leq & \frac{\log d}{\sqrt{d}}\cdot (-\tilde{\bfx}_{j}^\top)\frac{1}{T}\sum_{t=1}^T\mathbb{E}_{\bfE,\bfX^0}\big[(\bff(\bfW^{(s_0)};\sqrt{\bar{\alpha}_t}\bfX^0+\sqrt{1-\bar{\alpha}_t}\bfE,t)_p-\boldsymbol{\epsilon}_p)^\top (-\frac{v_t^{(s_0)}}{d^2P}) \sum_{i=1}^P\bfx_i^{t}\\
        &\cdot \text{softmax}_p(\frac{{\bfx_i^t} {}^\top\bfW^{(s_0)}\bfx_p^t}{d})\cdot (\bfx_i^t-\sum_{r=1}^P\text{softmax}_p(\frac{{\bfx_r^t}{}^\top\bfW^{(s_0)}\bfx_p^t}{d})\bfx_r^t){\bfx_p^t}{}^\top\big]\tilde{\bfx}_{j'}.\label{grad kj}
    \end{aligned}
\end{equation}
Since
\begin{equation}
    \bfW^{(s_0+1)}=\bfW^{(s_0)}-\eta\frac{1}{T}\sum_{t=1}^T\sum_{q=1}^{s_0}\mathbb{E}_{\bfE,\bfX^0,t}[\frac{\partial L(\Psi^{(s_0)};\bfX^t)}{\partial \bfW}],
\end{equation}
where $\frac{1}{T}\sum_{t=1}^T\mathbb{E}_{\bfE,\bfX^0,t}[\frac{\partial L(\Psi^{(s_0)};\bfX^t)}{\partial \bfW}]$ is not a function of any noise term, we have
\begin{equation}
    \begin{aligned}
        &\Big\|\tilde{\bfx}_j^t{}^\top\frac{1}{T}\sum_{t=1}^T\mathbb{E}_{\bfE,\bfX^0,t}[\frac{\partial L(\Psi^{(s_0)};\bfX^t)}{\partial \bfW}]\tilde{\bfx}_p^t-\frac{1}{T}\sum_{t=1}^T\mathbb{E}_{\tilde{\bfX}^0}[\tilde{\bfx}_j^t{}^\top\mathbb{E}_{\bfE,\bfX^0}[\frac{\partial L(\Psi^{(s_0)})}{\partial \bfW}]\tilde{\bfx}_p^t]\Big\|\\
        =& \Big\|\tilde{\bfx}_j^t{}^\top\frac{1}{T}\sum_{t=1}^T\mathbb{E}_{\bfE,\bfX^0}[\frac{\partial L(\Psi^{(s_0)})}{\partial \bfW}]\tilde{\bfx}_p^t-\frac{1}{T}\sum_{t=1}^T\mathbb{E}_{\tilde{\bfX}^0}[\bfmu_u^\top\mathbb{E}_{\bfE,\bfX^0}[\frac{\partial L(\Psi^{(s_0)})}{\partial \bfW}]\bfmu_u]\Big\|\\
        =& \Big\| \bfa_1\frac{1}{T}\sum_{t=1}^T\mathbb{E}_{\bfE,\bfX^0}[\frac{\partial L(\Psi^{(s_0)})}{\partial \bfW}]\bfa_2+\frac{1}{T}\sum_{t=1}^T\bfa_1^\top \mathbb{E}_{\bfE,\bfX^0}[\frac{\partial L(\Psi^{(s_0)})}{\partial \bfW}]\bfmu_u\\
        &+\frac{1}{T}\sum_{t=1}^T\bfmu_u^\top \mathbb{E}_{\bfE,\bfX^0}[\frac{\partial L(\Psi^{(s_0)})}{\partial \bfW}]\bfa_2\Big\|,
    \end{aligned}
\end{equation}
where $\bfa_1=\tilde{\bfx}_j^t-\sqrt{\bar{\alpha}_t}\bfmu_u$, $\bfa_2=\tilde{\bfx}_p^t-\sqrt{\bar{\alpha}_t}\bfmu_u$. We can obtain $\bfa_1,\bfa_2\sim\mathcal{N}(0, \bar{\alpha}_t\rho^2+1-\bar{\alpha}_t)$. We then have the following discussion on $\bfx_i^t$ in the gradient (\ref{grad_ini}). $\bfa_1^\top (\bfx_i^t-\sum_{r=1}^P\text{softmax}_p(\frac{{\bfx_r^t}{}^\top\bfW^{(0)}\bfx_p^t}{d})\bfx_r^t){\bfx_p^t}{}^\top\bfa_2$ is in the order of $d\log d$. $\bfa_1^\top (\bfx_i^t-\sum_{r=1}^P\text{softmax}_p(\frac{{\bfx_r^t}{}^\top\bfW^{(0)}\bfx_p^t}{d})\bfx_r^t){\bfx_p^t}{}^\top\bfmu_u$ is in the order of $d^\frac{3}{2}\log d$. $\bfmu_u^\top (\bfx_i^t-\sum_{r=1}^P\text{softmax}_p(\frac{{\bfx_r^t}{}^\top\bfW^{(0)}\bfx_p^t}{d})\bfx_r^t){\bfx_p^t}{}^\top\bfa_2$ is in the order of $d^\frac{3}{2}\log d$. $(\bff(\bfW;\sqrt{\bar{\alpha}_t}\bfX^0+\sqrt{1-\bar{\alpha}_t}\bfE,t)_p-\boldsymbol{\epsilon}_p)^\top\bfx_i^t$ is in the order of $-d$ if $d$ is large enough, and the order of positive term is no more than $\sqrt{d}\log d$. Therefore, 
\begin{equation}
    \Big\|\frac{1}{T}\sum_{t=1}^T\bfa_1\mathbb{E}_{\bfE,\bfX^0}[\frac{\partial L(\Psi^{(s_0)})}{\partial \bfW}]\bfa_2\Big\|\lesssim \frac{\log d}{d^\frac{3}{2}}\cdot \frac{1}{T}\sum_{t=1}^T\mathbb{E}_{\tilde{\bfX}^0}[\tilde{\bfx}_j^t{}^\top\mathbb{E}_{\bfE,\bfX^0,t}[\frac{\partial L(\Psi^{(s_0)})}{\partial \bfW}]\tilde{\bfx}_p^t],
\end{equation}
\begin{equation}
    \Big\|\frac{1}{T}\sum_{t=1}^T\bfa_1\mathbb{E}_{\bfE,\bfX^0,t}[\frac{\partial L(\Psi^{(s_0)})}{\partial \bfW}]\bfmu_U\Big\|\lesssim \frac{\log d}{d}\cdot \frac{1}{T}\sum_{t=1}^T\mathbb{E}_{\tilde{\bfX}^0}[\tilde{\bfx}_j^t{}^\top\mathbb{E}_{\bfE,\bfX^0,t}[\frac{\partial L(\Psi^{(s_0)})}{\partial \bfW}]\tilde{\bfx}_p^t],
\end{equation}
\begin{equation}
    \Big\|\frac{1}{T}\sum_{t=1}^T\bfmu_u\mathbb{E}_{\bfE,\bfX^0,t}[\frac{\partial L(\Psi^{(s_0)})}{\partial \bfW}]\bfa_2\Big\|\lesssim \frac{\log d}{d}\cdot \frac{1}{T}\sum_{t=1}^T\mathbb{E}_{\tilde{\bfX}^0}[\tilde{\bfx}_j^t{}^\top\mathbb{E}_{\bfE,\bfX^0,t}[\frac{\partial L(\Psi^{(s_0)})}{\partial \bfW}]\tilde{\bfx}_p^t],
\end{equation}
for any $Y_j=Y_p=u$. Hence, we have that for any $u\in [M]$ and $Y_j=Y_p=u$,
\begin{equation}
    \begin{aligned}
        &\Big\|\frac{1}{T}\sum_{t=1}^T\tilde{\bfx}_j^t{}^\top\mathbb{E}_{\bfE,\bfX^0,t}[\frac{\partial L(\Psi^{(s_0)})}{\partial \bfW}]\tilde{\bfx}_p^t-\frac{1}{T}\sum_{t=1}^T\mathbb{E}_{\tilde{\bfX}^0}[\tilde{\bfx}_j^t{}^\top\mathbb{E}_{\bfE,\bfX^0,t}[\frac{\partial L(\Psi^{(s_0)})}{\partial \bfW}]\tilde{\bfx}_p^t]\Big\|\\
        \lesssim & \frac{\log d}{d}\cdot \frac{1}{T}\sum_{t=1}^T\mathbb{E}_{\tilde{\bfX}^0}[\tilde{\bfx}_j^t{}^\top\mathbb{E}_{\bfE,\bfX^0,t}[\frac{\partial L(\Psi^{(s_0)})}{\partial \bfW}]\tilde{\bfx}_p^t].
    \end{aligned}
\end{equation}
We can also derive 
\begin{equation}
    \begin{aligned}
        \Big\|\tilde{\bfx}_j^t{}^\top\bfW^{(0)}\tilde{\bfx}_p^t-\mathbb{E}_{\tilde{\bfX}^0}[\tilde{\bfx}_j^t{}^\top\bfW^{(0)}\tilde{\bfx}_p^t]\Big\|
        \lesssim  \frac{\log d}{d}\cdot \mathbb{E}_{\tilde{\bfX}^0}[\tilde{\bfx}_j^t{}^\top \bfW^{(0)}\tilde{\bfx}_p^t],
    \end{aligned}
\end{equation}
by replacing $\mathbb{E}_{\bfE,\bfX^0}[\frac{\partial L(\Psi^{(s_0)})}{\partial \bfW}]$ with an arbitrarily initialized $\bfW^{(0)}$. Thus, we have
We can also derive 
\begin{equation}
    \begin{aligned}
        \Big\|\tilde{\bfx}_j^t{}^\top\bfW^{(s_0+1)}\tilde{\bfx}_p^t-\mathbb{E}_{\tilde{\bfX}^0}[\tilde{\bfx}_j^t{}^\top\bfW^{(s_0+1)}\tilde{\bfx}_p^t]\Big\|
        \lesssim  \frac{\log d}{d}\cdot \mathbb{E}_{\tilde{\bfX}^0}[\tilde{\bfx}_j^t{}^\top \bfW^{(s_0+1)}\tilde{\bfx}_p^t].
    \end{aligned}
\end{equation}
This ensures that for any $u\in [M]$ and $Y_i=Y_p=Y_j=u$, 
\begin{equation}
    \begin{aligned}
        &\|\text{softmax}_p(\frac{\bfx_i^t{}^\top\bfW^{(s_0+1)}\bfx_p^t}{d})-\mathbb{E}[\text{softmax}_p(\frac{\bfx_i^t{}^\top\bfW^{(s_0+1)}\bfx_p^t}{d})]\|\\\lesssim  & \mathbb{E}[\text{softmax}_p(\frac{\bfx_i^t{}^\top\bfW^{(s_0+1)}\bfx_p^t}{d})]\cdot \frac{\log d\log \epsilon^{-1}(\min_{u\in [M]}\{\min_{u\in [M]}(\nu_u^{\tilde{\boldsymbol{\pi}}}(K))\}^{-1}-1)^2}{d}.
    \end{aligned}
\end{equation}
Hence,
\begin{equation}
    \Big|\text{softmax}_p(\frac{\bfx_i^t{}^\top\bfW^{(s_0+1)}\bfx_p^t}{d})-\text{softmax}_p(\frac{\bfx_j^t{}^\top\bfW^{(s_0+1)}\bfx_p^t}{d})\Big|\lesssim \frac{\log d\log \epsilon^{-1}((\min_{u\in [M]}\nu_u^{\tilde{\boldsymbol{\pi}}}(K))^{-1}-1)^2}{d \sum_{q=1}^P\mathbbm{1}[Y_q=u]}.\label{zeta_uniform}
\end{equation}
Therefore, by summing up over (\ref{it0_i!=p_new}), we can obtain that for $\bfx_j^t$ and $\bfx_{j'}^t$ with $\bfmu_u$ as the mean, 
\begin{equation}
    \begin{aligned}
        & \frac{1}{T}\sum_{t=1}^T(-\tilde{\bfx}_j^t{}^\top) \mathbb{E}_{\bfE,\bfX^0}\Big[\frac{\partial L(\Psi^{(s_0)})}{\partial \bfW}\Big]\tilde{\bfx}_{j'}^t\\
        =&\frac{1}{T}\sum_{t=1}^T(-\tilde{\bfx}_j^t{}^\top)\mathbb{E}_{\bfE,\bfX^0}\big[\sum_{p=1}^P(\bff(\bfW^{(s_0)};\sqrt{\bar{\alpha}_t}\bfX^0+\sqrt{1-\bar{\alpha}_t}\bfE,t)_p-\boldsymbol{\epsilon}_p)^\top (-\frac{v_t^{(s_0)}}{d^2P}) \sum_{i=1}^P\bfx_i^{t}\\
        &\cdot\text{softmax}_p(\frac{{\bfx_i^t} {}^\top\bfW^{(s_0)}\bfx_p^t}{d})\cdot (\bfx_i^t-\sum_{r=1}^P\text{softmax}_p(\frac{{\bfx_r^t}{}^\top\bfW^{(s_0)}\bfx_p^t}{d})\bfx_r^t){\bfx_p^t}{}^\top\big]\tilde{\bfx}_{j'}^t\\
        \gtrsim & \frac{1}{T}\sum_{t=1}^T\mathbb{E}_{\bfE,\bfX^0}[(v_t^{(s_0)})^2\bar{\alpha}_t^3 d (1-\sum_{l=1}^P \zeta_{l,p,t}^u(s_0))^2\sum_{i=1}^P\zeta_{i,p,t}^u(s_0)]\cdot \nu_u^{\tilde{\boldsymbol{\pi}}}(K),
    \end{aligned}
\end{equation}
where the last step holds with a high probability if 
    \begin{equation}
    P\cdot \nu_u^{\tilde{\boldsymbol{\pi}}}(K)\gtrsim \log dK.
\end{equation}

\noindent Similarly, we hvae that for $\bfmu_u\neq \bfmu_{u'}$ and $\tilde{\bfx}_j$, $\tilde{\bfx}_{j'}$ with the mean of $\bfmu_u$ and $\tilde{\bfx}_{k}$ with the mean of $\bfmu_{u'}$, we have that 
\begin{equation}
    \begin{aligned}
        & (-\tilde{\bfx}_{k}^\top)\mathbb{E}_{\bfE,\bfX^0}\Big[\frac{\partial L(\Psi^{(s_0)})}{\partial \bfW}\Big]\tilde{\bfx}_{j'}^t\\
        =&(-\tilde{\bfx}_{k}^\top)\mathbb{E}_{\bfE,\bfX^0}\big[\sum_{p=1}^P(\bff(\bfW^{(s_0)};\sqrt{\bar{\alpha}_t}\bfX^0+\sqrt{1-\bar{\alpha}_t}\bfE,t)_p-\boldsymbol{\epsilon}_p)^\top (-\frac{v_t^{(s_0)}}{d^2P}) \sum_{i=1}^P\bfx_i^{t}\\
        &\cdot \text{softmax}_p(\frac{{\bfx_i^t} {}^\top\bfW^{(s_0)}\bfx_p^t}{d})\cdot (\bfx_i^t-\sum_{r=1}^P\text{softmax}_p(\frac{{\bfx_r^t}{}^\top\bfW^{(s_0)}\bfx_p^t}{d})\bfx_r^t){\bfx_p^t}{}^\top\big]\tilde{\bfx}_{j'}^t\\
        \leq &\frac{\log d}{\sqrt{d}}\cdot (-\tilde{\bfx}_{j}^\top)\mathbb{E}_{\bfE,\bfX^0}\big[\sum_{p=1}^P(\bff(\bfW^{(s_0)};\sqrt{\bar{\alpha}_t}\bfX^0+\sqrt{1-\bar{\alpha}_t}\bfE,t)_p-\boldsymbol{\epsilon}_p)^\top (-\frac{v_t^{(s_0)}}{d^2P}) \sum_{i=1}^P\bfx_i^{t}\\
        &\cdot \text{softmax}_p(\frac{{\bfx_i^t} {}^\top\bfW^{(s_0)}\bfx_p^t}{d})\cdot (\bfx_i^t-\sum_{r=1}^P\text{softmax}_p(\frac{{\bfx_r^t}{}^\top\bfW^{(s_0)}\bfx_p^t}{d})\bfx_r^t){\bfx_p^t}{}^\top\big]\tilde{\bfx}_{j'}^t\\
        =& \frac{\log d}{\sqrt{d}}\cdot (-\tilde{\bfx}_{j}^\top)\mathbb{E}_{\bfE,\bfX^0}\Big[\frac{\partial L(\Psi^{(s_0)})}{\partial \bfW}\Big]\tilde{\bfx}_{j'}^t
    \end{aligned}
\end{equation}
\begin{equation}
    \begin{aligned}
    &\frac{1}{T}\sum_{t=1}^T(-\tilde{\bfx}_{k}^\top)\mathbb{E}_{\bfE,\bfX^0}\Big[\frac{\partial L(\Psi^{(s_0)})}{\partial \bfW}\Big]\tilde{\bfx}_{j'}^t\\
        =&\frac{1}{T}\sum_{t=1}^T(-\tilde{\bfx}_{k}^\top)\mathbb{E}_{\bfE,\bfX^0}\big[\sum_{p=1}^P(\bff(\bfW^{(s_0)};\sqrt{\bar{\alpha}_t}\bfX^0+\sqrt{1-\bar{\alpha}_t}\bfE,t)_p-\boldsymbol{\epsilon}_p)^\top (-\frac{v_t^{(s_0)}}{d^2P}) \sum_{i=1}^P\bfx_i^{t}\\
        &\cdot \text{softmax}_p(\frac{{\bfx_i^t} {}^\top\bfW^{(s_0)}\bfx_p^t}{d})\cdot (\bfx_i^t-\sum_{r=1}^P\text{softmax}_p(\frac{{\bfx_r^t}{}^\top\bfW^{(s_0)}\bfx_p^t}{d})\bfx_r^t){\bfx_p^t}{}^\top\big]\tilde{\bfx}_{j'}\\
        \leq & \frac{1}{T}\sum_{t=1}^T(-\tilde{\bfx}_{j}^\top)\frac{\log d}{\sqrt{d}}\cdot \mathbb{E}_{\bfE,\bfX^0}\big[\sum_{p=1}^P(\bff(\bfW^{(s_0)};\sqrt{\bar{\alpha}_t}\bfX^0+\sqrt{1-\bar{\alpha}_t}\bfE,t)_p-\boldsymbol{\epsilon}_p)^\top (-\frac{v_t^{(s_0)}}{d^2P}) \sum_{i=1}^P\bfx_i^{t}\\
        &\cdot \text{softmax}_p(\frac{{\bfx_i^t} {}^\top\bfW^{(s_0)}\bfx_p^t}{d})\cdot (\bfx_i^t-\sum_{r=1}^P\text{softmax}_p(\frac{{\bfx_r^t}{}^\top\bfW^{(s_0)}\bfx_p^t}{d})\bfx_r^t){\bfx_p^t}{}^\top\big]\tilde{\bfx}_{j'}\\
        =& \frac{\log d}{\sqrt{d}}\cdot \frac{1}{T}\sum_{t=1}^T(-\tilde{\bfx}_{j}^\top)\mathbb{E}_{\bfE,\bfX^0}\Big[\frac{\partial L(\Psi^{(s_0)})}{\partial \bfW}\Big]\tilde{\bfx}_{j'}^t.
    \end{aligned}
\end{equation}
\end{proof}

\subsection{Proof of Lemma \ref{lemma: stage 1}}
\begin{proof}
\noindent For any $\tilde{\bfx}_j$ and $\tilde{\bfx}_{j'}$ with $\bfmu_u$ as the mean, where $\tilde{\bfX}$ is generated by Definition, denote  
\begin{equation}
\begin{aligned}
    Z^u_{j,j'}(s)=&\frac{\tilde{\bfx}_j^\top \bfW^{(s)}\tilde{\bfx}_{j'}}{d}\\
    =& \frac{\tilde{\bfx}_j^\top \bfW^{(0)}\tilde{\bfx}_{j'}}{d}-\eta\frac{1}{T}\sum_{t=1}^T\sum_{s_0=0}^{s-1}\sum_{p=1}^P\mathbb{E}_{\bfE,\bfX^0}\big[-\tilde{\bfx}_j^\top(\bff(\bfW^{(s_0)};\sqrt{\bar{\alpha}_t}\bfX^0+\sqrt{1-\bar{\alpha}_t}\bfE,t)_p-\boldsymbol{\epsilon}_p)^\top \\
    &\cdot(-\frac{v_t^{(s)}}{d^3P}) \sum_{i=1}^P\bfx_i^{t}\text{softmax}_p(\frac{{\bfx_i^t} {}^\top\bfW^{(s_0)}\bfx_p^t}{d})\cdot (\bfx_i^t-\sum_{r=1}^P\text{softmax}_p(\frac{{\bfx_r^t}{}^\top\bfW^{(s_0)}\bfx_p^t}{d})\bfx_r^t){\bfx_p^t}{}^\top\tilde{\bfx}_{j'}\big]
    \end{aligned}
\end{equation}

Denote $\gamma(u)=\eta \cdot\frac{1}{T}\sum_{t=1}^T(v_t^{(s)})^2\bar{\alpha}_t^3\cdot \nu_u^{\tilde{\boldsymbol{\pi}}}(K)$. Note that
\begin{equation}
    \sum_{j=1}^P\zeta_{j,j',t}^u(s)\gtrsim \frac{\sum_{p=1}^P\mathbbm{1}[Y_p=u]\cdot e^{Z_{j,j'}^u(s)}}{\sum_{p=1}^P\mathbbm{1}[Y_p=u]\cdot e^{Z_{j,j'}^u(s)}+\sum_{p=1}^P\mathbbm{1}[Y_p\neq u]}.
\end{equation}
When $Z_{j,j'}^u(s)\leq \Theta(\log \frac{1-\nu_u^{\tilde{\boldsymbol{\pi}}}(K)}{\nu_u^{\tilde{\boldsymbol{\pi}}}(K)})$, we have that by Lemma \ref{lemma: chernoff}, with a high probability,
\begin{equation}
    1-\sum_{j=1}^P\zeta_{j,j',t}^u(s)\gtrsim \frac{\sum_{p=1}^P\mathbbm{1}[Y_p\neq u]}{\sum_{p=1}^P\mathbbm{1}[Y_p= u]\cdot e^{Z_{j,j'}^u(s)}+\sum_{p=1}^P\mathbbm{1}[Y_p\neq u]}\gtrsim  (1-\nu_u^{\tilde{\boldsymbol{\pi}}}(K))e^{-Z_{j,j'}^u(s)}.
\end{equation}
Otherwise, 
\begin{equation}
    1-\sum_{j=1}^P\zeta_{j,j',t}^u(s)\gtrsim \frac{\sum_{p=1}^P\mathbbm{1}[Y_p\neq u]}{\sum_{p=1}^P\mathbbm{1}[Y_p= u]\cdot e^{Z^u_{j,j'}(s)}+\sum_{p=1}^P\mathbbm{1}[Y_p\neq u]}\gtrsim \frac{1-\nu_u^{\tilde{\boldsymbol{\pi}}}(K)}{\nu_u^{\tilde{\boldsymbol{\pi}}}(K)}\cdot e^{-Z^u_{j,j'}(s)}.
\end{equation}
Note that by (\ref{same mean W}) in Lemma \ref{lemma: W training},
\begin{equation}
    Z_{j,j'}^u(s)\gtrsim \eta\frac{1}{T}\sum_{t=1}^T\sum_{b=0}^{s-1}\mathbb{E}_{\bfE,\bfX^0}[\bar{\alpha}_t^3 (v_t^{(s)})^2 (1-\sum_{l=1}^P \zeta_{l,p,t}^u(b))^2\sum_{i=1}^P\zeta_{i,p,t}^u(b)]\cdot \nu_u^{\tilde{\boldsymbol{\pi}}}(K)
\end{equation}
We first prove that when $Z^u_{j,j'}(s)\leq \Theta(\log \frac{1-\nu_u^{\tilde{\boldsymbol{\pi}}}(K)}{\nu_u^{\tilde{\boldsymbol{\pi}}}(K)})$, we have $Z^u_{j,j'}(s)\geq \frac{1}{2}\log (1+2\gamma(u) (1-\nu_u^{\tilde{\boldsymbol{\pi}}}(K))^2\nu_u^{\tilde{\boldsymbol{\pi}}}(K)\cdot s)$ by induction. The conclusion holds when $s=0$. Suppose that this conclusion holds when $s\leq s_0$. Then, 
\begin{equation}
    \begin{aligned}
        &Z_{j,j'}^u(s+1)\\
        \geq &\eta\frac{1}{T}\sum_{t=1}^T\sum_{b=0}^{s-1}\mathbb{E}_{\bfE,\bfX^0}[(v_t^{(s)})^2\bar{\alpha}_t^3  (1-\sum_{l=1}^P \zeta_{l,p,t}^u(b))^2\sum_{i=1}^P\zeta_{i,p,t}^u(b)]\cdot \nu_u^{\tilde{\boldsymbol{\pi}}}(K)\\
        &+\frac{1}{T}\sum_{t=1}^T\mathbb{E}_{\bfE,\bfX^0}[(v_t^{(s)})^2\bar{\alpha}_t^3  (1-\sum_{l=1}^P \zeta_{l,p,t}^u(s))^2\sum_{i=1}^P\zeta_{i,p,t}^u(s)]\cdot \nu_u^{\tilde{\boldsymbol{\pi}}}(K)\\
        \geq &\frac{1}{2}\log (1+2\gamma(u) (1-\nu_u^{\tilde{\boldsymbol{\pi}}}(K))^2 \nu_u^{\tilde{\boldsymbol{\pi}}}(K) s)+\frac{1}{T}\sum_{t=1}^T\mathbb{E}_{\bfE,\bfX^0}[(v_t^{(s)})^2\bar{\alpha}_t^3  (1-\sum_{l=1}^P \zeta_{l,p,t}^u(s))^2\sum_{i=1}^P\zeta_{i,p,t}^u(s)]\cdot \nu_u^{\tilde{\boldsymbol{\pi}}}(K)\\
        \geq & \frac{1}{2}\log (1+2\cdot \gamma(u) (1-\nu_u^{\tilde{\boldsymbol{\pi}}}(K))^2\nu_u^{\tilde{\boldsymbol{\pi}}}(K)(s+1)),
    \end{aligned}
\end{equation}
where the last step is by 
\begin{equation}
    \begin{aligned}
        &\frac{1}{T}\sum_{t=1}^T\mathbb{E}_{\bfE,\bfX^0}[(v_t^{(s)})^2\bar{\alpha}_t^3  (1-\sum_{l=1}^P \zeta_{l,p,t}^u(s_0))^2\sum_{i=1}^P\zeta_{i,p,t}^u(s_0)]\cdot  \nu_u^{\tilde{\boldsymbol{\pi}}}(K)\\
        \geq & \frac{1}{T}\sum_{t=1}^T\mathbb{E}_{\bfE,\bfX^0}[(v_t^{(s)})^2\bar{\alpha}_t^3  \cdot e^{-2Z_{j,j'}(s)}]\cdot (1-\nu_u^{\tilde{\boldsymbol{\pi}}}(K))^2\cdot (\nu_u^{\tilde{\boldsymbol{\pi}}}(K))^2\\
        \geq & \frac{1}{T}\sum_{t=1}^T\mathbb{E}_{\bfE,\bfX^0}[(v_t^{(s)})^2\bar{\alpha}_t^3  \cdot \frac{1}{1+2\gamma(u) (1-\nu_u^{\tilde{\boldsymbol{\pi}}}(K))^2\nu_u^{\tilde{\boldsymbol{\pi}}}(K)s}]\cdot (\nu_u^{\tilde{\boldsymbol{\pi}}}(K))^2(1-\nu_u^{\tilde{\boldsymbol{\pi}}}(K))^2\\
        \geq & \frac{1}{2}\log (1+\frac{2\gamma(u) \nu_u^{\tilde{\boldsymbol{\pi}}}(K)(1-\nu_u^{\tilde{\boldsymbol{\pi}}}(K))^2}{1+2\gamma(u) \nu_u^{\tilde{\boldsymbol{\pi}}}(K)(1-\nu_u^{\tilde{\boldsymbol{\pi}}}(K))^2\cdot s}).
    \end{aligned}
\end{equation}
Therefore, the conclusion holds when $s=s_0+1$. When $s\geq I_0(u):=\frac{1}{\gamma(u) (\nu_u^{\tilde{\boldsymbol{\pi}}}(K))^3}$, we have $Z_{j,j'}^u(s)\geq \log \frac{1-\nu_u^{\tilde{\boldsymbol{\pi}}}(K)}{\nu_u^{\tilde{\boldsymbol{\pi}}}(K)}$. 
Then, we prove that $Z_{j,j'}^u(s)\geq \frac{1}{2}\log (1+2\gamma(u) \cdot (\frac{1-\nu_u^{\tilde{\boldsymbol{\pi}}}(K)}{\nu_u^{\tilde{\boldsymbol{\pi}}}(K)})^2(s-I_0(u))+2\gamma(u) (1-\nu_u^{\tilde{\boldsymbol{\pi}}}(K))^2\nu_u^{\tilde{\boldsymbol{\pi}}}(K) I_0(u))$ by induction for $Z_{j,j'}^u(s)\geq \Theta(\log \frac{1-\nu_u^{\tilde{\boldsymbol{\pi}}}(K)}{\nu_u^{\tilde{\boldsymbol{\pi}}}(K)})$. Suppose that this conclusion holds when $s\leq s_0$. Then, 
\begin{equation}
    \begin{aligned}
        Z_{j,j'}^u(s+1)\geq &\eta\frac{1}{T}\sum_{t=1}^T\sum_{b=0}^{s-1}\mathbb{E}_{\bfE,\bfX^0}[(v_t^{(s)})^2\bar{\alpha}_t^3  (1-\sum_{l=1}^P \zeta_{l,p,t}^u(b))^2\sum_{i=1}^P\zeta_{i,p,t}^u(b)]\cdot \nu_u^{\tilde{\boldsymbol{\pi}}}(K)\\
        &+\frac{1}{T}\sum_{t=1}^T\mathbb{E}_{\bfE,\bfX^0}[(v_t^{(s)})^2\bar{\alpha}_t^3  (1-\sum_{l=1}^P \zeta_{l,p,t}^u(s))^2\sum_{i=1}^P\zeta_{i,p,t}^u(s)]\cdot \nu_u^{\tilde{\boldsymbol{\pi}}}(K)\\
        \geq & \frac{1}{2}\log (1+2\gamma(u)\cdot (\frac{1-\nu_u^{\tilde{\boldsymbol{\pi}}}(K)}{\nu_u^{\tilde{\boldsymbol{\pi}}}(K)})^2 (s-I_0(u))+2\gamma(u) (1-\nu_u^{\tilde{\boldsymbol{\pi}}}(K))^2\nu_u^{\tilde{\boldsymbol{\pi}}}(K)I_0(u))\\
        &+\frac{1}{T}\sum_{t=1}^T\mathbb{E}_{\bfE,\bfX^0}[(v_t^{(s)})^2\bar{\alpha}_t^3  (1-\sum_{l=1}^P \zeta_{l,p,t}^u(s))^2\sum_{i=1}^P\zeta_{i,p,t}^u(s)] \nu_u^{\tilde{\boldsymbol{\pi}}}(K)\\
        \geq & \frac{1}{2}\log (1+2\cdot (\frac{1-\nu_u^{\tilde{\boldsymbol{\pi}}}(K)}{\nu_u^{\tilde{\boldsymbol{\pi}}}(K)})^2\gamma(u) (s+1-I_0(u))+2\gamma(u) (1-\nu_u^{\tilde{\boldsymbol{\pi}}}(K))^2\nu_u^{\tilde{\boldsymbol{\pi}}}(K)I_0(u)),
    \end{aligned}
\end{equation}
where the last step is by 
\begin{equation}
    \begin{aligned}
        &\frac{1}{T}\sum_{t=1}^T\mathbb{E}_{\bfE,\bfX^0}[(v_t^{(s)})^2\bar{\alpha}_t^3  (1-\sum_{l=1}^P \zeta_{l,p,t}^u(s))^2\sum_{i=1}^P\zeta_{i,p,t}^u(s)]\cdot \nu_u^{\tilde{\boldsymbol{\pi}}}(K)\\
        \geq &\frac{1}{T}\sum_{t=1}^T\mathbb{E}_{\bfE,\bfX^0}[(v_t^{(s)})^2\bar{\alpha}_t^3  e^{-2Z^u_{j,j}(s)}]\cdot \nu_u^{\tilde{\boldsymbol{\pi}}}(K) \cdot (\frac{1-\nu_u^{\tilde{\boldsymbol{\pi}}}(K)}{\nu_u^{\tilde{\boldsymbol{\pi}}}(K)})^2 \\
        \geq &  \frac{\gamma(u)\cdot \frac{(1-\nu_u^{\tilde{\boldsymbol{\pi}}}(K))^2}{(\nu_u^{\tilde{\boldsymbol{\pi}}}(K))^2}}{1+2\gamma(u) \cdot \frac{(1-\nu_u^{\tilde{\boldsymbol{\pi}}}(K))^2}{(\nu_u^{\tilde{\boldsymbol{\pi}}}(K))^2}(s-I_0(u))+2\gamma(u) (1-\nu_u^{\tilde{\boldsymbol{\pi}}}(K))^2\nu_u^{\tilde{\boldsymbol{\pi}}}(K)I_0(u)}\\
        \geq & \frac{1}{2}\log (1+\frac{2\gamma(u)\cdot\frac{(1-\nu_u^{\tilde{\boldsymbol{\pi}}}(K))^2}{(\nu_u^{\tilde{\boldsymbol{\pi}}}(K))^2}}{1+2\gamma(u) \cdot \frac{(1-\nu_u^{\tilde{\boldsymbol{\pi}}}(K))^2}{(\nu_u^{\tilde{\boldsymbol{\pi}}}(K))^2}(s-I_0(u))+2\gamma(u) (1-\nu_u^{\tilde{\boldsymbol{\pi}}}(K))^2\nu_u^{\tilde{\boldsymbol{\pi}}}(K)I_0(u)}).
    \end{aligned}
\end{equation}
Then,
\begin{equation}
\begin{aligned}
    Z_{j,j'}(\gamma(u)^{-1}\epsilon^{-1}+I_0(u))\geq &\frac{1}{2}\log (\gamma(u)(\frac{(1-\nu_u^{\tilde{\boldsymbol{\pi}}}(K))}{\nu_u^{\tilde{\boldsymbol{\pi}}}(K)})^2\gamma(u)^{-1}\epsilon^{-1}+(\frac{1-\nu_u^{\tilde{\boldsymbol{\pi}}}(K)}{\nu_u^{\tilde{\boldsymbol{\pi}}}(K)})^2)\\
    \gtrsim &\frac{1}{2}\log ((\frac{(1-\nu_u^{\tilde{\boldsymbol{\pi}}}(K))}{\nu_u^{\tilde{\boldsymbol{\pi}}}(K)})^2\epsilon^{-1}),\label{Z lower bound}
    \end{aligned}
\end{equation}
where the second step holds if $\epsilon\leq 1$. Let $u^*=\arg\min_{u\in [M]}\{\nu_u^{\tilde{\boldsymbol{\pi}}}(K)\}$. We have
\begin{equation}
\begin{aligned}
    &Z_{j,j'}(\gamma(u^*)^{-1}\epsilon^{-1}+I_0(u^*))\\
    \gtrsim &\frac{1}{2}\log ((\frac{1-\nu_u^{\tilde{\boldsymbol{\pi}}}(K)}{\nu_u^{\tilde{\boldsymbol{\pi}}}(K)})^2\epsilon^{-1}+\gamma(u)(\frac{1-\nu_u^{\tilde{\boldsymbol{\pi}}}(K)}{\nu_u^{\tilde{\boldsymbol{\pi}}}(K)})^2(I_0(u^*)-I_0(u)+\epsilon^{-1}(\gamma(u^*)^{-1}-\gamma(u)^{-1}))).
    \end{aligned}
\end{equation}
Hence, as long as
\begin{equation}
    \gamma(u)(\frac{1-\nu_u^{\tilde{\boldsymbol{\pi}}}(K)}{\nu_u^{\tilde{\boldsymbol{\pi}}}(K)})^2\epsilon
    ^{-1}(\gamma(u^*)^{-1}-\gamma(u)^{-1})\lesssim \text{poly}((\frac{(1-\nu_u^{\tilde{\boldsymbol{\pi}}}(K))}{\nu_u^{\tilde{\boldsymbol{\pi}}}(K)})^2\epsilon^{-1}), 
\end{equation}
and
\begin{equation}
    \gamma(u)(\frac{1-\nu_u^{\tilde{\boldsymbol{\pi}}}(K)}{\nu_u^{\tilde{\boldsymbol{\pi}}}(K)})^2(I_0(u^*)-I_0(u))\lesssim \text{poly}((\frac{(1-\nu_u^{\tilde{\boldsymbol{\pi}}}(K))}{\nu_u^{\tilde{\boldsymbol{\pi}}}(K)})^2\epsilon^{-1}), 
\end{equation}
which hold if 

\begin{equation}
    \epsilon\leq \delta^{\Theta(1)},
\end{equation}
we have
\begin{equation}
    Z_{j,j'}(\epsilon^{-1}\gamma(u)^{-1}+I_0(u))=\Theta( Z_{j,j'}(\epsilon^{-1}\gamma(u^*)^{-1}+I_0(u^*))).
\end{equation}
We then have
\begin{equation}
\begin{aligned}
    &\Big\|\frac{\partial L(\Psi^{(s)};\bfX^t,t)}{\partial \bfW^{(s)}}|_{s=I}\Big\|\\
    = &\max_{\bfx\neq 0}\frac{\|\frac{\partial L(\Psi^{(s)};\bfX^t,t)}{\partial \bfW^{(s)}}|_{s=I}\bfx\|}{\|\bfx\|}\\
    \leq & \max_{u\in[M]}\Big\|\bfmu_u\frac{\partial L(\Psi^{(s)};\bfX^t,t)}{\partial \bfW^{(s)}}\Big|_{s=I}\bfmu_u\Big\|\cdot \frac{1}{d}\\
    \lesssim & \gamma(u) \cdot  (\frac{1-\nu_u^{\tilde{\boldsymbol{\pi}}}(K)}{\nu_u^{\tilde{\boldsymbol{\pi}}}(K)})^2\cdot e^{-2Z^u_{j,j}(s)}\\
    \leq & \gamma(u)\epsilon^2,
    \end{aligned}
\end{equation}
as long as 
\begin{equation}
\begin{aligned}
    s\gtrsim &\gamma(u)^{-1}\epsilon^{-1}+I_0(u)\\
    \gtrsim &(\epsilon^{-1}+\nu_u^{\tilde{\boldsymbol{\pi}}}(K)^{-3})\eta^{-1}\nu_u^{\tilde{\boldsymbol{\pi}}}(K)^{-1}\mathbb{E}_t [\bar{\alpha}_t/(1-\bar{\alpha}_t)]^{-3}(1-\alpha_1)^{-3}\\
    \gtrsim & (\epsilon^{-1}+\nu_u^{\tilde{\boldsymbol{\pi}}}(K)^{-3})\eta^{-1}\nu_u^{\tilde{\boldsymbol{\pi}}}(K)^{-1}\textrm{SNR}^{-3}(1-\alpha_1)^{-3},
    \end{aligned}
\end{equation}
since $\bar{\alpha}_t^3\geq (1-\alpha_1)^3\bar{\alpha}_t^3/(1-\bar{\alpha}_t)^3$. By Jensen's inequality, we have 
\begin{equation}
    \frac{1}{T}\sum_{t=1}^T\frac{\bar{\alpha}_t^3}{(1-\bar{\alpha}_t)^3}\geq \Big(\frac{1}{T}\sum_{t=1}^T\frac{\bar{\alpha}_t}{(1-\bar{\alpha}_t)}\Big)^3.
\end{equation}
$\frac{1}{T}\sum_{t=1}^T\frac{\bar{\alpha}_t}{(1-\bar{\alpha}_t)}$ can be approximated by $\textrm{SNR}$, which comes from Hoeffding's inequality (\ref{hoeffding}), i.e., with a high probability,
\begin{equation}
    \Big|\frac{1}{T}\sum_{t=1}^T \frac{\bar{\alpha}_t}{1-\bar{\alpha}_t}-\textrm{SNR}\Big|\leq \frac{\alpha_1}{1-\alpha_1}\sqrt{\frac{\log d}{T}}
\end{equation}
Therefore, when $T\gtrsim \log d$, the required condition for the number of iterations is 
\begin{equation}
    s\gtrsim I_1:= (\epsilon^{-1}+ \min_{u\in [M]}\{\nu_u^{\tilde{\boldsymbol{\pi}}}(K)\}^{-3})\eta^{-1}\min_{u\in[M]}\{\nu_u^{\tilde{\boldsymbol{\pi}}}(K)\}^{-1}\textrm{SNR}^{-3}.
\end{equation}
When $s\gtrsim I_1$, we have
\begin{equation}
    1-\sum_{j=1}^P\zeta_{j,j',t}^u(s)\gtrsim \sqrt{\epsilon}.\label{1-zeta final} 
\end{equation}
\end{proof}

\subsection{Proof of Lemma \ref{lemma: stage 2}}

\begin{proof}
First, for any $u\in [M]$ and $Y_p=u$, by (\ref{zeta_uniform}), we know that for $i$, $i'$, such that $Y_i=Y_{i'}=u$, we have $\zeta^u_{i,p,t}(s)=(1+\log \epsilon^{-1}(\min_{u\in [M]}\mathbb{E}[\pi_u])^{-1}-1)^2/d)\cdot \zeta^u_{i',p,t}(s)$ for any $s\leq I$. Then, by Hoeffding's inequality (\ref{hoeffding}),
\begin{equation}
\begin{aligned}
    &\Big\|\sum_{i=1}^P \mathbbm{1}[Y_i=u] \zeta^u_{i,p,t}(s)(\bfx_i^t-\sqrt{\bar{\alpha}_t}\bfmu_u)\Big\|\\
    =&\Big\|\sum_{i=1}^P \mathbbm{1}[Y_i=u] \zeta^u_{i,p,t}(s)(\bfx_i^t-\sqrt{\bar{\alpha}_t}\bfmu_u)-\mathbb{E}[\sum_{i=1}^P\mathbbm{1}[Y_i=u] \zeta^u_{i,p,t}(s)(\bfx_i^t-\sqrt{\bar{\alpha}_t}\bfmu_u)]\Big\|\\
    \leq &\Big\|\sum_{i=1}^P \mathbbm{1}[Y_i=u] \frac{(1+\frac{\log d \log \epsilon^{-1}(\min_{u\in [M]}(\nu_u^{\tilde{\boldsymbol{\pi}}}(K))^{-1}-1)^2}{d})}{\sum_{i=1}^P \mathbbm{1}[Y_i=u]}(\bfx_i^t-\sqrt{\bar{\alpha}_t}\bfmu_u)-\mathbb{E}[\sum_{i=1}^P\mathbbm{1}[Y_i=u] \zeta^u_{i,p,t}(s)(\bfx_i^t-\sqrt{\bar{\alpha}_t}\bfmu_u)]\Big\|\\
    \lesssim & (1+\frac{\log d\log \epsilon^{-1}(\min_{u\in [M]}(\nu_u^{\tilde{\boldsymbol{\pi}}}(K))^{-1}-1)^2}{d})\cdot \sqrt{1-\bar{\alpha}_t+\rho^2\bar{\alpha}_t}\cdot\sqrt{\frac{d\log d}{\sum_{i=1}^P \mathbbm{1}[Y_i=u]}},
    \end{aligned}\label{Yi=u hoeffding}
\end{equation}
with a high probability. 
We also have a similar result for $i$ such that $Y_i\neq u$. 

\noindent We next study the training phase of the model after $I_1$ iterations. Let $\bfmu_u$ be the mean of $\bfx_p^t$. Then, we have
\begin{equation}
    \begin{aligned}
        &\|(\sqrt{\bar{\alpha}}_t\bfmu_u-\sum_{i=1}^P \bfx_i^t\text{softmax}_p(\frac{{\bfx_i^t}^\top\bfW^{(I_1)}\bfx^t_p}{d}))\|^2\\
        \lesssim & ((1-\sum_{i=1}^P\zeta^u_{i,p,t}(I_1))^2\bar{\alpha}_t d+\frac{(\rho^2\bar{\alpha}_t+1-\bar{\alpha}_t)(1+\frac{\log d \log \epsilon^{-1}(\min_{u\in [M]}((\min_{u\in [M]}(\nu_u^{\tilde{\boldsymbol{\pi}}}(K))^{-1}-1)^2)^{-1}-1)^2}{d})^2\log d}{\sum_{i=1}^P\mathbbm{1}[Y_p=u]}d)\\
        \leq & d\epsilon,\label{stage 2 second term}
    \end{aligned}
\end{equation}
where the second term of the first step is by (\ref{Yi=u hoeffding}), and the last step is by (\ref{1-zeta final}) and holds if
\begin{equation}
    \begin{aligned}
        \sum_{i=1}^P\mathbbm{1}[Y_i=u]
        \geq  (\rho^2+1)\epsilon^{-1}\log d.\label{S_lb_1}
    \end{aligned}
\end{equation}
for any $u\in [M]$ as long as $\epsilon\in (0, \delta^{\Theta(1)})$ with $d\gtrsim \log \epsilon^{-1}\nu_{\min}^{\tilde{\boldsymbol{\pi}}}(K)^{-1}$, which is equivalent to
\begin{equation}
    P\gtrsim \min_{u\in[M]}\{\nu_u^{\tilde{\boldsymbol{\pi}}}(K)\}^{-1}(\rho^2+1)\epsilon^{-1}\log d,
\end{equation}
with a high probability, because
\begin{equation}
\begin{aligned}
    &\Pr\Big(\frac{1}{P}\sum_{i=1}^P\mathbbm{1}[Y_i=\arg\min_{u\in [M]}\gamma(u)]\geq \Big(1+\sqrt{\frac{\log dK}{P\min_{u\in[M]}\{\nu_u^{\tilde{\boldsymbol{\pi}}}(K)\}}}\Big)\min_{u\in[M]}\{\nu_u^{\tilde{\boldsymbol{\pi}}}(K)\}\Big)\\
    \leq & e^{-\frac{P\min_{u\in[M]}\{\nu_u^{\tilde{\boldsymbol{\pi}}}(K)\}\cdot \frac{\log dK}{P\min_{u\in[M]}\{\nu_u^{\tilde{\boldsymbol{\pi}}}(K)\}}}{3}}\\
    \leq & (dK)^{-C}
    \end{aligned}
\end{equation}
for some $C>1$. We reduce $\log dK$ to $\log d$ in the final bound of $P$ since $d\geq M\geq K$. 
Denote $\beta_1=\boldsymbol{\epsilon}_p'{}^\top (\sqrt{\bar{\alpha}}_t\bfmu_u-\sum_{i=1}^P \bfx_i^t\text{softmax}_p(\frac{{\bfx_i^t}^\top\bfW^{(I_1)}\bfx^t_p}{d}))$ and $\beta_2=\boldsymbol{\epsilon}_p^\top (\sqrt{\bar{\alpha}}_t\bfmu_u-\sum_{i=1}^P \bfx_i^t\text{softmax}_p(\frac{{\bfx_i^t}^\top\bfW^{(I_1)}\bfx^t_p}{d}))$. Note that by Cauchy-Schwarz inequality,
\begin{equation}
    \mathbb{E}_{\bfE,\bfX^0}[\beta_1], \mathbb{E}_{\bfE,\bfX^0}[\beta_2]\leq \sqrt{d\epsilon}
\end{equation}
Therefore, with (\ref{stage 2 second term}), we have

\begin{equation}
    \begin{aligned}
        &\mathbb{E}_{\bfE,\bfX^0} [\|v_t^{(I_1)}(\bfx_p^t-\sum_{i=1}^P \bfx_i^t\text{softmax}_p(\frac{{\bfx_i^t}^\top\bfW^{(I_1)}\bfx^t_p}{d}))-\boldsymbol{\epsilon}_p\|^2/d]\\
        =& \mathbb{E}_{\bfE,\bfX^0} [\|v_t^{(I_1)}\rho\sqrt{\bar{\alpha}}_t\boldsymbol{\epsilon}_p'+(v_t^{(I_1)}\sqrt{1-\bar{\alpha}_t}-1)\boldsymbol{\epsilon}_p+v_t^{(I_1)}(\sqrt{\bar{\alpha}}_t\bfmu_u-\sum_{i=1}^P \bfx_i^t\text{softmax}_p(\frac{{\bfx_i^t}^\top\bfW^{(I_1)}\bfx^t_p}{d}))\|^2/d]\\
        = & \mathbb{E}_{\bfE,\bfX^0} [\|v_t^{(I_1)}\rho\sqrt{\bar{\alpha}}_t\boldsymbol{\epsilon}_p'+(v_t^{(I_1)}\sqrt{1-\bar{\alpha}_t}-1)\boldsymbol{\epsilon}_p\|/d]^2+(v_t^{(I_1)})^2\mathbb{E}_{\bfE,\bfX^0}[\|\sqrt{\bar{\alpha}}_t\bfmu_u\\
        &-\sum_{i=1}^P \bfx_i^t\text{softmax}_p(\frac{{\bfx_i^t}^\top\bfW^{(I_1)}\bfx^t_p}{d})\|^2/d]+ 2\mathbb{E}_{\bfE,\bfX^0}[\frac{\beta_1 v_t^{(I_1)}{}^2\sqrt{\bar{\alpha}_t}}{d}]+2\mathbb{E}_{\bfE,\bfX^0}[\frac{\beta_2(v_t^{(I_1)}{}^2\sqrt{1-\bar{\alpha}_t}-1)}{d}]\\
        =& ((v_t^{(I_1)}{})^2\rho^2\bar{\alpha}_t+(v_t^{(I_1)}\sqrt{1-\bar{\alpha}_t}-1)^2)+\mathbb{E}_{\bfE,\bfX^0}[(v_t^{(I_1)})^2\|\sqrt{\bar{\alpha}}_t\bfmu_u-\sum_{i=1}^P \bfx_i^t\text{softmax}_p(\frac{{\bfx_i^t}^\top\bfW^{(I_1)}\bfx^t_p}{d})\|^2/d]\\
        &+ 2\mathbb{E}_{\bfE,\bfX^0}[\frac{\beta_1 v_t^{(I_1)}{}^2\rho\sqrt{\bar{\alpha}_t}}{d}]+2\mathbb{E}_{\bfE,\bfX^0}[\frac{\beta_2v_t^{(I_1)}(v_t^{(I_1)}\sqrt{1-\bar{\alpha}_t}-1)}{d}]
    \end{aligned}
\end{equation}
Since that $\mathbb{E}_t[\bar{\alpha}_t^2]<\infty$, we have
\begin{equation}
    \begin{aligned}
         &\frac{\partial \mathbb{E}_{\bfE,\bfX^0} [\|v_t (\bfx_p^t-\sum_{i=1}^P \bfx_i^t\text{softmax}_p(\frac{{\bfx_i^t}^\top\bfW^{(I_1)}\bfx^t_p}{d}))-\boldsymbol{\epsilon}_p\|^2/d]}{\partial v_t}\\
         =& 2v_t\rho^2\bar{\alpha}_t+2\sqrt{1-\bar{\alpha}_t}(v_t\sqrt{1-\bar{\alpha}_t}-1)+\mathbb{E}_{\bfE,\bfX^0}[2v_t\|\sqrt{\bar{\alpha}}_t\bfmu_u-\sum_{i=1}^P \bfx_i^t\text{softmax}_p(\frac{{\bfx_i^t}^\top\bfW^{(I_1)}\bfx^t_p}{d})\|^2/d]\\
        &+ 4\mathbb{E}_{\bfE,\bfX^0}[\frac{\beta_1 v_t\rho\sqrt{\bar{\alpha}_t}}{d}]+2\mathbb{E}_{\bfE,\bfX^0}[\frac{2\beta_2\sqrt{1-\bar{\alpha}_t}v_t-\beta_2}{d}]\\
        =& 2(v_t\cdot \mathbb{E}_{\bfE,\bfX^0}[\bar{\alpha}_t\rho^2+1-\bar{\alpha}_t+\|\sqrt{\bar{\alpha}}_t\bfmu_u-\sum_{i=1}^P \bfx_i^t\text{softmax}_p(\frac{{\bfx_i^t}^\top\bfW^{(I_1)}\bfx^t_p}{d})\|^2/d\\
        &+\frac{2\beta_1\rho\sqrt{\bar{\alpha}_t}}{d}+\frac{2\beta_2\sqrt{1-\bar{\alpha}_t}}{d}]-(\sqrt{1-\bar{\alpha}_t}+\frac{\beta_2}{d})).\label{gd v_t}
    \end{aligned}
\end{equation}
Denote 
\begin{equation}
    v_t^*=\frac{\sqrt{1-\bar{\alpha}_t}+\frac{\beta_2}{d}}{\bar{\alpha}_t\rho^2+1-\bar{\alpha}_t+\mathbb{E}_{\bfE,\bfX^0}[\|\sqrt{\bar{\alpha}}_t\bfmu_u-\sum_{i=1}^P \bfx_i^t\text{softmax}_p(\frac{{\bfx_i^t}^\top\bfW^{(I_1)}\bfx^t_p}{d})\|^2/d]+2\frac{\beta_1\rho\sqrt{\bar{\alpha}_t}+\beta_2\sqrt{1-\bar{\alpha}_t}}{d}}.
\end{equation}
Define the error $e_t(s)=v_t^{(s)}-v_t^*$. Then, by gradient update,
\begin{equation}
    \begin{aligned}
    &e_t(s+1)\\
    =&v_t^{(s+1)}-v_t^*\\
       =& v_t^{(s)}-v_t^*-2\eta \cdot(v_t^{(s)}\cdot \mathbb{E}_{\bfE,\bfX^0}[\bar{\alpha}_t\rho^2+1-\bar{\alpha}_t+\|\sqrt{\bar{\alpha}}_t\bfmu_u-\sum_{i=1}^P \bfx_i^t\text{softmax}_p(\frac{{\bfx_i^t}^\top\bfW^{(I_1)}\bfx^t_p}{d})\|^2/d\\
        &+\frac{2\beta_1\rho\sqrt{\bar{\alpha}_t}}{d}+\frac{2\beta_2\sqrt{1-\bar{\alpha}_t}}{d}]-(\sqrt{1-\bar{\alpha}_t}+\frac{\beta_2}{d}))\\
        =& (v_t^{(s)}-v_t^*)-2\eta\cdot \mathbb{E}_{\bfE,\bfX^0}[\bar{\alpha}_t\rho^2+1-\bar{\alpha}_t+\|\sqrt{\bar{\alpha}}_t\bfmu_u-\sum_{i=1}^P \bfx_i^t\text{softmax}_p(\frac{{\bfx_i^t}^\top\bfW^{(I_1)}\bfx^t_p}{d})\|^2/d\\
        &+\frac{2\beta_1\rho\sqrt{\bar{\alpha}_t}}{d}+\frac{2\beta_2\sqrt{1-\bar{\alpha}_t}}{d}]\cdot (v_t^{(s)}-v_t^*)\\
        =& (1-2\eta \cdot \mathbb{E}_{\bfE,\bfX^0}[\bar{\alpha}_t\rho^2+1-\bar{\alpha}_t+\|\sqrt{\bar{\alpha}}_t\bfmu_u-\sum_{i=1}^P \bfx_i^t\text{softmax}_p(\frac{{\bfx_i^t}^\top\bfW^{(I_1)}\bfx^t_p}{d})\|^2/d\\
        &+\frac{2\beta_1\rho\sqrt{\bar{\alpha}_t}}{d}+\frac{2\beta_2\sqrt{1-\bar{\alpha}_t}}{d}])\cdot e_t(s).
    \end{aligned}
\end{equation}
Hence, given 
\begin{equation}
\begin{aligned}
\eta\lesssim &(\mathbb{E}_{\bfE,\bfX^0}[\bar{\alpha}_t\rho^2+1-\bar{\alpha}_t+\|\sqrt{\bar{\alpha}}_t\bfmu_u-\sum_{i=1}^P \bfx_i^t\text{softmax}_p(\frac{{\bfx_i^t}^\top\bfW^{(I_1)}\bfx^t_p}{d})\|^2/d+\frac{2\beta_1\rho\sqrt{\bar{\alpha}_t}}{d}\\
&+\frac{2\beta_2\sqrt{1-\bar{\alpha}_t}}{d}])^{-1/2},
\end{aligned}
\end{equation}
i.e., 
\begin{equation}
\eta\lesssim \frac{1}{\max\{\rho,1\}+\epsilon} 
\end{equation}
by (\ref{stage 2 second term}), we can derive
\begin{equation}
\begin{aligned}
    |e_t(s)|\leq &(1-2\eta \cdot \mathbb{E}_{\bfE,\bfX^0}[\bar{\alpha}_t\rho^2+1-\bar{\alpha}_t+\|\sqrt{\bar{\alpha}}_t\bfmu_u-\sum_{i=1}^P \bfx_i^t\text{softmax}_p(\frac{{\bfx_i^t}^\top\bfW^{(I_1)}\bfx^t_p}{d})\|^2/d\\
        &+\frac{2\beta_1\rho\sqrt{\bar{\alpha}_t}}{d}+\frac{2\beta_2\sqrt{1-\bar{\alpha}_t}}{d}])^s \cdot |e_t(I_1)|,
        \end{aligned}
\end{equation}
which means after 
\begin{equation}
\begin{aligned}
    s\gtrsim I_2:=&\log \frac{|e_t(I_1)|}{\epsilon}/\log (1-2\eta \cdot \mathbb{E}_{\bfE,\bfX^0}[\bar{\alpha}_t\rho^2+1-\bar{\alpha}_t+\|\sqrt{\bar{\alpha}}_t\bfmu_u-\sum_{i=1}^P \bfx_i^t\text{softmax}_p(\frac{{\bfx_i^t}^\top\bfW^{(I_1)}\bfx^t_p}{d})\|^2/d\\
        &+\frac{2\beta_1\rho\sqrt{\bar{\alpha}_t}}{d}+\frac{2\beta_2\sqrt{1-\bar{\alpha}_t}}{d}])^{-1}\\
        =& \Theta(\log \frac{|e_t(I_1)|}{\epsilon})
    \end{aligned}
\end{equation}
iterations, we can achieve that $|e_t(s)|\leq \epsilon$ and $v_t^{(s)}$ converges to $v_t^*$. Given $v_t^{(0)}=\Theta(1)$, by (\ref{gd v_t}), we have that for any $v_t\leq \Theta(1)$,
\begin{equation}
\begin{aligned}
    &\frac{\partial \mathbb{E}_{\bfE,\bfX^0} [\|v_t (\bfx_p^t-\sum_{i=1}^P \bfx_i^t\text{softmax}_p(\frac{{\bfx_i^t}^\top\bfW^{(I_1)}\bfx^t_p}{d}))-\boldsymbol{\epsilon}_p\|^2/d]}{\partial v_t}\\
    \leq & \mathbb{E}_{\bfE,\bfX^0}[\bar{\alpha}_t\rho^2+1-\bar{\alpha}_t+\bar{\alpha}_t+\frac{2\beta_1\rho\sqrt{\bar{\alpha}_t}}{d}+\frac{2\beta_2\sqrt{1-\bar{\alpha}_t}}{d}]-(\sqrt{1-\bar{\alpha}_t}+\frac{\beta_2}{d})]\\
    \lesssim & \rho^2+1,
    \end{aligned}
\end{equation}
where the last step is by $\beta_1\sim \mathcal{N}(0, \|\sqrt{\bar{\alpha}}_t\bfmu_u-\sum_{i=1}^P \bfx_i^t\text{softmax}_p(\frac{{\bfx_i^t}^\top\bfW^{(I_1)}\bfx^t_p}{d})\|^2)$, $\beta_2\sim \mathcal{N}(0, \|\sqrt{\bar{\alpha}}_t\bfmu_u-\sum_{i=1}^P \bfx_i^t\text{softmax}_p(\frac{{\bfx_i^t}^\top\bfW^{(I_1)}\bfx^t_p}{d})\|^2)$, so that with a high probability, 
\begin{equation}
    \beta_1, \beta_2\lesssim \sqrt{d}.
\end{equation}
Hence,
\begin{equation}
    \begin{aligned}
        |v_t^{(I_1)}|=&\Big|v_t^{(0)}-\sum_{s=0}^{I_1-1}\frac{\partial \mathbb{E}_{\bfE,\bfX^0} [\|v_t^{(s)} (\bfx_p^t-\sum_{i=1}^P \bfx_i^t\text{softmax}_p(\frac{{\bfx_i^t}^\top\bfW^{(I_1)}\bfx^t_p}{d}))-\boldsymbol{\epsilon}_p\|^2/d]}{\partial v_t}\Big|\\
        \lesssim & I_1(\rho^2+1),
    \end{aligned}
\end{equation}
and 
\begin{equation}
    |e_t(I_1)|\lesssim I_1(\rho^2+1).
\end{equation}
Then, 
\begin{equation}
    I_2=\Theta(\log \frac{I_1(\rho^2+1)}{\epsilon}).
\end{equation}

\end{proof}

%% file: main.bbl
\begin{thebibliography}{64}
\providecommand{\natexlab}[1]{#1}
\providecommand{\url}[1]{\texttt{#1}}
\expandafter\ifx\csname urlstyle\endcsname\relax
  \providecommand{\doi}[1]{doi: #1}\else
  \providecommand{\doi}{doi: \begingroup \urlstyle{rm}\Url}\fi

\bibitem[Allen-Zhu \& Li(2023)Allen-Zhu and Li]{AL23}
Allen-Zhu, Z. and Li, Y.
\newblock Towards understanding ensemble, knowledge distillation and self-distillation in deep learning.
\newblock In \emph{The Eleventh International Conference on Learning Representations}, 2023.

\bibitem[Arriola et~al.(2025)Arriola, Sahoo, Gokaslan, Yang, Qi, Han, Chiu, and Kuleshov]{ASGY25}
Arriola, M., Sahoo, S.~S., Gokaslan, A., Yang, Z., Qi, Z., Han, J., Chiu, J.~T., and Kuleshov, V.
\newblock Block diffusion: Interpolating between autoregressive and diffusion language models.
\newblock In \emph{The Thirteenth International Conference on Learning Representations}, 2025.
\newblock URL \url{https://openreview.net/forum?id=tyEyYT267x}.

\bibitem[Azangulov et~al.(2024)Azangulov, Deligiannidis, and Rousseau]{ADR24}
Azangulov, I., Deligiannidis, G., and Rousseau, J.
\newblock Convergence of diffusion models under the manifold hypothesis in high-dimensions.
\newblock \emph{arXiv preprint arXiv:2409.18804}, 2024.

\bibitem[Bar-Tal et~al.(2024)Bar-Tal, Chefer, Tov, Herrmann, Paiss, Zada, Ephrat, Hur, Liu, Raj, et~al.]{BCTH24}
Bar-Tal, O., Chefer, H., Tov, O., Herrmann, C., Paiss, R., Zada, S., Ephrat, A., Hur, J., Liu, G., Raj, A., et~al.
\newblock Lumiere: A space-time diffusion model for video generation.
\newblock In \emph{SIGGRAPH Asia 2024 Conference Papers}, pp.\  1--11, 2024.

\bibitem[Boffi et~al.(2025)Boffi, Jacot, Tu, and Ziemann]{BJTZ25}
Boffi, N.~M., Jacot, A., Tu, S., and Ziemann, I.
\newblock Shallow diffusion networks provably learn hidden low-dimensional structure.
\newblock In \emph{The Thirteenth International Conference on Learning Representations}, 2025.

\bibitem[Bonnaire et~al.(2025)Bonnaire, Urfin, Biroli, and M{\'e}zard]{BUBM25}
Bonnaire, T., Urfin, R., Biroli, G., and M{\'e}zard, M.
\newblock Why diffusion models don't memorize: The role of implicit dynamical regularization in training.
\newblock \emph{arXiv preprint arXiv:2505.17638}, 2025.

\bibitem[Cai et~al.(2025)Cai, Cun, Li, Liu, Zhang, Zhang, Shan, and Yue]{CCLL25}
Cai, M., Cun, X., Li, X., Liu, W., Zhang, Z., Zhang, Y., Shan, Y., and Yue, X.
\newblock Ditctrl: Exploring attention control in multi-modal diffusion transformer for tuning-free multi-prompt longer video generation.
\newblock In \emph{Proceedings of the Computer Vision and Pattern Recognition Conference}, pp.\  7763--7772, 2025.

\bibitem[Chen et~al.(2023)Chen, Chewi, Li, Li, Salim, and Zhang]{CCLL23}
Chen, S., Chewi, S., Li, J., Li, Y., Salim, A., and Zhang, A.
\newblock Sampling is as easy as learning the score: theory for diffusion models with minimal data assumptions.
\newblock In \emph{The Eleventh International Conference on Learning Representations}, 2023.

\bibitem[Han et~al.(2025)Han, Huang, Cao, and Zou]{HHCZ25}
Han, A., Huang, W., Cao, Y., and Zou, D.
\newblock On the feature learning in diffusion models.
\newblock In \emph{The Thirteenth International Conference on Learning Representations}, 2025.

\bibitem[Han et~al.(2024)Han, Razaviyayn, and Xu]{HRX24}
Han, Y., Razaviyayn, M., and Xu, R.
\newblock Neural network-based score estimation in diffusion models: Optimization and generalization.
\newblock In \emph{The Twelfth International Conference on Learning Representations}, 2024.

\bibitem[Ho et~al.(2020)Ho, Jain, and Abbeel]{HJA20}
Ho, J., Jain, A., and Abbeel, P.
\newblock Denoising diffusion probabilistic models.
\newblock \emph{Advances in neural information processing systems}, 33:\penalty0 6840--6851, 2020.

\bibitem[Hoogeboom et~al.(2022)Hoogeboom, Satorras, Vignac, and Welling]{HSVW22}
Hoogeboom, E., Satorras, V.~G., Vignac, C., and Welling, M.
\newblock Equivariant diffusion for molecule generation in 3d.
\newblock In \emph{International conference on machine learning}, pp.\  8867--8887. PMLR, 2022.

\bibitem[Huang et~al.(2023)Huang, Cheng, and Liang]{HCL23}
Huang, Y., Cheng, Y., and Liang, Y.
\newblock In-context convergence of transformers.
\newblock In \emph{NeurIPS 2023 Workshop on Mathematics of Modern Machine Learning}, 2023.

\bibitem[Huang et~al.(2024{\natexlab{a}})Huang, Wen, Chi, and Liang]{HWCL24}
Huang, Y., Wen, Z., Chi, Y., and Liang, Y.
\newblock Transformers provably learn feature-position correlations in masked image modeling.
\newblock \emph{arXiv preprint arXiv:2403.02233}, 2024{\natexlab{a}}.

\bibitem[Huang et~al.(2024{\natexlab{b}})Huang, Wei, and Chen]{HWC24}
Huang, Z., Wei, Y., and Chen, Y.
\newblock Denoising diffusion probabilistic models are optimally adaptive to unknown low dimensionality.
\newblock \emph{arXiv preprint arXiv:2410.18784}, 2024{\natexlab{b}}.

\bibitem[Jacot et~al.(2018)Jacot, Gabriel, and Hongler]{JGH18}
Jacot, A., Gabriel, F., and Hongler, C.
\newblock Neural tangent kernel: Convergence and generalization in neural networks.
\newblock \emph{Advances in neural information processing systems}, 31, 2018.

\bibitem[Jelassi et~al.(2022)Jelassi, Sander, and Li]{JSL22}
Jelassi, S., Sander, M., and Li, Y.
\newblock Vision transformers provably learn spatial structure.
\newblock \emph{Advances in Neural Information Processing Systems}, 35:\penalty0 37822--37836, 2022.

\bibitem[Jiang et~al.(2024)Jiang, Huang, Zhang, Suzuki, and Nie]{JHZS24}
Jiang, J., Huang, W., Zhang, M., Suzuki, T., and Nie, L.
\newblock Unveil benign overfitting for transformer in vision: Training dynamics, convergence, and generalization.
\newblock In \emph{The Thirty-eighth Annual Conference on Neural Information Processing Systems}, 2024.
\newblock URL \url{https://openreview.net/forum?id=FGJb0peY4R}.

\bibitem[Kong et~al.(2021)Kong, Ping, Huang, Zhao, and Catanzaro]{KPHZ21}
Kong, Z., Ping, W., Huang, J., Zhao, K., and Catanzaro, B.
\newblock Diffwave: A versatile diffusion model for audio synthesis.
\newblock In \emph{International Conference on Learning Representations}, 2021.

\bibitem[LeCun et~al.(2002)LeCun, Bottou, Bengio, and Haffner]{LBBH02}
LeCun, Y., Bottou, L., Bengio, Y., and Haffner, P.
\newblock Gradient-based learning applied to document recognition.
\newblock \emph{Proceedings of the IEEE}, 86\penalty0 (11):\penalty0 2278--2324, 2002.

\bibitem[Li \& Yan(2024)Li and Yan]{LY25}
Li, G. and Yan, Y.
\newblock Adapting to unknown low-dimensional structures in score-based diffusion models.
\newblock \emph{Advances in Neural Information Processing Systems}, 37:\penalty0 126297--126331, 2024.

\bibitem[Li et~al.(2023{\natexlab{a}})Li, Wang, Liu, and Chen]{LWLC23}
Li, H., Wang, M., Liu, S., and Chen, P.-Y.
\newblock A theoretical understanding of shallow vision transformers: Learning, generalization, and sample complexity.
\newblock In \emph{The Eleventh International Conference on Learning Representations}, 2023{\natexlab{a}}.
\newblock URL \url{https://openreview.net/forum?id=jClGv3Qjhb}.

\bibitem[Li et~al.(2023{\natexlab{b}})Li, Wang, Lu, Wan, Cui, and Chen]{LWLW23}
Li, H., Wang, M., Lu, S., Wan, H., Cui, X., and Chen, P.-Y.
\newblock Transformers as multi-task feature selectors: Generalization analysis of in-context learning.
\newblock In \emph{NeurIPS 2023 Workshop on Mathematics of Modern Machine Learning}, 2023{\natexlab{b}}.
\newblock URL \url{https://openreview.net/forum?id=BMQ4i2RVbE}.

\bibitem[Li et~al.(2024{\natexlab{a}})Li, Wang, Lu, Cui, and Chen]{LWLC24}
Li, H., Wang, M., Lu, S., Cui, X., and Chen, P.-Y.
\newblock How do nonlinear transformers learn and generalize in in-context learning?
\newblock In \emph{Forty-first International Conference on Machine Learning}, 2024{\natexlab{a}}.
\newblock URL \url{https://openreview.net/forum?id=I4HTPws9P6}.

\bibitem[Li et~al.(2024{\natexlab{b}})Li, Wang, Lu, Cui, and Chen]{LWLC24_cot0}
Li, H., Wang, M., Lu, S., Cui, X., and Chen, P.-Y.
\newblock How do nonlinear transformers acquire generalization-guaranteed cot ability?
\newblock In \emph{High-dimensional Learning Dynamics 2024: The Emergence of Structure and Reasoning}, 2024{\natexlab{b}}.

\bibitem[Li et~al.(2024{\natexlab{c}})Li, Wang, Ma, Liu, ZHANG, and Chen]{LWMS24}
Li, H., Wang, M., Ma, T., Liu, S., ZHANG, Z., and Chen, P.-Y.
\newblock What improves the generalization of graph transformers? a theoretical dive into the self-attention and positional encoding.
\newblock In \emph{Forty-first International Conference on Machine Learning}, 2024{\natexlab{c}}.
\newblock URL \url{https://openreview.net/forum?id=mJhXlsZzzE}.

\bibitem[Li et~al.(2024{\natexlab{d}})Li, Wang, Zhang, Liu, and Chen]{LWZL24}
Li, H., Wang, M., Zhang, S., Liu, S., and Chen, P.-Y.
\newblock Learning on transformers is provable low-rank and sparse: A one-layer analysis.
\newblock In \emph{2024 IEEE 13rd Sensor Array and Multichannel Signal Processing Workshop (SAM)}, pp.\  1--5. IEEE, 2024{\natexlab{d}}.

\bibitem[Li et~al.(2025{\natexlab{a}})Li, Lu, Chen, Cui, and Wang]{LWLC24_cot}
Li, H., Lu, S., Chen, P.-Y., Cui, X., and Wang, M.
\newblock Training nonlinear transformers for chain-of-thought inference: A theoretical generalization analysis.
\newblock In \emph{The Thirteenth International Conference on Learning Representations}, 2025{\natexlab{a}}.

\bibitem[Li et~al.(2025{\natexlab{b}})Li, Lu, Cui, Chen, and Wang]{LLCC25}
Li, H., Lu, S., Cui, X., Chen, P.-Y., and Wang, M.
\newblock Can mamba learn in context with outliers? a theoretical generalization analysis.
\newblock \emph{arXiv preprint arXiv:2510.00399}, 2025{\natexlab{b}}.

\bibitem[Li et~al.(2025{\natexlab{c}})Li, Zhang, Zhang, Chen, Liu, and Wang]{LZZC25}
Li, H., Zhang, Y., Zhang, S., Chen, P.-Y., Liu, S., and Wang, M.
\newblock When is task vector provably effective for model editing? a generalization analysis of nonlinear transformers.
\newblock In \emph{The Thirteenth International Conference on Learning Representations}, 2025{\natexlab{c}}.

\bibitem[Li et~al.(2023{\natexlab{c}})Li, Li, Zhang, and Bian]{LLZB23}
Li, P., Li, Z., Zhang, H., and Bian, J.
\newblock On the generalization properties of diffusion models.
\newblock \emph{Advances in Neural Information Processing Systems}, 36:\penalty0 2097--2127, 2023{\natexlab{c}}.

\bibitem[Li et~al.(2024{\natexlab{e}})Li, Biferale, Bonaccorso, Scarpolini, and Buzzicotti]{LBBS24}
Li, T., Biferale, L., Bonaccorso, F., Scarpolini, M.~A., and Buzzicotti, M.
\newblock Synthetic lagrangian turbulence by generative diffusion models.
\newblock \emph{Nature Machine Intelligence}, 6\penalty0 (4):\penalty0 393--403, 2024{\natexlab{e}}.

\bibitem[Li et~al.(2024{\natexlab{f}})Li, Dai, and Qu]{LDQ24}
Li, X., Dai, Y., and Qu, Q.
\newblock Understanding generalizability of diffusion models requires rethinking the hidden gaussian structure.
\newblock \emph{Advances in neural information processing systems}, 37:\penalty0 57499--57538, 2024{\natexlab{f}}.

\bibitem[Li et~al.(2025{\natexlab{d}})Li, Zhang, Li, Chen, Zhu, Wang, and Qu]{LZLC25}
Li, X., Zhang, Z., Li, X., Chen, S., Zhu, Z., Wang, P., and Qu, Q.
\newblock Understanding representation dynamics of diffusion models via low-dimensional modeling.
\newblock \emph{arXiv preprint arXiv:2502.05743}, 2025{\natexlab{d}}.

\bibitem[Liang et~al.(2025)Liang, Huang, and Chen]{LHC25}
Liang, J., Huang, Z., and Chen, Y.
\newblock Low-dimensional adaptation of diffusion models: Convergence in total variation.
\newblock \emph{arXiv preprint arXiv:2501.12982}, 2025.

\bibitem[Luo(2022)]{L22}
Luo, C.
\newblock Understanding diffusion models: A unified perspective.
\newblock \emph{arXiv preprint arXiv:2208.11970}, 2022.

\bibitem[Luo et~al.(2024)Luo, Li, Shi, and Wu]{LLSW24}
Luo, Y., Li, H., Shi, L., and Wu, X.-M.
\newblock Enhancing graph transformers with hierarchical distance structural encoding.
\newblock In \emph{The Thirty-eighth Annual Conference on Neural Information Processing Systems}, 2024.
\newblock URL \url{https://openreview.net/forum?id=U4KldRgoph}.

\bibitem[Min \& Vidal(2025)Min and Vidal]{MV25}
Min, H. and Vidal, R.
\newblock Gradient flow provably learns robust classifiers for orthonormal gmms.
\newblock In \emph{Forty-second International Conference on Machine Learning}, 2025.

\bibitem[Mohri et~al.(2018)Mohri, Rostamizadeh, and Talwalkar]{MRT18}
Mohri, M., Rostamizadeh, A., and Talwalkar, A.
\newblock \emph{Foundations of machine learning}.
\newblock MIT press, 2018.

\bibitem[Peebles \& Xie(2023)Peebles and Xie]{PX23}
Peebles, W. and Xie, S.
\newblock Scalable diffusion models with transformers.
\newblock In \emph{Proceedings of the IEEE/CVF international conference on computer vision}, pp.\  4195--4205, 2023.

\bibitem[Pham et~al.(2025)Pham, Raya, Negri, Zaki, Ambrogioni, and Krotov]{PRNZ25}
Pham, B., Raya, G., Negri, M., Zaki, M.~J., Ambrogioni, L., and Krotov, D.
\newblock Memorization to generalization: Emergence of diffusion models from associative memory.
\newblock \emph{arXiv preprint arXiv:2505.21777}, 2025.

\bibitem[Price et~al.(2025)Price, Sanchez-Gonzalez, Alet, Andersson, El-Kadi, Masters, Ewalds, Stott, Mohamed, Battaglia, et~al.]{PSAA25}
Price, I., Sanchez-Gonzalez, A., Alet, F., Andersson, T.~R., El-Kadi, A., Masters, D., Ewalds, T., Stott, J., Mohamed, S., Battaglia, P., et~al.
\newblock Probabilistic weather forecasting with machine learning.
\newblock \emph{Nature}, 637\penalty0 (8044):\penalty0 84--90, 2025.

\bibitem[Rahimi \& Recht(2007)Rahimi and Recht]{RR07}
Rahimi, A. and Recht, B.
\newblock Random features for large-scale kernel machines.
\newblock In \emph{Advances in Neural Information Processing Systems}, volume~20, 2007.

\bibitem[Rombach et~al.(2022)Rombach, Blattmann, Lorenz, Esser, and Ommer]{RBLE22}
Rombach, R., Blattmann, A., Lorenz, D., Esser, P., and Ommer, B.
\newblock High-resolution image synthesis with latent diffusion models.
\newblock In \emph{Proceedings of the IEEE/CVF conference on computer vision and pattern recognition}, pp.\  10684--10695, 2022.

\bibitem[Ruan et~al.(2023)Ruan, Ma, Yang, He, Liu, Fu, Yuan, Jin, and Guo]{RMYH23}
Ruan, L., Ma, Y., Yang, H., He, H., Liu, B., Fu, J., Yuan, N.~J., Jin, Q., and Guo, B.
\newblock Mm-diffusion: Learning multi-modal diffusion models for joint audio and video generation.
\newblock In \emph{Proceedings of the IEEE/CVF Conference on Computer Vision and Pattern Recognition}, pp.\  10219--10228, 2023.

\bibitem[Sahoo et~al.(2024)Sahoo, Arriola, Schiff, Gokaslan, Marroquin, Chiu, Rush, and Kuleshov]{SASG24}
Sahoo, S., Arriola, M., Schiff, Y., Gokaslan, A., Marroquin, E., Chiu, J., Rush, A., and Kuleshov, V.
\newblock Simple and effective masked diffusion language models.
\newblock \emph{Advances in Neural Information Processing Systems}, 37:\penalty0 130136--130184, 2024.

\bibitem[Sclocchi et~al.(2025)Sclocchi, Favero, and Wyart]{SFWM25}
Sclocchi, A., Favero, A., and Wyart, M.
\newblock A phase transition in diffusion models reveals the hierarchical nature of data.
\newblock \emph{Proceedings of the National Academy of Sciences}, 122\penalty0 (1):\penalty0 e2408799121, 2025.

\bibitem[Shandirasegaran et~al.(2026)Shandirasegaran, Li, Zhang, Wang, and Zhang]{SLZW26}
Shandirasegaran, M., Li, H., Zhang, S., Wang, M., and Zhang, S.
\newblock A theoretical analysis of mamba’s training dynamics: Filtering relevant features for generalization in state space models.
\newblock In \emph{The Fourteenth International Conference on Learning Representations}, 2026.

\bibitem[Shen et~al.(2025)Shen, Zhou, Yang, and Shen]{SZYS25}
Shen, W., Zhou, R., Yang, J., and Shen, C.
\newblock On the training convergence of transformers for in-context classification of gaussian mixtures.
\newblock In \emph{Forty-second International Conference on Machine Learning}, 2025.

\bibitem[Sohl-Dickstein et~al.(2015)Sohl-Dickstein, Weiss, Maheswaranathan, and Ganguli]{SWMG15}
Sohl-Dickstein, J., Weiss, E., Maheswaranathan, N., and Ganguli, S.
\newblock Deep unsupervised learning using nonequilibrium thermodynamics.
\newblock In \emph{International conference on machine learning}, pp.\  2256--2265. pmlr, 2015.

\bibitem[Song \& Ermon(2019)Song and Ermon]{SE19}
Song, Y. and Ermon, S.
\newblock Generative modeling by estimating gradients of the data distribution.
\newblock \emph{Advances in neural information processing systems}, 32, 2019.

\bibitem[Song et~al.(2021)Song, Sohl-Dickstein, Kingma, Kumar, Ermon, and Poole]{SSKK21}
Song, Y., Sohl-Dickstein, J., Kingma, D.~P., Kumar, A., Ermon, S., and Poole, B.
\newblock Score-based generative modeling through stochastic differential equations.
\newblock In \emph{International Conference on Learning Representations}, 2021.

\bibitem[Sun et~al.(2025)Sun, Zhang, Li, and Wang]{SZLW25}
Sun, J., Zhang, S., Li, H., and Wang, M.
\newblock Contrastive learning with data misalignment: Feature purity, training dynamics and theoretical generalization guarantees.
\newblock In \emph{The Thirty-ninth Annual Conference on Neural Information Processing Systems}, 2025.

\bibitem[Tarzanagh et~al.(2023{\natexlab{a}})Tarzanagh, Li, Thrampoulidis, and Oymak]{TLTO23}
Tarzanagh, D.~A., Li, Y., Thrampoulidis, C., and Oymak, S.
\newblock Transformers as support vector machines.
\newblock \emph{arXiv preprint arXiv:2308.16898}, 2023{\natexlab{a}}.

\bibitem[Tarzanagh et~al.(2023{\natexlab{b}})Tarzanagh, Li, Zhang, and Oymak]{TLZO23}
Tarzanagh, D.~A., Li, Y., Zhang, X., and Oymak, S.
\newblock Max-margin token selection in attention mechanism.
\newblock \emph{CoRR}, 2023{\natexlab{b}}.

\bibitem[Vershynin(2010)]{V10}
Vershynin, R.
\newblock Introduction to the non-asymptotic analysis of random matrices.
\newblock \emph{arXiv preprint arXiv:1011.3027}, 2010.

\bibitem[Wang \& Pehlevan(2025)Wang and Pehlevan]{WP25}
Wang, B. and Pehlevan, C.
\newblock An analytical theory of spectral bias in the learning dynamics of diffusion models.
\newblock In \emph{The Thirty-ninth Annual Conference on Neural Information Processing Systems}, 2025.

\bibitem[Wang et~al.(2024{\natexlab{a}})Wang, Zhang, Zhang, Chen, Ma, and Qu]{WZZC24}
Wang, P., Zhang, H., Zhang, Z., Chen, S., Ma, Y., and Qu, Q.
\newblock Diffusion models learn low-dimensional distributions via subspace clustering.
\newblock \emph{arXiv preprint arXiv:2409.02426}, 2024{\natexlab{a}}.

\bibitem[Wang et~al.(2024{\natexlab{b}})Wang, He, and Tao]{WHT24}
Wang, Y., He, Y., and Tao, M.
\newblock Evaluating the design space of diffusion-based generative models.
\newblock \emph{Advances in Neural Information Processing Systems}, 37:\penalty0 19307--19352, 2024{\natexlab{b}}.

\bibitem[Xing et~al.(2024)Xing, Feng, Chen, Dai, Hu, Xu, Wu, and Jiang]{XFCD24}
Xing, Z., Feng, Q., Chen, H., Dai, Q., Hu, H., Xu, H., Wu, Z., and Jiang, Y.-G.
\newblock A survey on video diffusion models.
\newblock \emph{ACM Computing Surveys}, 57\penalty0 (2):\penalty0 1--42, 2024.

\bibitem[Zhang et~al.(2025{\natexlab{a}})Zhang, Li, Shi, Rong, Zhao, Wang, Guo, and Wang]{ZLSR25}
Zhang, B., Li, H., Shi, C., Rong, G., Zhao, H., Wang, D., Guo, D., and Wang, M.
\newblock Merging smarter, generalizing better: Enhancing model merging on ood data.
\newblock \emph{arXiv preprint arXiv:2506.09093}, 2025{\natexlab{a}}.

\bibitem[Zhang et~al.(2023{\natexlab{a}})Zhang, Zhang, Zheng, Zhang, Qamar, Bae, and Kweon]{ZZZZ23}
Zhang, C., Zhang, C., Zheng, S., Zhang, M., Qamar, M., Bae, S.-H., and Kweon, I.~S.
\newblock A survey on audio diffusion models: Text to speech synthesis and enhancement in generative ai.
\newblock \emph{arXiv preprint arXiv:2303.13336}, 2023{\natexlab{a}}.

\bibitem[Zhang et~al.(2023{\natexlab{b}})Zhang, Frei, and Bartlett]{ZFB23}
Zhang, R., Frei, S., and Bartlett, P.~L.
\newblock Trained transformers learn linear models in-context.
\newblock \emph{arXiv preprint arXiv:2306.09927}, 2023{\natexlab{b}}.

\bibitem[Zhang et~al.(2025{\natexlab{b}})Zhang, Li, Yao, Chen, Zhang, Chen, Wang, and Liu]{ZLYC25}
Zhang, Y., Li, H., Yao, Y., Chen, A., Zhang, S., Chen, P.-Y., Wang, M., and Liu, S.
\newblock Visual prompting reimagined: The power of activation prompts.
\newblock In \emph{The Second Conference on Parsimony and Learning (Recent Spotlight Track)}, 2025{\natexlab{b}}.

\end{thebibliography}
